\documentclass[a4paper]{article}
\usepackage[a4paper, margin=2.5cm]{geometry} %

\usepackage[square]{natbib}
\let\cite\citep %

\usepackage[svgnames]{xcolor}

\newcommand{\data}[1]{#1}

\newcommand{\bld}[1]{\boldmath\textbf{#1}\unboldmath}

\usepackage{tabularx}
\usepackage{booktabs}
\newcommand{\bestCovr}[1]{\textbf{#1}}

\usepackage{xurl}                                 %
\usepackage[backref=page]{hyperref}               %
\hypersetup{
linktocpage,                                      %
colorlinks  = true,                               %
urlcolor    = cyan,
linkcolor   = DarkGreen,
citecolor   = blue
}
\renewcommand*{\backref}[1]{}
\renewcommand*{\backrefalt}[4]{[%
\ifcase #1 Not cited.%
  \or Cited on page~#2.%
  \else Cited on pages #2.%
\fi]}

\usepackage{url}

\usepackage{cancel} %
\usepackage{subcaption} %
\usepackage[titletoc]{appendix}
\usepackage[tight-spacing=true]{siunitx} %
\usepackage{multirow}
\usepackage{enumitem} %
\usepackage[export]{adjustbox}
\usepackage[ruled,vlined,linesnumbered]{algorithm2e} %
\SetKwProg{function}{Function}{}{}
\usepackage{lscape}
\usepackage{makecell}
\usepackage{xspace}
\usepackage{fwt_math_and_notation}
\usepackage{fwt_lp}
\usepackage{cuted}
\usepackage{framed}
\usepackage{mathrsfs}
\usepackage{bm}
\usepackage{tikz-network}
\usepackage{tikz}
\usetikzlibrary{fit, calc, patterns}
\usetikzlibrary{shapes.multipart}  %

\newcommand{\citeAuthor}[1]{\citet{#1}} %

\usepackage{cleveref}

\crefname{algorithm}{algorithm}{algorithms}
\crefname{line}{line}{lines}
\crefname{definition}{definition}{definitions}
\crefname{lemma}{lemma}{lemmata}
\crefname{theorem}{theorem}{theorems}
\crefname{figure}{figure}{figures}
\crefname{table}{table}{tables}
\crefname{assumption}{assumption}{assumptions}

\newcommand{\wrt}{\ensuremath{\text{w.r.t.\ }}}
\newcommand{\st}{\ensuremath{\text{s.t.\ }}}

\newcommand{\astar}{A$^*$\xspace}
\newcommand{\peastar}{PE\astar}
\newcommand{\lao}{LAO$^*$\xspace}
\newcommand{\ilao}{\ensuremath{\text{iLAO}^*}\xspace}
\newcommand{\cgilao}{CG-\ilao}

\newcommand{\lrtdp}{LRTDP\xspace}
\newcommand{\ftvi}{FTVI\xspace}

\newcommand{\bellmanTied}{\ensuremath{\textsc{All-Greedy-Actions}}\xspace}
\newcommand{\bellmanSingle}{\ensuremath{\textsc{Single-Greedy-Action}}\xspace}

\newcommand{\cgilaoBellmanTied}{\ensuremath{\text{CG-iLAO}^{*}_{\text{tied}}}\xspace}
\newcommand{\cgilaoBellmanSingle}{\ensuremath{\text{CG-iLAO}^{*}_{\text{single}}}\xspace}
\newcommand{\cgilaoBellmanComplete}{\ensuremath{\text{CG-iLAO}^{*}_{\text{all}}}\xspace}
\newcommand{\cgilaoTrial}{\ensuremath{\text{CG-iLAO}^{*}_{\text{trial}}}\xspace}
\newcommand{\cgilaoFFAO}{\ensuremath{\text{CG-iLAO}^{*}_{\text{FF-AO}}}\xspace}
\newcommand{\cgilaoFFMLO}{\ensuremath{\text{CG-iLAO}^{*}_{\text{FF-MLO}}}\xspace}
\newcommand{\cgilaoBellmanTiedActionElim}{\ensuremath{\text{CG-iLAO}^{*}_{\text{tied-elim}}}\xspace}
\newcommand{\cgilaoBellmanSingleActionElim}{\ensuremath{\text{CG-iLAO}^{*}_{\text{single-elim}}}\xspace}
\newcommand{\cgilaoBellmanCompleteActionElim}{\ensuremath{\text{CG-iLAO}^{*}_{\text{all-elim}}}\xspace}

\newcommand{\find}{\textsc{Find}\xspace}
\newcommand{\revise}{\textsc{Revise}\xspace}
\newcommand{\findAndRevise}{\textsc{Find-and-Revise}\xspace}

\newcommand{\roc}{\ensuremath{\text{ROC}}\xspace}
\newcommand{\lmcut}{\ensuremath{\text{LMCut}}\xspace}
\newcommand{\pdb}{\ensuremath{\text{PDB}}\xspace}

\newcommand{\supp}{\ensuremath{\text{succ}}}
\newcommand{\definedas}{\ensuremath{\overset{\text{\tiny{def}}}{=}}}

\newcommand{\partssptuple}{\ensuremath{\langle \partstates, \sZ, \partgoals, \partactions, \pr,\C, \h \rangle}\xspace}
\newcommand{\partssp}{\ensuremath{\widehat{\ssp}}\xspace}
\newcommand{\partstates}{\ensuremath{\widehat{\Ss}}\xspace}
\newcommand{\partgoals}{\ensuremath{\widehat{\Sg}}\xspace}
\newcommand{\partactions}{\ensuremath{\widehat{\A}}\xspace}

\newcommand{\bw}{BW\xspace}

\newcommand{\elevators}{Elev\xspace}
\newcommand{\exbw}{ExBW\xspace}
\newcommand{\random}{Rand\xspace}
\newcommand{\recttireworld}{\ensuremath{\Box\text{TW}}\xspace}
\newcommand{\schedule}{Sched\xspace}
\newcommand{\sar}{S\&R\xspace}
\newcommand{\sysadmin}{Sys\xspace}
\newcommand{\tritireworld}{\ensuremath{\Delta\text{TW}}\xspace}
\newcommand{\zenotravel}{Zeno\xspace}
\newcommand{\coresec}{Coresec\xspace}
\newcommand{\parc}{PARC\xspace}
\newcommand{\parcn}{PARC-N\xspace}
\newcommand{\parcr}{PARC-R\xspace}

\newcommand{\lpxs}{\ensuremath{\bm{x}}\xspace}

\newcommand{\duals}{\ensuremath{\bm{\mathcal{V}}}\xspace}

\newcommand{\incDecTol}{\ensuremath{\eta}\xspace} %
\newcommand{\currp}{\ensuremath{\widehat{\p}_{\text{curr}}}\xspace}
\newcommand{\oldp}{\ensuremath{\widehat{\p}_{\text{old}}}\xspace}
\newcommand{\partp}{\ensuremath{\widehat{\p}_{\V}}\xspace}
\newcommand{\colsToCheck}{\ensuremath{\Gamma}\xspace}
\newcommand{\envelope}{\ensuremath{\mathcal{E}}\xspace}

\newcommand{\residual}{\ensuremath{\textsc{res}}\xspace}

\newcommand{\improvePolicy}{\ensuremath{\textsc{Backups}}}
\newcommand{\CGimprovePolicy}{\ensuremath{\textsc{CG-Backups}}}
\newcommand{\addStateActions}{\ensuremath{\textsc{Add-Actions}}}
\newcommand{\expandFringes}{\textsc{Expand-Fringes}}
\newcommand{\partiallyExpandFringes}{\textsc{Partly-Expand-Fringes}}
\newcommand{\fixViolatedConstrs}{\textsc{Fix-Constrs}}
\newcommand{\predecessors}{\ensuremath{\textsc{Preds}}\xspace}
\newcommand{\successors}{\ensuremath{\textsc{Ext-Succs}}\xspace}

\newcommand{\V}{\ensuremath{V}\xspace}
\newcommand{\Vlb}{\ensuremath{\V_{\text{lb}}}\xspace}
\newcommand{\Vub}{\ensuremath{\V_{\text{ub}}}\xspace}
\newcommand{\Q}{\ensuremath{Q}\xspace}
\newcommand{\Qlb}{\ensuremath{\Q_{\text{lb}}}\xspace}
\newcommand{\Qub}{\ensuremath{\Q_{\text{ub}}}\xspace}
\newcommand{\Qsa}{\ensuremath{Q(\s, \ac)}\xspace}
\newcommand{\qvalue}{\Q-value\xspace}
\newcommand{\qvalues}{\Q-values\xspace}

\newcommand{\econsistent}{\ensuremath{\epsilon\text{-consistent}}\xspace}
\newcommand{\econsistency}{\ensuremath{\epsilon\text{-consistency}}\xspace}
\newcommand{\N}[3]{\ensuremath{N({#1, #2, #3})}\xspace}
\newcommand{\Nbar}[2]{\ensuremath{\overline{N}({#1, #2})}\xspace}

\usepackage{authblk}

\title{Efficient Constraint Generation for \\ Stochastic Shortest Path Problems}
\author[1,2]{Johannes Schmalz}
\author[1]{Felipe Trevizan}
\affil[1]{Australian National University, Australia}
\affil[2]{Saarland University, Germany}
\date{} %

\begin{document}

\maketitle

\begingroup
\renewcommand\thefootnote{}

\makeatletter
\long\def\@makefntext#1{\noindent #1}
\makeatother

\footnotetext{\textit{Email addresses:} \texttt{johannes.schmalz@anu.edu.au} (Johannes Schmalz), \texttt{felipe.trevizan@anu.edu.au} (Felipe Trevizan).\vspace{4mm}}
\footnotetext{\noindent \textit{This manuscript has been accepted for publication in Artificial Intelligence. The final version of record is available at:} \url{https://doi.org/10.1016/j.artint.2026.104505}}
\endgroup

\begin{abstract}
\noindent
Stochastic Shortest Path problems (SSPs) are traditionally solved by computing each state's cost-to-go by applying Bellman backups.
A Bellman backup updates a state's cost-to-go by iterating through every applicable action, computing the cost-to-go after applying each one, and selecting a minimal action's cost-to-go.
State-of-the-art algorithms use heuristic functions; these give an initial estimate of costs-to-go, and lets the algorithm apply Bellman backups only to promising states, determined by low estimated costs-to-go.
However, each Bellman backup still considers all applicable actions, even if the heuristic tells us that some of these actions are too expensive, with the effect that such algorithms waste time on unhelpful actions.
To address this gap we present a technique that uses the heuristic to avoid expensive actions, by reframing heuristic search in terms of linear programming and introducing an efficient implementation of constraint generation for SSPs.
We present \cgilao, a new algorithm that adapts \ilao with our novel technique, and considers only \data{\(40\%\)} of \ilao's actions on many problems, and as few as \data{\(1\%\)} on some.
Consequently, \cgilao computes on average \data{\(3.5\times\)} fewer costs-to-go for actions than the state-of-the-art \ilao and \lrtdp, enabling it to solve problems faster an average of \data{\(2.8\times\) and \(3.7\times\)} faster, respectively.
\end{abstract}

\section{Introduction}

Planning problems are abstract tasks where an agent must move through a system of states.
To do so, the agent is given a set of actions that move it through the states according to some rules.
This type of problem is ubiquitous throughout theory and practice, and has been used to encode many real world problems, e.g., finding the shortest route on a map, coordinating warehouse robots to fulfil orders in an efficient and collaborative way~\cite{Koenig2023}, and producing fuel-efficient flight routes that take into account bad weather and complex aviation restrictions~\cite{Geisser2020:airbus}.
In the classical planning problem, actions move the agent between states in a deterministic way.
This paper focuses on probabilistic planning, in particular Stochastic Shortest Path Problems (SSPs)~\cite{Bertsekas1991:SSPs}.
SSPs generalise classical planning by making the effect of applying an action probabilistic: when the agent applies an action, it moves into its next state according to a probability distribution that is known a priori.
These have been used to model many practical problems, e.g., penetration testing (pentesting)~\cite{Hoffmann2015:pentesting}, ecological management~\cite{Peron2017:MDPs-for-mosquitos}, and management of hydroelectric and thermal power generation~\cite{White1985:ApplicationsOfMDPs}.

All state-of-the-art algorithms for optimally solving SSPs compute the optimal costs-to-go of relevant states.
The cost-to-go of a state \s can be understood as the expected cost that the agent incurs by starting at \s and repeatedly applying actions until it reaches a goal.
Respectively, the optimal cost-to-go gives the expected cost to reach a goal with an optimal policy, which minimises the expected costs.
The purpose of computing the optimal costs-to-go is that they induce an optimal policy: if the agent chooses its action for state \s greedily, i.e., it picks an action that minimises the optimal cost-to-go, then this induces an optimal policy.
The state-of-the-art algorithms are all based on one particular algorithm for computing optimal costs-to-go, called Value Iteration (VI)~\cite{Bellman57}.
VI is a dynamic programming algorithm that iteratively applies \emph{Bellman backups} to converge to the optimal costs-to-go.
A Bellman backup applied to state \s
\begin{itemize}
\item evaluates the cost-to-go after applying each applicable action \ac,
\item selects an action with minimal cost-to-go, and then
\item sets the cost-to-go of \s to the minimal action's cost-to-go.
\end{itemize}

The two state-of-the-art algorithms for solving SSPs, \ilao~\cite{Hansen2001:ilao} and \lrtdp~\cite{Bonet2003:lrtdp}, build on VI and can be orders of magnitude faster because they use heuristics to guide their search.
VI is a blind-search algorithm, which means it considers all states in the problem.
In contrast, \ilao and \lrtdp have access to heuristic functions that estimate the optimal costs-to-go for all states, which lets the algorithms focus on promising states with small heuristic values, and avoid unpromising states with large values.
Thus, heuristic-search algorithms apply Bellman backups only to a promising subset of states.
In practice, heuristic-search algorithms scale to much larger problems than blind-search algorithms because the promising subset of states is usually orders of magnitude smaller than the total reachable state space.\footnote{There are situations where heuristic search is not faster, e.g., if the heuristic gives misleading estimates of the optimal costs-to-go. However, with the heuristics and problems we consider, heuristic search is always significantly faster.}
So, these heuristic-search algorithms enjoy a significant performance advantage by restricting the number of states considered, but there is a gap because they do not restrict the actions to be considered.
These state-of-the-art algorithms are fundamentally built on Bellman backups, so they always consider all applicable actions for their states, but it is not clear that heuristic-search algorithms need to do this --- they can use their heuristics to avoid not only expensive states, but also actions that lead to expensive states.
This is an important issue because considering unneeded actions is not a one-time expense: in SSP planning, states must often be backed up many times before the algorithm converges.
In particular, cycles require many backups before converging to the correct cost-to-go.
Also, the search may reconsider the same states in many of their iterations, for example, \sZ is backed up in every iteration of \lrtdp and \ilao.
Each time the state is backed up we compute all its actions' costs-to-go, thus incurring many unneeded cost-to-go computations.

To make this concept of unneeded actions that waste computation more concrete, consider the gridworld example in \cref{fig:ilao-grid-problem}, which was used as a demonstration by~\citeAuthor{Hansen2001:ilao}.
The agent must navigate from the starting cell to the goal cell using the movement actions North (N), East (E), and South (S), each of which, assuming that the destination is unblocked, has a probability 0.5 of moving the agent to the expected cell, and a 0.5 probability of remaining in the same cell.
\Cref{fig:ilao-iter-2} shows which states and actions are considered by \ilao after 2 iterations.\footnote{In the original paper~\cite{Hansen2001:ilao} this figure is used as a demonstration of \lao, but in this case \lao behaves the same as \ilao.}
Observe that there are unneeded actions:
\begin{itemize}
\item the heuristic (shown in the bottom of each node) determines that cell 8 is too expensive and it will never be expanded, but nevertheless, the action S from cell 1 needs to be considered and its cost-to-go re-evaluated in each Bellman backup for cell 1; and
\item actions such as S and N in cell 2 do not progress the agent and will never have the lowest costs-to-go, but their costs-to-go are still evaluated in each Bellman backup on cell 2.
\end{itemize}
These actions are not useful for finding the optimal solution, and the heuristic values make this clear --- \ilao is wasting computational resources by considering these actions and re-evaluating their costs-to-go each time the relevant states are backed up.

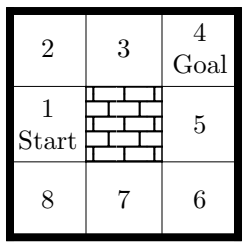
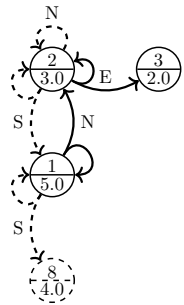
\begin{figure}[t]

\begin{subfigure}[t]{0.5\textwidth}
\centering
\begin{tikzpicture}[align=center]
\draw[pattern=bricks] (1,1) rectangle (2,2);
\draw[step=1cm,black,very thin] (0,0) grid (3,3);
\draw (0.5,1.5) node{1 \\ Start};
\draw (0.5,2.5) node{2};
\draw (1.5,2.5) node{3};
\draw (2.5,2.5) node{4 \\ Goal};
\draw (2.5,1.5) node{5};
\draw (2.5,0.5) node{6};
\draw (1.5,0.5) node{7};
\draw (0.5,0.5) node{8};
\draw[line width=3pt,-] (-0.05, 0) -- (3+0.05, 0);
\draw[line width=3pt,-] (-0.05, 3) -- (3+0.05, 3);
\draw[line width=3pt,-] (0, -0.05) -- (0, 3+0.05);
\draw[line width=3pt,-] (3, -0.05) -- (3, 3+0.05);
\end{tikzpicture}
\caption{Probabilistic navigation problem.}
\label{fig:ilao-grid-problem}
\end{subfigure}
\begin{subfigure}[t]{0.5\textwidth}
\centering
\scalebox{0.7}{
\begin{tikzpicture}
\begin{scope}[every node/.style={circle split, thick, draw, minimum size=0.6cm, inner sep=0.04cm}]
        \node (1) at (0, 0) {1 \nodepart{lower} \(5.0\)};
        \node (2) at (0, 2) {2 \nodepart{lower} \(3.0\)};
        \node (3) at (2, 2) {3 \nodepart{lower} \(2.0\)};
        \node[dashed] (8) at (0, -2) {8 \nodepart{lower} \(4.0\)};
\end{scope}
\begin{scope}[every edge/.style={draw, very thick}]
        \path [->] (1) edge[out=60, in=-60] node[right] {N} (2);
        \path [->] (1) edge[out=60, in=0, looseness=4] (1);
        \path [->] (2) edge[out=270+60, in=270-60] node[above] {E} (3);
        \path [->] (2) edge[out=270+60, in=20, looseness=4] (2);
        \path [dashed, ->] (2) edge[out=90-30, in=90+30, looseness=4] node[above] {N} (2);
        \path [dashed, ->] (2) edge[out=180+60, in=180-60] node[left] {S} (1);
        \path [dashed, ->] (2) edge[out=180+60, in=180, looseness=4] (2);
        \path [dashed, ->] (1) edge[out=180+60, in=180-60] node[left] {S} (8);
        \path [dashed, ->] (1) edge[out=180+60, in=180, looseness=4] (1);
\end{scope}
\end{tikzpicture}}
\caption{\ilao's states and actions after 2 iterations. The top number describes the cell index and the bottom number is the corresponding state's heuristic value.}
\label{fig:ilao-iter-2}
\end{subfigure}

\caption{A gridworld probabilistic navigation problem and the states and actions that \ilao considers on this problem after 2 iterations.
This example was taken from \citeAuthor{Hansen2001:ilao}.}

\end{figure}

There is an existing approach for removing such unneeded actions called \emph{action elimination}~\cite{Bertsekas1995}.
Action elimination enables an algorithm to prove that an action can never be part of the optimal solution, so that the action can be subsequently permanently removed from search, i.e., eliminated.
Action elimination requires a lower bound and an upper bound on an SSP's costs-to-go, and then determines that an action \(\tilde{\ac}\) may be eliminated if the lower bound on \(\tilde{\ac}\)'s cost-to-go is larger than the upper bound on another action \(\ac\)'s cost-to-go, i.e., \(\text{LB}(\tilde{\ac}) > \text{UB}(\ac)\).
Domain-independent lower bounds (a.k.a. admissible heuristics) can be automatically computed efficiently; however, we are not aware of any efficient ways to automatically derive informative domain-independent upper bounds.
Consequently, action elimination has not found much use in domain-independent planning, and \ftvi~\cite{Dai2009:ftvi} is the only algorithm we know that uses action elimination in the context of optimal heuristic planning for SSPs.
There are other algorithms that use upper bounds, but they do not explicitly use action elimination to remove actions from their search, e.g., BRTDP~\cite{McMahan2005:BRTDP}, FRTDP~\cite{Smith2006:FRTDP}, VPI-RTDP~\cite{Sanner2009:vpirtdp}, and IBLAO\(^*\)~\cite{Warnquist2010:IBLAO}.

So, the state-of-the-art algorithms for SSPs suffer from considering all applicable actions, even when it is clear that some actions are not needed for finding the optimal solution, and action elimination does not satisfactorily address this gap.
We address this open problem by exploiting connections between planning and operations research, and use the technique of constraint generation to intelligently select promising actions.
More concretely: our algorithm initially considers none of the actions, and only adds actions lazily when they are deemed promising by the current costs-to-go, which is determined with theory from constraint generation.
This contrasts with action elimination, which initially considers all actions and only removes them when it can prove they are suboptimal.
Moreover, action elimination requires lower and upper bounds, whereas we only need lower bounds.
Our approach improves upon state-of-the-art algorithms by saving on cost-to-go computations, which action elimination has not been able to do.
This is an important step in heuristic search, which lets algorithms use heuristics to select actions as well as states.

We introduce the novel algorithm \cgilao, which uses our technique of ignoring actions and adding them as required to generalise \ilao.
We confirm experimentally that \cgilao often considers only \data{40\%} of \ilao's actions, and in some cases considers as little as \data{\(1\%\)}.
Consequently, \cgilao computes \data{\(3.5\times\)} fewer \qvalues on average, and solves problems \data{\(2.8\times\) and \(3.7\times\)} faster than \ilao and \lrtdp, respectively.
In specific problems, the savings in \qvalues reach up to \data{\(80\times\)} resulting in a speedup of over \data{\(50\times\)}.

We give a breakdown of this paper's contents by section:
\begin{itemize}
\item \Cref{sec:background} gives the background for SSPs, value functions, Bellman backups, and Value Iteration.
\item \Cref{sec:ilao} describes \ilao, a state-of-the-art heuristic search algorithm for SSPs upon which we build.
We also present a novel way to view \ilao under the lens of linear programming.
\item \Cref{sec:cgilao} introduces our novel algorithm \cgilao, and give proofs of its correctness.
\item \Cref{sec:cgilao-expansions} describes different expansions methods that can be used for \cgilao.
\item \Cref{sec:related-work} discusses action elimination as an alternative way to cope with unneeded actions.
It also gives the context of existing constraint generation approaches in planning and relates \cgilao to the Partial Expansion \astar algorithm.
\item \Cref{sec:experiments} describes the setup and results of our experiments.
\item \Cref{sec:conclusion} and \cref{sec:future-work} are the conclusion and future work.
\end{itemize}

\clearpage

\section{Background}\label{sec:background}

Stochastic Shortest Path problems (SSPs)~\cite{Bertsekas1991:SSPs} are used to model problems where an agent must navigate through a state space, using actions with probabilistic effects that are known a priori:

\begin{definition}[Stochastic Shortest Path problem (SSP)~\cite{Bertsekas1991:SSPs}]
\label{def:ssp}
An SSP is defined by the tuple \sspTuple where:
\begin{itemize}
\item \Ss is the finite set of states;
\item \(\sZ \in \Ss\) is the agent's starting state, called the initial state;
\item \(\Sg \subsetneq \Ss\) is the set of goal states -- the agent finishes its task successfully if it reaches one of these.
We assume \(\Sg \neq \emptyset\) and \(\sZ \not\in \Sg\);
\item \(\A\) is the finite set of actions available to the agent and \(\A(\s) \subseteq \A\) denotes the actions applicable in state \s;
\item \(\pr(\s'|\s, \ac)\) indicates the probability of reaching \(\s'\) after applying action \ac to state \s;
\item \(\C(\s, \ac) \in \Rp\) gives the cost incurred when the agent applies action \ac in state \s.
\end{itemize}
\end{definition}

The set of states that may be reached after applying action \ac to state \s are called the successors of \s and \ac, and are given by \(\supp(\s, \ac) \definedas \{\s' \in \Ss : \pr(\s' | \s, \ac) > 0\}\).
SSPs are solved by policies, which are potentially partial functions \(\p: \Ss \to \A\) that map states \s that the agent may encounter onto actions \(\p(\s)\) that the agent should apply there.
\(\Ss^{\p,\s} \subseteq \Ss\) denotes the set of states that can be reached by following \p from \s, and is called the policy envelope of \p rooted at \s.
In this work we only consider policy envelopes rooted at \sZ, so we write \(\Ss^{\p} = \Ss^{\p,\sZ}\) and call \(\Ss^{\p}\) the policy envelope of \p.
It is desirable that the agent knows what to do in each state it encounters, and eventually reaches the goal, which leads to the following definitions:
\begin{definition}[Closed and Open Policies]
A policy \p is \textit{closed} \wrt \sZ if for all states in the policy envelope \(\s \in \Ss^{\p}\), either \(\p(\s)\) is defined or \s is a goal state.
It is open otherwise.
\end{definition}
\begin{definition}[Proper and Improper Policies]
A policy \p is \textit{proper} \wrt \sZ if, by following \p from \sZ, the agent will reach a goal with probability 1.
It is improper otherwise.
\end{definition}
Properness implies closedness.
For closed policies \p, their expected cost is the cumulative action cost incurred by the agent following \p from \sZ, over expectation.
An optimal policy \(\p^*\) is any proper policy that minimises the expected cost.
In this paper, we make two standard assumptions for SSPs:

\begin{assumption}[Reachability]\label{assump:reachability}
There exists at least one proper policy \wrt \sZ.
\end{assumption}

\begin{assumption}\label{assump:infinite-improper}
All improper policies have infinite expected cost.
\end{assumption}

Note that \(\A(\s) \neq \emptyset\) for all \(\s \in \Ss \setminus \Sg\) as a consequence of \cref{assump:infinite-improper}.
If this were not the case, then there could be a state \(\s \in \Ss \setminus \Sg\) with \(\A(\s) = \emptyset\) and we could construct an example where a policy \p reaches \s with probability 1 --- such \p's expected cost is finite even though \p is improper.
In our experiments, we consider SSPs that violate \cref{assump:reachability}, i.e., SSPs with dead ends; we address this by applying the fixed-penalty transformation of SSPs~\cite{trevizan17:mcmp}, which yields a new SSP that has no dead ends and is equivalent to the original SSP, provided that the penalty term is large enough.
Other approaches for relaxing our assumptions on SSPs, such as S\(^{3}\)P~\cite{Teichteil-Koenigsbuch2012:S3P}, can also be implemented without significant changes.

\textbf{Value Function.}
A value function \(\V : \Ss \to \Reals_{\geq 0}\) tries to encode the costs-to-go of each state, and we will use \V and costs-to-go interchangeably.
We use the following notations:
\begin{itemize}
\item Given \V, the \qvalue of \s and \ac is a common shorthand \(\Qsa \definedas \C(\s, \ac) + \sum_{\s' \in \A(\ac)} \pr(\s'|\s, \ac) \V(\s)\), which can be read as the expected cost-to-go from \s if the agent applies \ac.
\item Whenever we have a value function \(\V^x\) or \(\V_y\) we will refer to its associated \qvalues as \(\Q^x\) or \(\Q_y\) respectively.
\item We will write \(\V \leq \V'\) to compare value functions element-wise, i.e., \(\V \leq \V'\) iff \(\V(\s) \leq \V'(\s) \; \forall \s \in \Ss\), and similarly for other comparators, e.g., \(=, \geq\).
\end{itemize}
The optimal value function \(\V^*\) gives optimal costs-to-go for each state, and is the unique fixed-point solution to the following set of equations, called the Bellman equations:
\begin{align}\label{eqn:bellman-equation}
\V(\s) &= \begin{cases}
0 &\text{ if } \s \in \Sg \\
\min_{\ac \in \A(\s)} \Qsa &\text{ if } \s \in \Ss \setminus \Sg.
\end{cases}
\end{align}
Although an SSP may have multiple optimal policies, the optimal value function is unique.

Given a value function \V, its greedy policy is \(\p_{\V}(\s) \definedas \argmin_{\ac \in \A(\s)} \Qsa\), where any ties can be broken arbitrarily.
The set of all greedy policies for \(\V^*\), obtained by breaking ties in all possible ways, is precisely the set of optimal policies.
To avoid dealing with multiple greedy policies we make the following tie-breaking assumption:
\begin{assumption}[Tie-breaking for Greedy Policies] \label{assump:greedy-policy-tie-breaking}
Given \V, we ensure the greedy policy \(\p_{\V}\) is unique by breaking ties in \(\argmin_{\ac \in \A(\s)} \Qsa\) with arbitrary but reproducible tie-breaking rules.
An example of such a rule is lexicographical tie breaking.
\end{assumption}

Now that we have defined value functions and greedy policies, we briefly justify \cref{assump:infinite-improper}.
In \cref{fig:violate-assumption-infinite-improper} we present two SSPs that violate \cref{assump:infinite-improper}.
Notably, in \cref{fig:zero-cost-cycle} the state \(\s_1\) has a zero-cost action to itself, and in \cref{fig:trivial-dead-end} the state \(\s_1\) has no applicable actions.
Both SSPs have optimal value functions \(\V^*\) with \(\V^*(\sZ) = 1, \V^*(\s_1) = 0, \V^*(\s_g) = 0\) (arguably \(\V^*(\s_1) = -\infty\) for \cref{fig:trivial-dead-end}, depending on how we define the \(\min\) operator over the empty set), and the greedy policies \(\p_{\V^*}\) have \(\p_{\V^*}(\sZ) = \ac'_{0}\), making them improper, even though proper policies exists.
Thus, the optimal value function \(\V^*\) does not induce an optimal policy.
In general, for SSPs that violate \cref{assump:infinite-improper}, an optimal value function \(\V^*\) need not induce an optimal policy with \(\p_{\V^*}\), and moreover, \(\V^*\) need not even be unique, e.g., \(\V^*(\sZ) = 3, \V^*(\s_1) = 2, \V^*(\s_g) = 0\) satisfies the Bellman equations for \cref{fig:zero-cost-cycle}.
\Cref{assump:infinite-improper} avoids such issues, and prohibits both examples in \cref{fig:violate-assumption-infinite-improper}.
In practice, it is difficult to ensure an SSP satisfies \cref{assump:infinite-improper}, so it is common to make stronger assumptions that are easy to check.
In particular, it is common to require all states \s to have \(\A(\s) \neq \emptyset\) and restricting to strictly positive action costs to prohibit zero-cost cycles (and their generalisation, zero-cost end components), which is sufficient, but not necessary to ensure \cref{assump:infinite-improper} holds.
We take this practical approach, except when stated otherwise.

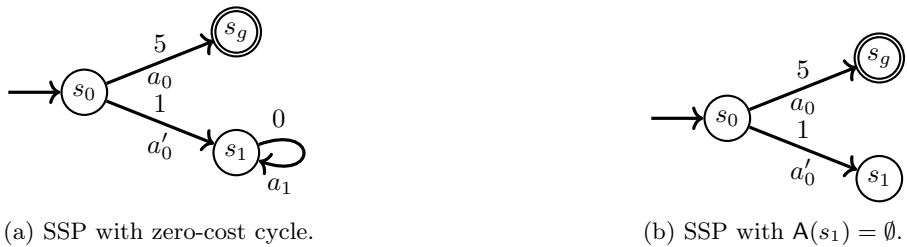
\begin{figure}[ht!]

\begin{subfigure}{0.5\textwidth}
\centering
\scalebox{1.0}{
\begin{tikzpicture}
\begin{scope}[every node/.style={circle, thick, draw, minimum size=0.6cm, inner sep=0.04cm}]
        \coordinate (-1) at (-1, 0);
        \node (0) at (0, 0) {\(\sZ\)};
        \node[double] (g) at (2, 0.8) {\(\s_g\)};
        \node (1) at (2, -0.8) {\(\s_1\)};
\end{scope}
\begin{scope}[every edge/.style={draw=black, very thick}]
        \path [->] (-1) edge (0);
        \path [->] (0) edge[] node[above] {5} node[below] {\(\ac_0\)} (g);
        \path [->] (0) edge[] node[above] {1} node[below] {\(\ac'_0\)} (1);
        \path [->] (1) edge[out=20, in=-20, looseness=10] node[above left=2mm and 1mm] {0} node[below left=2mm and 0.0mm] {\(\ac_1\)} (1);
\end{scope}
\end{tikzpicture}}
\caption{SSP with zero-cost cycle.}
\label{fig:zero-cost-cycle}
\end{subfigure}
\begin{subfigure}{0.5\textwidth}
\centering
\scalebox{1.0}{
\begin{tikzpicture}
\begin{scope}[every node/.style={circle, thick, draw, minimum size=0.6cm, inner sep=0.04cm}]
        \coordinate (-1) at (-1, 0);
        \node (0) at (0, 0) {\(\sZ\)};
        \node[double] (g) at (2, 0.8) {\(\s_g\)};
        \node (1) at (2, -0.8) {\(\s_1\)};
\end{scope}
\begin{scope}[every edge/.style={draw=black, very thick}]
        \path [->] (-1) edge (0);
        \path [->] (0) edge[] node[above] {5} node[below] {\(\ac_0\)} (g);
        \path [->] (0) edge[] node[above] {1} node[below] {\(\ac'_0\)} (1);
\end{scope}
\end{tikzpicture}}
\caption{SSP with \(\A(\s_1) = \emptyset\).}
\label{fig:trivial-dead-end}
\end{subfigure}

\caption{SSPs that violate \cref{assump:infinite-improper}, and consequently have optimal value functions that induce suboptimal policies.
One SSP violates it with a zero-cost cycle, and the other SSP violates it by having a state with no applicable actions.}
\label{fig:violate-assumption-infinite-improper}

\end{figure}

\textbf{Value Iteration (VI)}~\cite{Bellman57} is the foundational algorithm for finding a solution to the Bellman equations~\eqref{eqn:bellman-equation}, and thereby solving an SSP, that modern algorithms build on.
VI starts with an arbitrary value function \(\V^0\), and computes a sequence of value functions \(\V^1, \V^2, \dots\) which in the limit converge to \(\V^*\).
These value functions are computed by applying \textit{Bellman backups}:

\begin{definition}[Bellman Backup]
Given value functions \V and \(\V'\), a Bellman backup applied to state \s populates \(\V'(\s)\) with
\[
\V'(\s) \gets \min_{\ac \in \A(\s)} \Qsa
\]
and does not update \(\V'(\s')\) for any other state \(s' \neq \s\).
\end{definition}

To compute \(\V^{t+1}\) for \(t \in \Naturals_{>0}\), the classical version of VI applies Bellman backups on \(\V^{t}\) and \(\V^{t+1}\) for each \(\s \in \Ss\).
Asynchronous VI (also called Gauss-Seidel VI) computes \(\V^{t+1}\) by initialising it with \(\V^{t+1} = \V^{t}\) and then it only performs a Bellman backup for a single state in each \(t\); but importantly, it ensures that all states are backed up fairly, i.e., if the sequence is infinite then each state is updated infinitely often.
In both versions, for any choice of \(\V^0\), the sequence of value functions produced by VI will asymptotically converge to the optimal value function, i.e., \(\lim_{t \to \infty} \V^t = \V^*\) \cite{Bertsekas1995}.
In practice, VI requires a termination condition that is satisfied in finite time.
VI considers the difference between successive value functions, measured by the \textit{Bellman residual}, defined as \(\residual(\s) \definedas |\V^t(\s) - \min_{a \in \A(s)} \Q^t(\s,\ac)|\) and terminates when the residual between successive value functions is small over all states.
Formally, for a user-selected value \(\epsilon \in \Reals_{>0}\), VI terminates upon the following condition:

\begin{definition}[Global \econsistency]
\V is globally \econsistent for an SSP \ssp when \(\residual(\s) \leq \epsilon \; \forall \s \in \Ss\).
\end{definition}

Importantly, when VI is stopped with this condition its value function is only an approximate solution and for any \(\epsilon\) it is possible to construct an SSP so that VI terminates with \(\V\) whose greedy policy is not optimal.
In practice, for reasonably small \(\epsilon\), e.g., 0.0001, the greedy policy for globally \econsistent \(\V\) will be optimal or close to optimal.
Moreover, there is a bound on how much such a \V can deviate from \(\V^*\):

\begin{definition}\label{def:N}
Let
\(
\N{\s}{\V}{\ssp} = \max\{N^*(\s), N^{\p_{\V}}(\s)\}
\)
where \(N^{\p_{\V}}(\s)\) denotes the expected number of steps to reach some goal from \s by following the greedy policy \(\p_{\V}\), and \(N^*(\s)\) gives the maximum expected number of steps to reach a goal from \s over \ssp's optimal policies.
\end{definition}

\begin{theorem}[VI Error \cite{Mausam2012:MDPs}]\label{thm:vi-error}
Consider globally \econsistent \V.
Then,
\[
|\V(\s) - \V^*(\s)| \leq \epsilon \N{\s}{\V}{\ssp} \; \forall \s \in \Ss.
\]
\end{theorem}

Note that the error term \(\epsilon \N{\s}{\V}{\ssp}\) can get significantly larger than \(\epsilon\).
Nevertheless, \(\N{\s}{\V}{\ssp}\) can be bounded from above, and as \(\epsilon \to 0\) the error disappears and VI's \V converges to \(\V^*\).

A critical shortcoming of VI is that it applies Bellman backups to every single state \(\s \in \Ss\).
This means VI can not scale to large problems, which is especially problematic for compactly encoded problems, such as Probabilistic PDDL~\cite{Younes2005:ippc}, where the state space can grow exponentially with respect to the problem size.
An important insight is that not all states are needed to define an optimal policy, and we only need \(\V^*\) for the optimal policy's envelope.
Then, instead of considering global \econsistency, we consider \textit{\econsistency}:

\begin{definition}[\econsistency~\cite{Bonet2003:hdp}]\label{def:epsilon-consistency}
A value function \V is \econsistent for an SSP \ssp if \(\residual(\s) \le \epsilon \; \forall \s \in \Ss^{\p_{\V}}\).
\end{definition}
Note that \econsistency relies on \(\p_{\V}\), which is ambiguous when \V has multiple greedy policies; so, we fall back on \cref{assump:greedy-policy-tie-breaking} and assume that ties are broken such that \V has a single greedy policy.
This definition of \econsistency requires the residual to be small for states in the greedy policy's envelope, and states outside can have arbitrarily large residual.
Importantly, to ensure that \econsistent value functions produce optimal policies as \(\epsilon \to 0\), we need additional properties over the value functions.
In particular, we require \textit{admissibility}:

\begin{definition}[Admissible Value Function]
\V is admissible if \(\V \leq \V^*\).
\end{definition}

An admissible value function is simply a lower bound on the optimal value function.
Without admissibility, \econsistency does not produce optimal policies.
For example, consider \cref{fig:econsistency-weirdness}: the value function shown in the bottom of nodes is \econsistent \wrt the greedy policy with \(\p_{\V}(\sZ) = \ac_0\), but it is inadmissible, and fails to capture the optimal policy with \(\p^*(\sZ) = \ac'_0\).

\begin{figure}[ht!]
\centering
\scalebox{1.0}{
\begin{tikzpicture}
\begin{scope}[every node/.style={circle split, thick, draw, minimum size=0.6cm, inner sep=0.04cm}]
        \coordinate (-1) at (-1, 0);
        \node (0) at (0, 0) {\(\sZ\) \nodepart{lower} \(6\)};
        \node (1) at (2, 0.8) {\(\s_1\) \nodepart{lower} \(4\)};
        \node (2) at (2, -0.8) {\(\s_2\) \nodepart{lower} \(5\)};
        \node (3) at (4, 0.8) {\(\s_3\) \nodepart{lower} \(2\)};
        \node (4) at (4, -0.8) {\(\s_4\) \nodepart{lower} \(1\)};
        \node[double] (g) at (6, 0) {\(\s_g\) \nodepart{lower} \(0\)};
\end{scope}
\begin{scope}[every edge/.style={draw=black, very thick}]
        \path [->] (-1) edge (0);
        \path [->] (0) edge[] node[above] {2} node[below] {\(\ac_0\)} (1);
        \path [->] (0) edge[] node[above] {2} node[below] {\(\ac'_0\)} (2);
        \path [->] (1) edge[] node[above] {2} node[below] {\(\ac_1\)} (3);
        \path [->] (2) edge[] node[above] {1} node[below] {\(\ac_2\)} (4);
        \path [->] (3) edge[] node[above] {2} node[below] {\(\ac_3\)} (g);
        \path [->] (4) edge[] node[above] {1} node[below] {\(\ac_4\)} (g);
\end{scope}
\end{tikzpicture}}
\caption{An SSP with a value function that is \econsistent but inadmissible, and its greedy policy is suboptimal.}
\label{fig:econsistency-weirdness}
\end{figure}
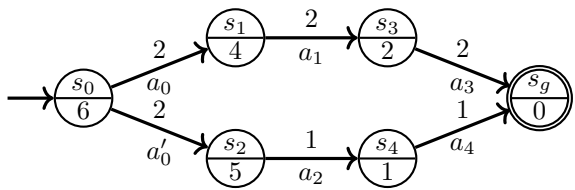

\clearpage

Another useful condition for value functions, which is stronger and implies admissibility, is \textit{monotonicity}:

\begin{definition}[Monotonic Value Function]
\V is monotonic if \(\V(\s) = 0 \; \forall \s \in \Sg\) and
\[\V(\s) \leq \min_{\ac \in \A(\s)} \Qsa \; \forall \s \in \Ss \setminus \Sg.\]
\end{definition}
An alternative name for monotonic value functions is consistent value functions, and it generalises the notion of consistent heuristics in deterministic planning.
Monotonic value functions get their name from the property that Bellman backups will never decrease a value function, i.e., the sequence \(\V^0, \V^1, \V^2, \dots\) obtained by applying Bellman backups is monotonically non-decreasing with \(\V^i \leq \V^{i+1} \; \forall i \in \Naturals\).
This property holds for any \(\V^{i+1}\) obtained from \(V^{i}\) via Bellman backups, for any arbitrary number of Bellman backups applied to an arbitrary set of states.
Conveniently, Bellman backups preserve both admissibility and monotonicity.
So, an admissible value function with arbitrarily many Bellman backups applied to an arbitrary set of states will yield another admissible value function, and similarly for monotonicity.
Thus, an algorithm that is initialised with an admissible value function and only uses Bellman backups to modify \V, has \(\V \leq \V^*\) as an invariant.

Heuristic-search algorithms that are initialised with an admissible value function \(\V^0 \leq \V^*\) and apply Bellman backups until \econsistency is satisfied, produce optimal policies.
Note that this contrasts with VI, which does not have any requirements over its initial value function \(\V^0\).
We present this result formally, via the \textit{\findAndRevise} framework:

\begin{definition}[\findAndRevise Algorithm~\cite{Bonet2003:hdp}]
A \findAndRevise algorithm applied to SSP \ssp uses a value function \V and consists of the steps
\begin{itemize}
\item \find a state \s in the greedy policy envelope \(\Ss^{\p_{\V}}\) \st \(\residual(\s) > \epsilon\),
\item \revise the found state by applying a Bellman backup to \s,
\item repeat these steps until no such state is found and \V is \econsistent \wrt \ssp.
\end{itemize}
\end{definition}

The \findAndRevise schema describes many heuristic search algorithms, including the state-of-the-art \ilao and \lrtdp.
\findAndRevise algorithms are optimal under the following conditions:

\begin{theorem}[Optimality of \findAndRevise \cite{Bonet2003:hdp,Mausam2012:MDPs}]\label{thm:find-and-revise-is-optimal}
Consider a \findAndRevise algorithm initialised with an admissible heuristic, and terminated when its \V is \econsistent for any given \(\epsilon \in \Reals_{>0}\).
As \(\epsilon \to 0\), its output \(\V(\s)\) approaches \(\V^*(\s)\) for \(\s \in \Ss^{\p_{\V}}\).
\end{theorem}

Other termination conditions are possible in place of \econsistency, e.g., \(\epsilon\)-optimality~\cite{Hansen2001:ilao,Hansen2015:UB-for-V}.
We do not consider these alternatives because they build on top of \econsistency, and are essentially automatic methods for selecting the parameter for \econsistency so that stronger guarantees can be provided.
Consequently, our \econsistent algorithms can be adapted to their termination conditions.

\section{\ilao}\label{sec:ilao}

\begin{algorithm}[!t]
{\small
\DontPrintSemicolon
  \caption{\ilao}\label{alg:ilao}
  \function{\ilao\((\ssp, \h, \epsilon)\)} {
    \(\partssp \gets\) partial SSP \(\langle \{\sZ\}, \sZ, \{\sZ\}, \emptyset, \pr, \C, \h \rangle\) \;
    \(\V \gets \text{Value Function initialised by } \h\) \;
    \(\currp \gets \text{candidate policy initialised as undefined everywhere (used to approximate \partp)}\) \;
    \Repeat{\(\fringe = \emptyset\) and \(\oldp = \currp\) and \(\residual \leq \epsilon\)} {
      \(\envelope \gets\) post-order DFS traversal of \currp from \sZ
      \; \label{line:ilao:dfs}
      \(\partssp, \currp \gets \expandFringes{}(\ssp, \partssp, \currp, \envelope)\)
       \; \label{line:ilao:call-to-expandFringes}
      \(\fringe \gets \Ss^{\currp} \cap (\partgoals \setminus \Sg)\) \; \label{line:ilao:updateF}
      \(\V, \residual, \currp, \oldp \gets \improvePolicy{}(\partssp, \currp, \envelope, \V, \fringe, \epsilon)\) \; \label{line:ilao:call-to-improvePolicy}
    }
    \Return \V
  }

  \function{\(\expandFringes{}(\ssp, \partssp, \currp, \envelope)\)}{
    \( \fringe \gets \envelope \cap (\partgoals \setminus \Sg) \) \;
    \For{\(\s_f \in \fringe\)} {
      \(\partssp \gets \addStateActions(\ssp, \partssp, \s_f, \A(\s_f))\) \; \label{line:expandFringes:addStateActions}
      \(\currp(\s_f) \gets \argmin_{\ac \in \partactions(\s)} \Qsa\) \;
    } \label{line:expandFringes:column-generation-begin}

    \Return \partssp, \currp \label{line:ilao:expandFringes-end}
    \label{line:expandFringes:column-generation-end}
  } \label{line:ilao:expandFringes-begin}

  \function{\(\addStateActions{}(\ssp, \partssp, \s, \A')\)}{
    \(\partgoals \gets \partgoals \setminus \{\s_f\}\) \;
    \(\partactions(\s) \gets \partactions(\s) \cup \A'\) \; \label{line:expandFringes:constraint-generation-end}
    \For{\(\ac \in \A'\)} {
      \(\partgoals \gets \partgoals \cup (\supp(\s, \ac) \setminus \partstates)\) \;
      \(\partstates \gets \partstates \cup \supp(\s, \ac)\) \;
    }
    \Return \partssp \;
  }

  \function{\(\improvePolicy{}(\partssp, \currp, \envelope, \V, \fringe, \epsilon)\)} {
    \(\oldp \gets \currp\) \;
    \Repeat{\(\fringe \neq \emptyset\) or \(\currp \neq \oldp\) or \(\residual \leq \epsilon\)} {
      \(\residual \gets 0\) \;
      \For {\(\s \in \envelope \setminus \partgoals\)} {
        \(\Q_{\min} \gets \min_{\ac \in \partactions(\s)} \Qsa\) \;
        \(\residual \gets \max(|V(\s) - \Q_{\min}|, \residual)\) \;
        \(\V(\s) \gets \Q_{\min}\) \;
      }
      \(\currp \gets \partp\) \; \label{line:ilao:improvePolicy:update-policy}

    }
    \Return \(\V, \residual, \currp, \oldp\)
  }
}
\end{algorithm}

In this section, we explain \ilao~\cite{Hansen2001:ilao} in detail, since it will be necessary for the insights and definition of \cgilao.
Then, we conclude the section by expressing \ilao in terms of linear programming, which is a novel way to look at the algorithm and opens the door to define \cgilao in \cref{sec:cgilao}.

In each step, \ilao considers a subproblem called a \emph{partial SSP}, which is a copy of the original SSP with a subset of states and actions.\footnote{Partial SSPs are called explicit graphs in the original paper~\cite{Hansen2001:ilao}.}
To define them, it is convenient to use SSPs with terminal costs:

\begin{definition}[SSP with Terminal Costs]\label{def:terminal-ssp}
An SSP with terminal costs is given by the tuple \(\ssp_{T} = \langle \Ss, \sZ, \Sg, \A, \pr, \C, T \rangle\).
All terms except \(T\) appear in the definition of SSPs, and are defined identically.
The new term \(T : \Sg \to \Reals_{\geq 0}\) is a terminal-cost function that applies the cost \(T(g)\) once, whenever the agent reaches the goal \(g \in \Sg\).
\end{definition}

SSPs with terminal costs are not more expressive than regular SSPs, and they are in fact equivalent.\footnote{SSPs with terminal costs can be converted to regular SSPs thus: add a new artificial goal state \(\overline{g}\), make all old goal states non-goals, and encode the one-time terminal cost for each old goal state \(g\) with a deterministic action that leads to \(\overline{g}\) with cost \(T(g)\).}
We use SSPs with terminal cost only because they are more convenient for the definition of partial SSPs:

\begin{definition}[Partial SSP]\label{def:partial-ssp}
Consider an SSP \(\ssp = \sspTuple\) and a heuristic \h.
A partial SSP for \ssp with \h is an SSP with terminal costs \(\partssp = \partssptuple\) where
\begin{samepage}
\begin{itemize}
\item \partstates and \partactions are subsets of their counterparts in \ssp, i.e., \(\partstates \subseteq \Ss\) and \(\partactions(\s) \subseteq \A(\s) \; \forall \s \in \partstates\),
\item \sZ is identical to \ssp,
\item \partgoals is a set of goals that satisfies \(\partgoals \subseteq \partstates, \Sg \cap \partstates \subseteq \partgoals\),
\item \pr and \C are the same as \ssp but restricted to \partstates and \partactions (we abuse notation and reuse the same symbols as in \ssp to emphasise that these functions are otherwise unchanged), and
\item \h is the terminal cost.
\end{itemize}
\end{samepage}
In \partssp, whenever the agent reaches an artificial goal \(\widehat{g} \in \partgoals\), it incurs a cost of \(\h(\widehat{g})\) once, and then terminates.
\end{definition}

The Bellman equations for partial SSPs are similar to \eqref{eqn:bellman-equation}, but the equations only consider the partial SSP's subset of states and actions, and the goals incur terminal costs:
\begin{align}\label{eqn:bellman-equation-with-terminal-costs}
\V(\s) = \begin{cases}
\h(\s) &\text{ if } \s \in \partgoals \\
\min_{\ac \in \partactions(\s)} \Qsa &\text{ if } \s \in \partstates \setminus \partgoals.
\end{cases}
\end{align}
We assume that \(\s \in \partstates, \ac \in \partactions(\s)\) implies that \(\supp(\s, \ac) \subseteq \partstates\), and therefore \Qsa is always defined in \eqref{eqn:bellman-equation-with-terminal-costs}.
For a partial SSP \partssp, we will use the following terminology
\begin{itemize}
\item $\partgoals \setminus \St$ is called the set of \emph{artificial goals},
\item \(\partstates \setminus \partgoals\) are called \emph{internal states},
\item \partactions are \emph{internal actions},
\item \partp is the greedy policy over \V restricted to \(\partstates \setminus \partgoals\), the non-goal states of \partssp.
\end{itemize}

\ilao~(\cref{alg:ilao}) is an iterative algorithm.
In each iteration, it works towards two opposing objectives:
\begin{enumerate}
\item \textbf{Optimise:} \ilao tries to find the optimal policy for \partssp by working towards solving the partial SSP \partssp's Bellman equations~\eqref{eqn:bellman-equation-with-terminal-costs}.
\item \textbf{Close:} \ilao expands any artificial goals reachable by its candidate policy \currp (called \emph{fringe states}) so that eventually \currp has no artificial goals and is closed \wrt the original SSP.
\end{enumerate}

To achieve a closed candidate policy (objective 2), \ilao performs a Depth-First Search (DFS) on \currp to find all fringe states (\cref{alg:ilao}~\cref{line:ilao:dfs}) and then expands them (\cref{alg:ilao}~\cref{line:ilao:call-to-expandFringes}).
To expand fringe state \s, \ilao adds \(\A(\s)\) to the partial SSP, and any states reachable from \s, i.e., \(\bigcup_{\ac \in \A(\s)} \supp(\s, \ac)\) are added to the partial SSP as new artificial goals, if they are not already in \partssp.
Expanding the fringe states in each iteration ensures that the candidate policy \currp will eventually be closed \wrt \sZ on the original SSP.

Simultaneously, \ilao works towards finding an optimal policy on its partial SSP (objective 1), by applying Bellman backups to all states in \currp's envelope \envelope (\cref{alg:ilao}~\cref{line:ilao:call-to-improvePolicy}).
The aim is to make \V \econsistent in order to show that \currp is optimal (\wrt \(\epsilon\)).
There is a subtlety regarding how \currp approximates \partp: during \expandFringes{} and \improvePolicy{} an update to \currp can cause it to reach internal states from which \currp leads to unaccounted fringe states, and these fringes are not tracked in \envelope.
\ilao recomputes \fringe in \cref{alg:ilao}~\cref{line:ilao:updateF}, immediately after \expandFringes{}, so there are no unaccounted fringe states there.
\improvePolicy{} addresses the issue by recomputing \(\currp \gets \partp\) after each pass of Bellman backups; recall that the issue only appears if \currp is updated so that it reaches unaccounted fringe states, so the main loop of \improvePolicy{} stops as soon as \(\currp \neq \oldp\), and \fringe will be recomputed in the main loop at \cref{alg:ilao}~\cref{line:ilao:updateF}.
Thus, this implementation of \ilao handles the issue by ensuring \partp's fringe states are in fact tracked in \fringe, and requires the policy to be recomputed in each loop of \improvePolicy{} (\cref{alg:ilao}~\cref{line:ilao:improvePolicy:update-policy}).
Our algorithm \cgilao deviates from this, as we explain in \cref{sec:cgilao}.

The objectives of optimising and closing interfere with each other, in the sense that expanding fringe states can break \V's \econsistency and create more work for \improvePolicy{}; and applying backups to the policy envelope can make \currp change in a way that introduces new fringe states that need to be expanded.
Arguably, this is what makes \ilao so efficient, because it is able to abandon working towards objective 1 if it conflicts with objective 2 and avoid redundant work, and vice versa.
Once \ilao achieves both objectives, it has a value function \V that is \econsistent \wrt the original SSP, since \currp has no artificial goals and \(\residual \leq \epsilon\).
We require that \ilao is initialised with an admissible heuristic, and since it only modifies \V with Bellman backups it preserves the invariant \(\V \leq \V^*\).
Consequently, we can refer to \cref{thm:find-and-revise-is-optimal}, and conclude that \ilao is optimal.

\subsection{\ilao as Linear Program}

Now, we present a new way to interpret \ilao in terms of Linear Programs (LPs) and constraint and variable generation.
As a starting point, it is well-known that SSPs can be solved by the LP presented in~\ref{lp:vi}.

\LP{11cm}{
\obj{\max_{\duals}}{V_{\sZ}}{lp:vi} \\
\st
\constr{\V_{\s} \leq \C(\s,\ac) +\!\sum_{\mathclap{\s' \in \Ss}}\pr(\s'|\s,\ac)\V_{\s'}}{\forall \s \in \Ss\!\setminus\!\St, \ac\!\in\!\A(\s)}{c:vi:consistency}\\
\constr{\V_g = 0}{\forall g \in \St}{c:vi:goal}
}\vspace{3mm}

\noindent
Each variable \(\V_{\s} \in \duals\) can be interpreted as \(\V(\s)\), and then the constraints \ref{c:vi:consistency} are requiring that \(\V(\s) \leq \min_{\ac \in \A(\s)} \Qsa\) for each \(\s \in \Ss\).
To emphasise this connection we write, when unambiguous, \(\V(\s)\) instead of \(\V_{\s}\), and \Qsa instead of the right-hand side of \s and \ac's constraint~\ref{c:vi:consistency}.
This LP encodes the Bellman equations~\eqref{eqn:bellman-equation}: its constraints together with the maximisation objective search for \(\V(\s)\) that is a maximal lower bound (infimum) on all \(\Qsa\), which is a way to express the minimum for a set of values (provided the infimum is in the set), i.e., we are encoding \(\V(\s) = \min_{\ac \in \A(\s)} \Qsa\).
Interestingly, the Bellman equations only need to be satisfied for states in the optimal policy envelope \(\Ss^{\p*}\).
In other words, given an optimal LP solution \(\duals^*\) that induces the policy \(\p^*\), the LP's constraints are active (tight) for the pairs \(\big( \s, \p^*(\s) \big)\) for all \(\s \in \Ss^{\p*}\) and may be inactive (slack) everywhere else.
We give an example of this in \cref{fig:vlp-not-accurate-everywhere}, where an optimal solution \(\duals^*\) is shown in the bottom of all nodes; this solution encodes the optimal policy \(\p^*\) with \(\p^*(\sZ) = \ac'_0\), but does not give the optimal cost-to-go for \(\s_1\) (it should be 4).

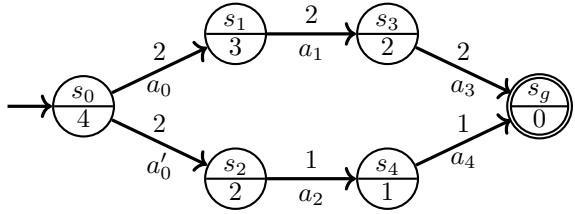
\begin{figure}[t!]
\centering
\scalebox{1.0}{
\newcommand{\scaleX}{1} %
\newcommand{\scaleY}{1.2} %
\begin{tikzpicture}
\begin{scope}[every node/.style={circle split, thick, draw, minimum size=0.8cm, inner sep=0.04cm}]
        \coordinate (-1) at (-1, 0);
        \node (0) at (\scaleX*0, \scaleY*0) {\(\sZ\) \nodepart{lower} \(4\)};
        \node (1) at (\scaleX*2, \scaleY*0.8) {\(\s_1\) \nodepart{lower} \(3\)};
        \node (2) at (\scaleX*2, \scaleY*-0.8) {\(\s_2\) \nodepart{lower} \(2\)};
        \node (3) at (\scaleX*4, \scaleY*0.8) {\(\s_3\) \nodepart{lower} \(2\)};
        \node (4) at (\scaleX*4, \scaleY*-0.8) {\(\s_4\) \nodepart{lower} \(1\)};
        \node[double] (g) at (\scaleX*6, \scaleY*0) {\(\s_g\) \nodepart{lower} \(0\)};
\end{scope}
\begin{scope}[every edge/.style={draw=black, very thick}]
        \path [->] (-1) edge (0);
        \path [->] (0) edge[] node[above] {2} node[below] {\(\ac_0\)} (1);
        \path [->] (0) edge[] node[above] {2} node[below] {\(\ac'_0\)} (2);
        \path [->] (1) edge[] node[above] {2} node[below] {\(\ac_1\)} (3);
        \path [->] (2) edge[] node[above] {1} node[below] {\(\ac_2\)} (4);
        \path [->] (3) edge[] node[above] {2} node[below] {\(\ac_3\)} (g);
        \path [->] (4) edge[] node[above] {1} node[below] {\(\ac_4\)} (g);
\end{scope}
\end{tikzpicture}}
\undef{\scaleX} %
\undef{\scaleY} %

\caption{An SSP where \ref{lp:vi}'s solution encodes an optimal policy, but does not give the optimal value function at \(\s_1\).}
\label{fig:vlp-not-accurate-everywhere}
\end{figure}

For an alternative interpretation of \ref{lp:vi}, observe that its constraints are requiring \duals to be a monotonic value function.
Monotonicity implies admissibility, so \ref{lp:vi} is looking for an admissible value function that maximises the value at \sZ, i.e., \(\V^*\).
Again, the objective only pushes against the constraints for states and actions encountered by \(\p^*\).

Similar to VI, \ref{lp:vi} can be initialised with any vector.
An LP solver will make the vector feasible and then improve its quality until it is optimal.
\ref{lp:vi}'s constraints require \duals to be a monotonic value function, so an admissible value function may not describe a feasible solution, e.g., the following admissible heuristic is not monotonic and therefore infeasible:
\[
\h(\s) = \begin{cases}
  \V^*(\s) &\text{ if } \s = \sZ \\
  0 &\text{ if } \s \neq \sZ.
\end{cases}
\]
So, given a non-monotonic value function, the LP solver will first make it monotonic, and then work towards making the value function \econsistent \wrt \sZ.

We now present our key insight that \ilao and its incremental growing of partial SSPs can be reinterpreted as an algorithm that solves \ref{lp:vi} with \emph{variable and constraint generation}~\cite{Bertsimas1997:LPs}.\footnote{Variable and constraint generation are also known column generation and the cutting plane method, respectively.}

\textbf{Variable generation} enables the solving of LPs with prohibitively many variables.
It works by considering a small LP with a subset of variables, and iteratively adding variables that are needed for the small LP to have the same solution as the original.
The original LP is called the Master Problem (MP), and the smaller LP with a subset of MP's variables is called the Reduced Master Problem (RMP).
An optimal solution for the RMP must also be feasible for the MP, but, assuming it is a maximisation problem, the RMP's solution may have a lower objective than the MP's optimal solution.
Variable generation iteratively adds variables to the RMP that have the potential to increase its objective, so that eventually the RMP's optimal solution is also optimal for the MP.
When solving an LP directly (with the simplex algorithm), the solver explicitly enumerates all unused (non-basic) variables and evaluates whether using it (making it basic) can improve the objective.
Variable generation also does this, but instead of enumerating all variables, it determines what variables to add by solving a \emph{pricing problem}.
Thus, variable generation is guaranteed to eventually terminate such that its RMP's optimal solution is also optimal for the MP.
Variable generation is often more efficient than solving the MP directly because the RMP often induces the optimal solution with significantly fewer variables than the original problem.
Note the similarity to heuristic search, where the problem can be solved by considering only a small subset of states.

\textbf{Constraint generation} is similar to variable generation, but solves LPs with excessively many constraints rather than variables.
It analogously works by considering small LPs with a subset of constraints, and adding constraints as needed to obtain an optimal solution for the MP.
The smaller LPs have fewer constraints and thereby relax the MP, so they are called \emph{relaxed LPs}.
For a maximisation problem, the solution to a relaxed LP must have its objective greater or equal to the MP's optimal objective, but, since some constraints are missing, the relaxed solution may not be feasible for the MP.
Constraint generation uses a \emph{separation oracle} to detect which of the MP's constraints are violated by the relaxed solution, and adds these to relaxed LP, so that its next solution is closer to being feasible for MP.
Once the separation oracle finds no violations and the relaxed solution is feasible for the MP, the solution is immediately optimal for the MP.

\textbf{Simultaneous variable and constraint generation} operates over LPs that are simultaneously RMPs with a subset of the MP's variables, and relaxations with a subset of the MP's variables.
Respectively, this requires a pricing problem to determine which variables to add, as well as a separation oracle that detects which of the MP's constraints are violated.
Solutions for intermediate LPs can be difficult to interpret, because they need not give lower bounds on the optimal solution (since variables may be missing) and they need not give upper bounds (because constraints may be missing).
However, as soon as we have a solution where the pricing problem does not flag any new variables and the separation oracles does not detect any violations, we can be sure that the solution is optimal for the MP.
To justify this fact, consider a MP with variables \(X\) and constraints \(\mathcal{C}\), and a relaxed RMP with variables \(\widehat{X} \subseteq X\) and constraints \(\widehat{\mathcal{C}}\) with an optimal solution \(\widehat{\lpxs}\).
If the pricing problem does not add any variables to the relaxed RMP, then we know that \(\widehat{\lpxs}\) is an optimal solution for the LP with all variables \(X\) but the subset of constraints \(\widehat{\mathcal{C}} \subseteq \mathcal{C}\).
If \(\widehat{\lpxs}\) additionally does not violate any constraints in \(\mathcal{C}\), then it must immediately be an optimal solution for the MP.

We can interpret \ilao as a variable and constraint generation algorithm over \ref{lp:vi}.
Its partial SSP induces \ref{lp:rmp-vi} which is simultaneously an RMP and a relaxation of for \ref{lp:vi}, because it has the subset of variables~\(\{\V_{\s} : \s \in \partstates \setminus \partgoals\}\) and only requires monotonicity constraints for actions in \partactions (\ref{c:rmp-vi:consistency}).
Recall that \ilao's partial SSPs are  SSPs with terminal costs of \h at \partgoals~(\cref{def:partial-ssp}), so we include the constraints \ref{c:rmp-vi:goal} for artificial goals, in order to capture the terminal costs.
Observe that \(\V_{g}\) are constants for \(g \in \partgoals\), similarly to \ref{lp:vi}'s \(\V_{g} = 0\) for \(g \in \Sg\).
Strictly, due to these artificial goal constraints, \ref{lp:rmp-vi} is not a true relaxation of \ref{lp:vi}, since these constraints do not appear in \ref{lp:vi}.
However, it can still be considered a relaxation because these constraints only affect the ``boundary condition'' of constants, and since we use an admissible heuristic \h, we know that an optimal solution for \ref{lp:rmp-vi} remains a lower bound an optimal solution for \ref{lp:vi}.

\LP{11cm}{
\obj{\max_{\duals}}{V_{\sZ}}{lp:rmp-vi} \\
\st
\constr{\V_{\s} \leq \C(\s,\ac) +\!\sum_{\mathclap{\s' \in \Ss}}\pr(\s'|\s,\ac)\V_{\s'}}{\forall \s \in \partstates\!\setminus\!\partgoals, \ac\!\in\!\partactions(\s)}{c:rmp-vi:consistency}\\
\constr{\V_g = \h(g)}{\forall g \in \partgoals}{c:rmp-vi:goal}
}\vspace{3mm}

In terms of \ref{lp:rmp-vi}, \ilao generates new variables for its RMP (partial SSP) by expanding \currp's fringe states \(\s_f\), and inserting them as variables \(\V_{\s_f}\), replacing the previous constant \(\V_{\s_f} = \h(\s_f)\).
Note that this procedure is guided by the heuristic \h in the sense that \currp takes \h into account.
The separation oracle for \ilao's constraint generation, i.e., the mechanism for detecting constraints in the MP that are violated, is very simple: whenever a variable \(\V_{\s}\) is added, all constraints associated with \s are immediately added.
Adding all constraints trivially ensures that the relaxed LP will not violate any of the MP's constraints, but \ilao does not get any savings from leaving out non-violated constraints.

We illustrate \ilao's variable and constraint generation with a small example.
Consider the SSP in \cref{fig:ssp-for-ilao-as-constraint-and-variable-generation} with the heuristic values shown in the bottom of each node.
Then, \ilao's iterations will look like this:
\begin{enumerate}[label=Iter. \arabic*, leftmargin=*]
\item \(\partstates = \{\sZ\}, \partgoals = \{\sZ\}\) --- \(\currp = \{\}\) --- expand \(\sZ\);
\item \(\partstates = \{\sZ, \s_1, \s_2, \s_3\}, \partgoals = \{\s_1, \s_2, \s_3\}\) --- \(\currp = \{\sZ \mapsto \ac'_0\}\) --- expand \(\s_2\);
\item \(\partstates = \{\sZ, \s_1, \s_2, \s_3, \s_g\}, \partgoals = \{\s_1, \s_3, \s_g\}\) --- \(\currp = \{\sZ \mapsto \ac_0\}\) --- expand \(\s_1\);
\item \(\partstates = \{\sZ, \s_1, \s_2, \s_3, \s_g\}, \partgoals = \{\s_3, \s_g\}\) --- \(\currp = \{\sZ \mapsto \ac_0, \s_1 \mapsto \ac_1, \s_2 \mapsto \ac_2\}\) --- done.
\end{enumerate}

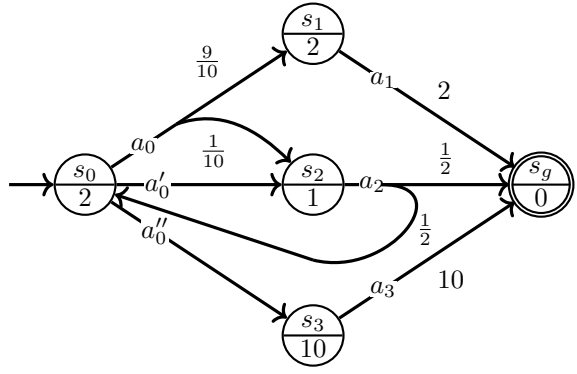
\begin{figure}[t]
\centering
\scalebox{1.0}{
\newcommand{\scaleX}{1.5} %
\newcommand{\scaleY}{1.0} %
\begin{tikzpicture}
\begin{scope}[every node/.style={circle split, thick, draw, minimum size=0.8cm, inner sep=0.04cm}]
  \coordinate (-1) at (-1, 0);
	\node (0) at (0, 0) {\(\sZ\) \nodepart{lower} \(2\) };
	\node (1) at (2*\scaleX, 2*\scaleY) {\(\s_1\) \nodepart{lower} \(2\) };
	\node (2) at (2*\scaleX, 0*\scaleY) {\(\s_2\) \nodepart{lower} \(1\) };
	\node (3) at (2*\scaleX, -2*\scaleY) {\(\s_3\) \nodepart{lower} \(10\) };
	\node[double] (g) at (4*\scaleX, 0*\scaleY) {\(\s_g\) \nodepart{lower} \(0\) };
\end{scope}
\begin{scope}
  \coordinate (fork-a00) at ($(0)!0.4!(1)$);
  \coordinate (fork-a2) at ($(2)!0.3!(g)$);
  \coordinate (curve-a2) at ($(2)!0.5!(3)$);
\end{scope}
\begin{scope}[every edge/.style={draw=black, very thick}]
\path [->] (-1) edge[] (0);
\path [-] (0) edge[] node[midway, fill=white, inner sep=0pt] {\(\ac_0\)} (fork-a00);
\path [->] (fork-a00) edge[] node[above left] {\({\frac{9}{10}}\)} (1);
\path [->] (fork-a00) edge[out=20] node[below left] {\({\frac{1}{10}}\)} (2);
\path [->] (0) edge[] node[near start, fill=white, inner sep=0pt] {\(\ac'_0\)} (2);
\path [->] (0) edge[] node[near start, fill=white, inner sep=0pt] {\(\ac''_0\)} (3);
\path [->] (1) edge[] node[near start, fill=white, inner sep=0pt] {\(\ac_1\)} node[above right] {\({2}\)} (g);
\path [-] (2) edge[] node[near end, fill=white, inner sep=0pt] {\(a_2\)} (fork-a2);
\path [->] (fork-a2) edge[] node[above] {\({\frac{1}{2}}\)} (g);
\path [-] (fork-a2) edge[out=0, in=-14, looseness=2] node[right] {\({\frac{1}{2}}\)} (curve-a2);
\path [->] (curve-a2) edge[] (0);
\path [->] (3) edge[] node[near start, fill=white, inner sep=0pt] {\(\ac_3\)} node[below right] {\({10}\)} (g);
\end{scope}
\end{tikzpicture}
\undef{\scaleX} %
\undef{\scaleY} %
}

\caption{Example SSP that we solve with \ilao under the lens of constraint and variable generation.
\(\h(\s)\) are given in the bottom of each node.}
\label{fig:ssp-for-ilao-as-constraint-and-variable-generation}
\end{figure}

\begin{figure}[t]
\begin{minipage}{\textwidth}
\begin{framed}
\vspace{-0.4cm}
\begin{equation*}
\begin{split}
&\text{Iter. 2} \\
&\underset{\V_{\sZ}}{\max} \V_{\sZ} \text{ s.t. } \\
&\V_{\sZ} \leq \Q(\sZ, \ac_0) \\
&\V_{\sZ} \leq \Q(\sZ, \ac'_0) \\
&\V_{\sZ} \leq \Q(\sZ, \ac''_0) \\
& \\
& \\
&\V_{\s_1} = \h(\s_1) = 2 \\
&\V_{\s_2} = \h(\s_2) = 1 \\
&\V_{\s_3} = \h(\s_3) = 10 \\
& \\
\end{split}
\hspace{2cm}
\begin{split}
&\text{Iter. 3} \\
&\underset{\V_{\sZ}, \V_{\s_2}}{\max \V_{\sZ}} \text{ s.t. } \\
&\V_{\sZ} \leq \Q(\sZ, \ac_0) \\
&\V_{\sZ} \leq \Q(\sZ, \ac'_0) \\
&\V_{\sZ} \leq \Q(\sZ, \ac''_0) \\
& \\
&\V_{\s_2} \leq \Q(\s_2, \ac_2) \\
&\V_{\s_1} = \h(\s_1) = 2 \\
& \\
&\V_{\s_3} = \h(\s_3) = 10 \\
&\V_{\s_g} = 0 \\
\end{split}
\hspace{2cm}
\begin{split}
&\text{Iter. 4} \\
&\underset{\V_{\sZ}, \V_{\s_1}, \V_{\s_2}}{\max \V_{\sZ} \text{ s.t. }} \\
&\V_{\sZ} \leq \Q(\sZ, \ac_0) \\
&\V_{\sZ} \leq \Q(\sZ, \ac'_0) \\
&\V_{\sZ} \leq \Q(\sZ, \ac''_0) \\
&\V_{\s_1} \leq \Q(\s_1, \ac_1) \\
&\V_{\s_2} \leq \Q(\s_2, \ac_2) \\
& \\
& \\
&\V_{\s_3} = \h(\s_3) = 10 \\
&\V_{\s_g} = 0 \\
\end{split}
\end{equation*}
\vspace{-0.35cm}
\end{framed}
\end{minipage}
\caption{RMPs corresponding to \ilao's partial SSPs at the start of iterations 2--4 from the example in \cref{fig:ssp-for-ilao-as-constraint-and-variable-generation}.}
\label{fig:ilao-rmps-example}
\end{figure}

The LPs for \ilao's partial SSPs at the starts of iterations 2--4 are shown in \cref{fig:ilao-rmps-example}.
For each partial SSP \partssp, its LP has variables for each internal state \(\s \in \partstates \setminus \partgoals\), and the goals \partgoals correspond to constants, e.g., \(\V_{\s_g} = 0\).
Thus, when \ilao expands a state, we generate the corresponding variable for the LP, as well as the constraints for its outgoing actions.
Note that the constants for \(\s \in \partgoals\) disappear when \s gets expanded, and they are replaced with variables.

We point out that the solution to the final partial SSP has
\begin{itemize}
\item \(\V_{\sZ} = 3.05\)
\item \(\V_{\s_1} = 2.0\)
\item \(\V_{\s_2} = 2.52\).
\end{itemize}
Importantly, notice that the constraint \(\V_{\sZ} \leq \Q(\sZ, \ac''_0)\) is trivially satisfied because \(3.05 \leq 11\); i.e., the constraint \(\V_{\sZ} \leq \Q(\sZ, \ac''_0)\) is not needed, and corresponds to an unneeded action that could be left out of the partial SSPs because \(\h(\s_3)\) is large and deems \(\s_3\) unpromising.
This is a consequence of \ilao's na{\"i}ve separation oracle that adds constraints for all outgoing actions as states get expanded.

This interpretation extends the currently understood connections between planning algorithms and operations research, and highlights the shortcoming of Bellman backups that consider all actions because these correspond to a na{\"i}ve separation oracle that adds all constraints.
In the next section, we address this shortcoming, and introduce an efficient separation oracle that only adds violated constraints.

\section{\cgilao}\label{sec:cgilao}

\begin{algorithm}[]
{\small
\DontPrintSemicolon
  \caption{\cgilao}\label{alg:cg-ilao}
  \function{\cgilao\((\ssp, \h, \epsilon, \incDecTol)\)} {

    \(\partssp \gets\) partial SSP \(\langle \{\sZ\}, \sZ, \{\sZ\}, \emptyset, \pr, \C, \h \rangle\) \;
    \(\V \gets \text{value function initialised by } \h\) \;
    \(\currp \gets \text{candidate policy initialised as undefined everywhere (used to approximate \partp)}\) \;
    \(\colsToCheck \gets \emptyset\) \;
    \Repeat{\(\fringe = \emptyset\) and \(\oldp = \currp\) and \(\residual \leq \epsilon\)} {
      \label{line:cgilao:main-loop-start}
      \(\envelope \gets\) post-order DFS traversal of \currp from \sZ \; \label{line:cgilao:dfs}
      \(\partssp, \currp, \envelope \gets \partiallyExpandFringes{}(\ssp, \partssp, \currp, \envelope, \V)\) \; \label{line:cgilao:call-to-partiallyExpandFringes}
      \(\fringe \gets \Ss^{\currp} \cap (\partgoals \setminus \Sg)\) \;
      \(\V, \residual, \currp, \oldp, \colsToCheck \gets \CGimprovePolicy{}(\ssp, \partssp, \currp, \envelope, \V, \fringe, \colsToCheck, \epsilon, \incDecTol)\) \; \label{line:cgilao:call-to-CGimprovePolicy}
      \(\V, \residual, \colsToCheck, \partssp \gets \fixViolatedConstrs{}(\ssp, \partssp, \currp, V, \colsToCheck, \residual, \epsilon)\) \; \label{line:cgilao:call-to-fixViolatedConstrs}
      \label{line:cgilao:main-loop-end-minus-one}
    } \label{line:cgilao:whileCondition}
      \label{line:cgilao:main-loop-end}
    \Return \V
  }

  \function{\(\partiallyExpandFringes{}(\ssp, \partssp, \currp, \envelope, \V)\)}{
    \While{\(\exists \s_f \in \envelope \cap (\partgoals \setminus \Sg)\)} {
    \label{line:cgilao:partiallyExpandFringes:main-loop-start}
      \(\A' \gets \{ \ac \in \A(\s_f) : \Q(\s_f, \ac) = \min_{\ac' \in \A(\s_f)} \Q(\s_f, \ac')\}\) \tcc*{Select greedy actions \(\argmin_{\ac \in \A} \Qsa\)}
      $\partssp \gets \addStateActions(\ssp, \partssp, \s_f, \A')$ \tcc*{Note: \(\s_f \not\in \partgoals\) after \addStateActions}
      \(\currp(\s_f) \gets \text{some greedy action } \ac \in \partactions(\s)\) \;
      \For{\(\s_{\text{int}} \in \supp(\s_f, \currp(\s_f)) \setminus \partgoals\)}{
        \(\envelope \gets \envelope\) updated with post-order DFS traversal of \currp from \(\s_{\text{int}}\) (no duplicate states) \;
      }
    } \label{line:cgilao:partiallyExpandFringes:main-loop-end}
    \Return \partssp, \currp, \envelope \;
  }

  \function{\(\CGimprovePolicy{}(\ssp, \partssp, \currp, \envelope, \V, \fringe, \colsToCheck, \epsilon, \incDecTol)\)} {
    \(\oldp \gets \currp\) \;
    \Repeat{\(\fringe \neq \emptyset\) or \(\currp \neq \oldp\) or \(\residual \leq \epsilon\)}{
      \label{line:cgilao:CGimprovePolicy:main-loop-start}
      \(\residual \gets 0\) \;
      \For {\(\s \in \envelope \setminus \partgoals\)} {
        \(\Q_{\min} \gets \min_{\ac \in \partactions(\s)} \Qsa\) \;
        \If {\(\Q_{\min} - \V(\s) > \incDecTol\)} {
          \label{line:cgilao:improvePolicy:check-violated-constraints-begin}
          \(\colsToCheck \gets \colsToCheck \cup \successors(\s, \ssp, \partssp)\) \; \label{line:cgilao:v-increases}
        }
        \ElseIf {\(V(\s) - \Q_{\min} > \incDecTol\)} {
          \(\colsToCheck \gets \colsToCheck \cup \predecessors(\s, \ssp, \partssp)\) \; \label{line:cgilao:v-decreases} \label{line:cgilao:improvePolicy:check-violated-constraints-end}
        }
        \(\residual \gets \max(|\V(\s) - \Q_{\min}|, \residual)\) \;
        \(\V(\s) \gets \Q_{\min}\) \;
        \(\currp(\s) \gets \text{some greedy action } \ac \in \partactions(\s) \text{ s.t.\ } \Qsa = \Q_{\min}\) \; \label{line:cgilao:improvePolicy:update-policy}

      }

    } \label{line:cgilao:improvePolicy:repeatCondition}
      \label{line:cgilao:CGimprovePolicy:main-loop-end}
    \Return \(\V, \residual, \currp, \oldp, \colsToCheck\) \;
  }

  \function{\(\fixViolatedConstrs{}(\ssp, \partssp, \currp, \V, \colsToCheck, \residual, \epsilon)\)} {
    \(\colsToCheck' \gets \emptyset\) \;
    \For{\((\s,\ac) \in \colsToCheck \text{ \st } \V(\s) >  \Q(\s,\ac) + \epsilon\)} {
      \If{\(\ac \not\in \partactions(\s)\)} {
        \(\partactions(\s) \gets \partactions(\s) \cup \{\ac\}\) \;
        \(\partgoals \gets \partgoals \cup (\supp(\s, \ac) \setminus \partstates)\) \;
        \(\partstates \gets \partstates \cup \supp(\s, \ac)\) \;
      }
      \(\residual \gets \max(V(\s) - \Qsa, \residual)\) \; \label{line:cgilao:fixViolatedConstrs:track-residual}
      \(\V(\s) \leftarrow \Qsa\) \; \label{line:cgilao:fixViolatedConstrs:fix-violation}
      \(\currp(\s) \gets \ac\) \; \label{line:cgilao:fixViolatedConstrs:update-policy}
      \(\colsToCheck' \gets \colsToCheck' \cup \predecessors(\s, \ssp, \partssp)\) \; \label{line:cgilao:fixViolatedConstrs:check-violated-constraints}
    }
    \Return \(\V, \residual, \colsToCheck', \partssp\)
  }
}
\end{algorithm}

All previous algorithms based on VI update a state's value function \(\V(\s)\) by considering all applicable actions actions.
In \ilao, this happens because all its states are either unexpanded and have no available actions \(\partactions(\s) = \emptyset\), or fully expanded with all applicable actions available \(\partactions(\s) = \A(\s)\); \ilao is not able to partially expand states.
We address this shortcoming by allowing states to be partially expanded so that backups in the partial SSP only consider a subset of \(\A(\s)\).
To this end, we first define which actions can be left out of the partial SSP without affecting optimality.

\begin{definition}[Inactive Action]
Consider an SSP \ssp, its partial SSP \partssp, a value function \V, and a state $\s \in \partstates$.
An action $\ac \in \A(\s) \setminus \partactions(\s)$ is inactive in state \s if $\min_{\widehat{\ac} \in \partactions(\s)} \Q(\s, \widehat{\ac}) < Q(s,a).$
\end{definition}

If an action \ac is inactive for \s, then \Qsa can not reduce \(\V(\s)\), and we can still obtain \(\V^*\) if \ac is ignored by our Bellman backups.
Of course \ac may become active later, but while it is inactive, it can be safely left out of the partial SSP.
To understand inactive actions in terms of linear programming, consider the \ref{lp:rmp-vi} associated with \partssp and let \V be a solution for it.
Each state-action pair \((\s, \ac)\) corresponds to the consistency constraint \ref{c:rmp-vi:consistency}.
An inactive action \ac for state \s does not have an associated constraint in the LP since \(\ac \not\in \partactions(\s)\), and importantly, if the constraint were added it would be inactive (also called loose).
Therefore, adding the constraint does not affect the solution and only introduces redundant work.

To exploit this insight, we generalise \ilao so it can ignore inactive actions in each of its partial SSPs and partially expand states, i.e., states may have a non-empty strict subset of applicable actions available in the partial SSP, \(\emptyset \neq \partactions(\s) \subsetneq \A(\s)\).
Then, we introduce a mechanism that efficiently adds actions that are not inactive as they appear.
In terms of linear programming, we are using the constraint generation framework to leave unneeded constraints out of the LP, and use a separation oracle to identify constraints that may be needed to encode the optimal solution.
This yields a new dynamic programming algorithm, which, in honour of this tight connection to constraint generation, we call Constraint-Generation \ilao (\cgilao).

\cgilao~(\cref{alg:cg-ilao}) is similar to \ilao~(\cref{alg:ilao}), but with the key difference that in its expansion phase (\cref{alg:cg-ilao} \cref{line:cgilao:call-to-partiallyExpandFringes}), \cgilao calls \partiallyExpandFringes{} and expands a state with only the greedy actions \wrt \V, and not all applicable actions.
This introduces two issues that can affect \cgilao's correctness.
\begin{enumerate}
\item Partial expansion may ignore an optimal action if \V is inaccurate.
This is clearly problematic because \cgilao can not possibly find the optimal policy if its partial SSP does not contain the optimal actions.
So, \cgilao requires a mechanism that can detect and add optimal actions if the relevant state's partial expansion missed them.
\item \V is evaluated \wrt the partial SSP, so if an optimal action \ac is missing, then it is not taken into account by \V (this can cause \(\V \geq \V^*\)).
If the missing action is later added, then \V does not automatically reflect that \ac has been made available, so we must ensure that \ac's improvements to \V are fully propagated.
\end{enumerate}
Elegantly, it turns out that both of these problems are instances of constraint violation, and can therefore be addressed by detecting constraint violations with a separation oracle and then enforcing that the constraints become satisfied.

\textbf{How can we find constraint violations efficiently?}
A na{\"i}ve separation oracle considers all constraints missing in the partial SSP, i.e., the constraints for \(\s \in \partstates, \ac \in \A(\s) \setminus \partactions(\s)\), and evaluates whether they are violated by checking if \(\V(\s) > \Qsa\).
This means we compute \qvalues for each external action and therefore save no work compared to adding them to the partial SSP in the first place.
The key realisation that lets us save \qvalue computations is that non-violated constraints remain non-violated unless \V is updated in a particular way.
Our efficient separation oracle exploits this by computing a set of constraints that might be violated, and then checking violations only in this set.
Suppose there are currently no constraint violations and \(\V(\s)\) is assigned \(\min_{\ac \in \partactions(\s)} \Qsa\).
Then, the separation oracle determines that the constraints associated with \s may become violated according to the following three cases:

\begin{enumerate}
\item{\bld{\(\V(\s)\) stays the same.}}
There can be no new constraint violations.
\item{\bld{\(V(\s)\) increases.}}
The constraints \(\V(\s) \leq \Q(\s, \ac')\) may be violated for \((\s, \ac') \in \successors(\s, \ssp, \partssp)\) where
\[\successors(\s, \ssp, \partssp) \definedas \left\{(\s, \ac) : \ac \in \A(\s) \setminus \partactions(\s)\right\}.\]
Note that we only have to consider external successor actions, and not internal ones \((\s, \ac')\) with \(\ac' \in \partactions(\s)\), because \(\V(\s) \gets \min_{\ac \in \partactions(\s)} \Qsa \leq \Q(\s, \ac')\), and therefore the constraint associated with \((\s, \ac')\) must be satisfied already.
\item{\bld{\(V(\s)\) decreases.}}
The only constraints that may be violated are \(\V(\s') \leq \Q(\s', \ac')\) for \((\s', \ac') \in \predecessors(\s, \ssp, \partssp)\) where
\[\predecessors(\s, \ssp, \partssp) \definedas \left\{(\s', \ac') :  \s' \in \partstates, \ac' \in \A(\s'), \s \in \supp(\s', \ac')\right\}.\]
In this case, we must consider internal actions as well as external ones.
However, we do not consider external predecessor states \(\s' \in \Ss \setminus \partstates\), because such states will eventually be expanded and handled that way if their constraint violations persist.
\end{enumerate}

In the pseudocode, these checks are performed whenever \V is updated, and all potential violations are stored in \colsToCheck:
\begin{itemize}
\item if \(\V(\s)\) increases then \(\colsToCheck \gets \colsToCheck \cup \successors(\s, \ssp, \partssp)\) (\cref{alg:cg-ilao}~\cref{line:cgilao:v-increases});
\item if \(\V(\s)\) decreases, then \(\colsToCheck \gets \colsToCheck \cup \predecessors(\s, \ssp, \partssp)\) (\cref{alg:cg-ilao}~\cref{line:cgilao:v-decreases} and \cref{line:cgilao:fixViolatedConstrs:check-violated-constraints}).
\end{itemize}
At the end of each iteration of its main loop, \cgilao calls \fixViolatedConstrs{} to check all potentially violated constraints \colsToCheck.
If a constraint in \colsToCheck is not violated, it is simply discarded.
If a constraint \((\s, \ac)\) is indeed violated, then \fixViolatedConstrs{} fixes it by setting \(\V(\s) \gets \Qsa\) (\cref{alg:cg-ilao} \cref{line:cgilao:fixViolatedConstrs:fix-violation}), and then removes \((\s, \ac)\) from \colsToCheck.
Importantly, this update to \V may create new constraint violations, so we must apply the same rules as before and track any new potential constraint violations in \colsToCheck for the next iteration.
This ensures that \colsToCheck is always a superset of \cgilao's constraint violations, and the algorithm only terminates once they have all been addressed.

A subtle technical difference that follows from this construction is that \cgilao requires the additional error parameter \(\incDecTol \in \Reals_{>0}\) that is used to decide whether a state's value function has changed sufficiently to warrant updating \colsToCheck.
For now, it is convenient to use \(\incDecTol = \epsilon\), and we explain the effects of this parameter in \cpageref{pg:incDecTol-discussion}.

There are also two subtle differences in the implementations of \ilao and \cgilao as presented in \cref{alg:ilao} and \cref{alg:cg-ilao}, which will turn out to be significant in our experiments (\cref{sec:experiments}).
These changes are orthogonal to \cgilao's partial expansion mechanism, so it is possible to define four similar algorithms: (i) full expansion with implementation 1, (ii) full expansion with implementation 2, (iii) partial expansion with implementation 1, (iv) partial expansion with implementation 2.
In previous work we only considered (i) and (iv)~\cite{Schmalz2024:cgilao}, which is problematic for comparing the effect of using partial expansions versus full expansions.
We address this by comparing (ii) with (iv), and we also include (i) as a way to compare to the previous work.
The two subtle implementation differences are:
\begin{enumerate}
\item \partiallyExpandFringes{} deviates from \expandFringes{} in the case where the expansion adds an internal state \(\s \in \partstates\) where \currp is defined: \ilao stops there and does not update \envelope; \cgilao follows \currp from \s, updating \envelope, and expands any new fringe states that are encountered this way.
Our pseudocode does this by updating \envelope during the \texttt{while} loop, and in practice it is implemented with recursion.
\ilao (\cref{alg:ilao})'s implementation follows \citeAuthor{Hansen2001:ilao}, whereas \cgilao (\cref{alg:cg-ilao})'s implementation follows the recursive definition from \citeAuthor{Hansen2016:GeneralErrorBounds}.
Thus, per call, \cgilao's \partiallyExpandFringes{} can potentially expand more states than \ilao's \expandFringes{} if it keeps expanding states that lead to internal states where \currp is defined.
\item The second change concerns how \currp is tracked, and how to handle its deviation from \partp.
Recall from \cref{sec:ilao}, that \ilao's \expandFringes{} or \improvePolicy{} may encounter an internal state whose policy leads to a fringe state that is not accounted for in \envelope.
There, it is handled by making sure \(\currp = \partp\) during \improvePolicy{}.
\cgilao embraces the property that \currp approximates \partp, and only updates \currp at states \s where \(\V(\s)\) changes (\cref{alg:cg-ilao}~\cref{line:cgilao:improvePolicy:update-policy} and \cref{line:cgilao:fixViolatedConstrs:update-policy}).
This is more efficient, but can have the effect that the inaccurate \currp has no fringe states, even though \partp does.
This is not an issue because in such situations \currp must be updated at some point, setting \(\currp \neq \oldp\), and thereby ensuring that \cgilao does not terminate until it has indeed closed all fringes.
Thus, \cgilao saves the overhead of computing \currp explicitly, but may require additional iterations to discover that its greedy policy still contains fringes.
\end{enumerate}
We will see in our experiments that these implementation changes can have a significant impact on performance.

In summary, \cgilao generalises \ilao by allowing actions to be left out of its partial SSPs.
\cgilao ignores inactive actions and only adds actions as they are required.
Under the lens of linear programming, the inactive actions correspond to inactive constraints, and \cgilao refines \ilao by making the separation oracle more selective, so that only constraints that are actively violated are added to the partial SSP.
We note that \cgilao's separation oracle incurs an overhead of \qvalue computations, since it must determine whether previously non-violated constraints have become violated.
In our experiments (\cref{sec:experiments}), we show that the \qvalue savings from ignoring expensive actions outweigh the separation oracle's \qvalue overhead, and \cgilao reliably computes fewer \qvalues than the state-of-the-art.
These savings let \cgilao outperform the state-of-the-art on our benchmarks.
\cgilao represents an important step for heuristic search, because it enables a heuristic to guide our algorithm away from expensive actions as well as expensive states, letting the algorithm save on \qvalue computations for those actions.

\subsection{Properties of \cgilao and Proof of Correctness}

In this section, we show some properties of \cgilao that distinguish it from previous approaches.
Then, we prove that \cgilao is correct.

\paragraph{\cgilao is not monotonic}
In all previous algorithms based on Bellman backups, such as \lrtdp and \ilao, if they are initialised with a monotonic heuristic, their backups are guaranteed to be monotonically non-decreasing, i.e., \(\V^{t+1} \geq \V^{t}\) for all timesteps \(t\).
\cgilao does not have this guarantee because it ignores inactive actions and may evaluate a suboptimal policy before recognising and inserting the missing optimal action, which decreases the value function.
As an example of this, consider the SSP in \cref{fig:v-decrease-example} where \(\h(\s)\) is given in the bottom of the respective node.
\h is monotonic, but as we step through the iterations of \cgilao, a cost decrease occurs:

\begin{enumerate}[label=Iter. \arabic*, leftmargin=*]
\item Expands \sZ.
\item Partly expands \(\s_1\) with \(\ac_1\).
\item Expands \(\s_2\) with \(\ac_2\) and, after \CGimprovePolicy{}, we have \(\V(\s_2) = 3\), \(\V(\s_1) = 4\), \(\V(\s_0) = 5\) and \(\colsToCheck = \{(\s_1, \ac'_1)\}\). Since \(\colsToCheck \neq \emptyset\), \fixViolatedConstrs{} verifies that \((\s_1, \ac'_1)\) is currently better than the existing action \(\ac_1\) for \(\s_1\), so \(\ac'_1\) is added to \(\partactions(\s_1)\) and \(\V(\s_1)\) is changed from 4 to 3.
Recall that \(\V(\s_0) = 5\), so when \(\V(\s_1)\) is updated to \(3\), \(\{\s_0, \ac_0\}\) is inserted into \colsToCheck, and no further changes are made in this iteration.
\item Expands \(\s_3\) and \CGimprovePolicy{} reduces \(\V(\sZ)\) from 5 to 4, so \V has been decreased by \CGimprovePolicy{}.
\end{enumerate}

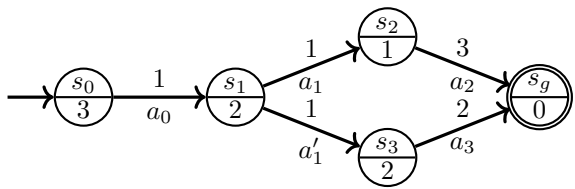
\begin{figure}[ht!]
\centering
\scalebox{1.0}{
\begin{tikzpicture}
\begin{scope}[every node/.style={circle split, thick, draw, minimum size=0.6cm, inner sep=0.04cm}]
        \coordinate (-1) at (-1, 0);
        \node (0) at (0, 0) {\(\sZ\) \nodepart{lower} \(3\)};
        \node (1) at (2, 0) {\(\s_1\) \nodepart{lower} \(2\)};
        \node (2) at (4, 0.8) {\(\s_2\) \nodepart{lower} \(1\)};
        \node (3) at (4, -0.8) {\(\s_3\) \nodepart{lower} \(2\)};
        \node[double] (g) at (6, 0) {\(\s_g\) \nodepart{lower} \(0\)};
\end{scope}
\begin{scope}[every edge/.style={draw=black, very thick}]
        \path [->] (-1) edge (0);
        \path [->] (0) edge[] node[above] {1} node[below] {\(\ac_0\)} (1);
        \path [->] (1) edge[] node[above] {1} node[below] {\(\ac_1\)} (2);
        \path [->] (1) edge[] node[above] {1} node[below] {\(\ac'_1\)} (3);
        \path [->] (2) edge[] node[above] {3} node[below] {\(\ac_2\)} (g);
        \path [->] (3) edge[] node[above] {2} node[below] {\(\ac_3\)} (g);
\end{scope}
\end{tikzpicture}}
\caption{An SSP equipped with a monotonic heuristic where \cgilao's value function is not monotonically non-decreasing.}
\label{fig:v-decrease-example}
\end{figure}

\paragraph{\cgilao is \econsistent on \partssp, not \ssp}
Recall the definition of \econsistency as given in \cref{def:epsilon-consistency}.
For \findAndRevise algorithms, the output \V is \econsistent \wrt the original SSP.
This is not the case for \cgilao: it guarantees \econsistency over its partial SSP by construction, but may produce a solution that is not \econsistent \wrt \ssp.
This is not an issue and does not prevent \cgilao from finding optimal solutions, as we explain later, but it is important to understand.
Consider \cgilao's steps on the pathological SSP in~\cref{fig:cgilao-not-econsistent}:

\begin{enumerate}[label=Iter. \arabic*, leftmargin=*]
\item Expands \sZ.
\item Expands \(\s_3\) and partly expands \(\s_1\) with \(\ac'_1\) since \(\Q(\s_1, \ac'_1) = 4\) and \(\Q(\s_1, \ac_1) = 6\).
After \CGimprovePolicy{}, \(\V(\s_3) = 6, \V(\s_1) = 10, \V(\s_0) = 9\).
Importantly, \(\V(\s_1)\) did not increase, so its successors are not added to \colsToCheck.
\item There are no fringes left to expand, so \cgilao does one pass of \improvePolicy{} which leaves \V unchanged and exits with \(\residual = 0\) and \(\colsToCheck = \emptyset\).
The algorithm terminates.
\end{enumerate}
The algorithm has terminated, but \(\V(\s_1) - \Q(\s_1, \ac_1) = 4 > \epsilon\), so \V is not \econsistent.
Nevertheless, the algorithm has found an optimal policy.
The upshot is that we must be careful and specify the SSP for which the algorithm is \econsistent.
This is analogous to \findAndRevise algorithms not being globally \econsistent, and requiring that the greedy policy is specified.

\begin{figure}[t!]
\centering
\scalebox{1.0}{
\begin{tikzpicture}
\begin{scope}[every node/.style={circle split, thick, draw, minimum size=0.6cm, inner sep=0.04cm}]
        \coordinate (-1) at (-2, -3);
        \node (0) at (-1, -3) {\(\sZ\) \nodepart{lower} \(6\)};
        \node (1) at (0, 0) {\(\s_1\) \nodepart{lower} \(10\)};
        \node (2) at (2, 2) {\(\s_2\) \nodepart{lower} \(5\)};
        \node (3) at (2, -2) {\(\s_3\) \nodepart{lower} \(0\)};
        \node[double] (g) at (4, 0) {\(\s_g\) \nodepart{lower} \(0\)};
        \coordinate (branch) at (0-0.35, -2-0.35);
\end{scope}
\begin{scope}[every edge/.style={draw=black, very thick}]
        \path [->] (-1) edge (0);
        \path [-] (0) edge[] node[above left] {1} node[below right] {\(\ac_0\)} (branch);
        \path [->] (branch) edge[out=45, in=270] node[left] {\(\frac{1}{2}\)} (1);
        \path [->] (branch) edge[out=45, in=180] node[below] {\(\frac{1}{2}\)} (3);
        \path [->] (1) edge[] node[above left] {1} node[below right] {\(\ac_1\)} (2);
        \path [->] (1) edge[] node[above right] {4} node[below left] {\(\ac'_1\)} (3);
        \path [->] (2) edge[] node[above right] {9} node[below left] {\(\ac_2\)} (g);
        \path [->] (3) edge[] node[above left] {6} node[below right] {\(\ac_3\)} (g);
\end{scope}
\end{tikzpicture}}
\caption{An SSP where \cgilao does not return an \econsistent w.r.t. the whole SSP.}
\label{fig:cgilao-not-econsistent}
\end{figure}
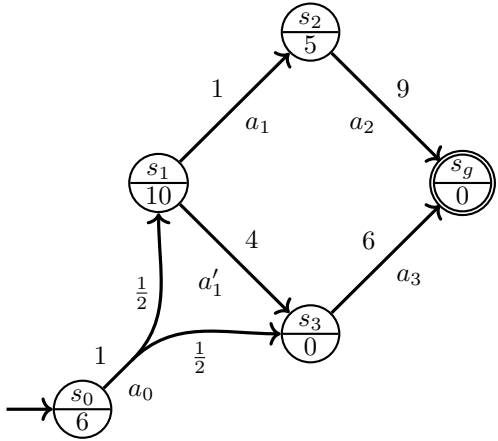

Thus, \cgilao has different behaviours to previous algorithms based on Bellman backups.
Consequently, we can not immediately apply previous proof techniques, and we must be careful with our definition of optimality.
We take a moment to formalise the notion of an optimal algorithm.
\begin{definition}[Optimal Algorithm]
An algorithm parameterised by the error term \(\epsilon \in \Reals_{>0}\) is called optimal if, as \(\epsilon \to 0\), it outputs an optimal policy.
\end{definition}
Observe that this definition was used implicitly in the background (\cref{sec:background}), and we are just making it explicit now.
\Cref{thm:vi-error} and \cref{thm:find-and-revise-is-optimal} demonstrate that VI and \findAndRevise algorithms are optimal, respectively.

It turns out that \econsistency over partial SSPs \partssp is sufficient to prove optimality over the original SSP \ssp, provided that \V is admissible \wrt \ssp, and the greedy policy \partp does not terminate in one of \partssp's artificial goal states \(\partgoals \setminus \Sg\).
The first condition is inherited from regular \econsistency, and the second condition ensures that the greedy policy \partp induces a closed policy for \ssp.
To formalise and prove this claim, we start by introducing the following upper bound over the expected number of steps to reach the goal.

\begin{definition}\label{def:Nbar}
Let \(\ssp_{\Sg'}\) denote an SSP that is identical to \ssp but has different goal states \(\Sg' \subseteq \Ss\), and let \(N(\s, \p, \ssp_{\Sg'})\) denote the expected number of steps to reach the goals \(\Sg'\) by following the proper policy \p from \s.
Then, \(\Nbar{\s}{\ssp}\) is the smallest number that satisfies \(\Nbar{\s}{\ssp} \geq N(\s, \p, \ssp_{\Sg'})\) for all \(\ssp_{\Sg'}\) and all policies \p that are proper for \(\ssp_{\Sg'}\) from \s, omitting any \(\ssp_{\Sg'}\) that have no proper policies from \s.
Importantly, for a fixed \ssp, the number of ways to pick \(\ssp_{\Sg'}\) is finite, and for each \(\ssp_{\Sg'}\) the number of proper policies from \s is also finite.
Therefore, we can indeed pick \(\Nbar{\s}{\ssp}\) to be the minimal number that satisfies our constraints, so it is well-defined and finite.
\end{definition}

This term will be useful because it lets us use the same term over many variants of the SSP \ssp.
Without it, we would have a different \(\N{\s}{\V}{\ssp'}\) for each SSP \(\ssp'\), and we would have to be careful how they relate.
The fact that \(\Nbar{\s}{\ssp}\) can get large is unimportant for our proofs as long as \(\epsilon \Nbar{\s}{\ssp} \to 0\) as \(\epsilon \to 0\).
It is possible to sharpen our proofs with tighter upper bounds, but this does not contribute towards proving that \cgilao is correct.

We are now equipped to formalise and prove our claim that \econsistency over partial SSPs \partssp is sufficient to prove optimality over the original SSP \ssp, provided that \V is admissible \wrt \ssp, and the greedy policy \partp does not terminate in one of \partssp's artificial goal states \(\partgoals \setminus \Sg\).
Note that we weaken the requirement of \(\V \leq \V^*\), which will be useful for later proofs.

\begin{theorem}\label{thm:action-specified-econsistency-optimal}
Consider an algorithm that for each \(\epsilon\) outputs \partssp and \V such that
\begin{itemize}
\item \V is \econsistent \wrt \partssp,
\item \V satisfies \(\V(\s) \leq \V^*(\s) + \epsilon \Nbar{\s}{\ssp} \; \forall \s \in \Ss^{\partp}\), and
\item \partp does not reach any artificial goal states, i.e., \(\Ss^{\partp} \cap (\partgoals \setminus \Sg) = \emptyset\).
\end{itemize}
As \(\epsilon \to 0\), the greedy policy \partp becomes an optimal policy, i.e., the algorithm is optimal.
\end{theorem}

\begin{proof}
First, \econsistency over \partssp implies that \partp is accurately evaluated by \V, i.e.,
\[| \V(\s) - \V_{\partp}(\s) | \leq \epsilon \Nbar{\s}{\ssp} \; \forall \s \in \Ss^{\partp}.\]
To get this bound, we construct a copy of \ssp called \(\ssp'\) with the actions changed thus: \(\A'(\s) = \{\p_{\V}(\s)\} \; \forall \s \in \Ss \setminus \Sg\).
The optimal value function of \(\ssp'\) is \(\V_{\partp}\) (defined over \(\Ss^{\partp}\)), where \(\V_{\partp}(\s)\) is the cost-to-go of \partp from \s.
Our \econsistent \V must be globally \econsistent for \(\ssp'\), and we can apply \cref{thm:vi-error} to obtain \(| \V(\s) - \V_{\partp}(\s) | \leq \epsilon \N{\s}{\V}{\ssp'} \; \forall \s \in \Ss^{\partp}\).
We get the desired bound by observing that \(\N{\s}{\V}{\ssp'} \leq \Nbar{\s}{\ssp}\) for all \(\s \in \Ss^{\partp}\) by \cref{def:Nbar}.

Second, we have \(\V_{\partp}(\s) \geq \V^*(\s) \; \forall \s \in \Ss^{\partp}\) because \partp is a closed policy for \ssp which can not be cheaper than an optimal one.
We combine this with the previous observation to get
\[\V(\s) + \epsilon \Nbar{\s}{\ssp'} \geq \V_{\partp}(\s) \geq \V^*(\s) \; \forall \s \in \Ss^{\partp}\]
which can be rearranged into
\[\V(\s) \geq \V^*(\s) - \epsilon \Nbar{\s}{\ssp} \; \forall \s \in \Ss^{\partp}.\]
On the other hand, we have assumed the upper bound \(\V(\s) \leq \V^*(\s) + \epsilon \Nbar{\s}{\ssp} \; \forall \s \in \Ss^{\partp}\), which lets us squeeze \V thus:
\[\V^*(\s) - \epsilon \Nbar{\s}{\ssp} \leq \V(\s) \leq \V^*(\s) + \epsilon \Nbar{\s}{\ssp} \; \forall \s \in \Ss^{\partp}.\]
As \(\epsilon \to 0\), the errors disappear, so \(\V(\s) = \V^*(\s) \; \forall \s \in \Ss^{\partp}\) and \(\V(\s) \leq \V^*(\s) \; \forall \s \in \Ss\).
Thus, the greedy policy \partp is greedy \wrt \(\V^*\) since \(\V(\s) = \V^*(\s)\) over \partp's envelope, i.e., \partp must be an optimal policy.
\end{proof}

Thus, \cref{thm:action-specified-econsistency-optimal} lets us relax the assumption that algorithms must be \econsistent \wrt the original SSP, and shows that \econsistency \wrt partial SSPs still provides optimality guarantees, provided that some additional requirements are satisfied.
We briefly point out that \econsistency \wrt partial SSPs can be relaxed even further by specifying a closed policy \p, and inferring a partial SSP from \p's envelope.
Algorithms that are in this sense \econsistent \wrt a specified policy are still optimal under similar assumptions to \econsistency \wrt partial SSPs by similar arguments.
The notion of \econsistency \wrt partial SSPs accommodates \cgilao, and equips us to prove that \cgilao is sound and complete, i.e., it always returns some solution and it is optimal.
We start by showing that \cgilao always terminates:

\begin{lemma}\label{lem:partiallyExpandFringes-terminates}
\partiallyExpandFringes{} terminates.
\end{lemma}

\begin{proof}
In each pass of \partiallyExpandFringes{}'s main loop~(\cref{alg:cg-ilao}~lines~\ref{line:cgilao:partiallyExpandFringes:main-loop-start}--\ref{line:cgilao:partiallyExpandFringes:main-loop-end}), a new state is removed from the artificial goals \(\partgoals \setminus \Sg\) and added as an internal state.
But the SSP \ssp has finitely-many states, so eventually \partiallyExpandFringes{} must run out of new states to add, and will therefore exit its main loop and terminate.
\end{proof}

\begin{lemma}\label{lem:CGimprovePolicy-terminates}
\CGimprovePolicy{} terminates.
\end{lemma}

\begin{proof}
The main loop of \CGimprovePolicy{}~(\cref{alg:cg-ilao}~lines~\ref{line:cgilao:CGimprovePolicy:main-loop-start}--\ref{line:cgilao:CGimprovePolicy:main-loop-end}) is running VI on \(\envelope \setminus \partgoals\) until \econsistency, with additional bookkeeping~(\cref{alg:cg-ilao}~lines~\ref{line:cgilao:improvePolicy:check-violated-constraints-begin}--\ref{line:cgilao:improvePolicy:check-violated-constraints-end}) that does not affect VI.
Given that \partssp and \envelope are fixed, this VI eventually reduces the residual to \(\residual \leq \epsilon\), so \CGimprovePolicy{} exits the main loop and terminates.
Note that the other conditions for termination \(\fringe \neq \emptyset\) or \(\currp \neq \oldp\) may cause \CGimprovePolicy{} to terminate sooner.
\end{proof}

\begin{theorem}\label{thm:cgilao-terminates}
\cgilao terminates.
\end{theorem}

\begin{proof}
For contradiction, suppose \cgilao does not terminate.
By \cref{lem:partiallyExpandFringes-terminates} and \cref{lem:CGimprovePolicy-terminates} we know that the functions that \cgilao calls must terminate, so if \cgilao does not terminate, it must be because it gets stuck in its main loop~(\cref{alg:cg-ilao}~lines~\ref{line:cgilao:main-loop-start}--\ref{line:cgilao:main-loop-end}).
Suppose that \(\fringe \neq \emptyset\) for infinitely-many steps.
In each pass of the main loop where \(\fringe \neq \emptyset\), \partiallyExpandFringes{} changes at least one of \currp's fringe states to an internal state.
But this is impossible, because \ssp only has finitely-many states and we only allow fringe states to be expanded once, so \cgilao must eventually run out of states to partially expand, and after finitely-many steps \fringe remains empty.
Moreover, after finitely-many steps \partssp remains constant by a similar argument, since \ssp only has finitely-many actions that can be added to \partssp.
Then there must be a finite set of states \(X \subseteq \partstates\) that are updated with Bellman
backups infinitely often by \CGimprovePolicy{} or \fixViolatedConstrs{}.
But \(X\) induces a new partial SSP, and applying Bellman backups infinitely often to all \(X\) solves this new partial SSP with VI.
Consequently, \V converges to a fixed point and the residual will be less than \(\epsilon\) in finite time.
Since \V will not be updated, \currp will also no longer be updated.
At this point all the termination conditions are satisfied~(\cref{alg:cg-ilao}~\cref{line:cgilao:whileCondition}), forcing \cgilao to terminate, giving us the desired contradiction.
\end{proof}

We now work towards proving that \cgilao is optimal by showing that it satisfies \econsistency \wrt its partial SSPs, and that it satisfies the extra conditions required by \cref{thm:action-specified-econsistency-optimal}.
To help with this, we prove some simpler properties of \cgilao first:

\begin{lemma}\label{lem:violations-are-in-gamma}
At the end of \cgilao's main loop~(\cref{alg:cg-ilao}~after~\cref{line:cgilao:main-loop-end-minus-one}), if there is \(\s \in \partstates \setminus \partgoals\) and \(\ac \in \A(\s)\) such that \(\V(\s) > \Qsa + \epsilon\), then \((\s, \ac) \in \colsToCheck\) or \(\V(\s) \leq \V^*(\s)\).
\end{lemma}

\begin{proof}
We prove by induction over \(n\), the number of iterations of \cgilao.
The base case is \(n = 0\), i.e., the partial SSP \partssp has been initialised and never updated, so \(\partstates \setminus \partgoals = \emptyset\), making the claim is vacuously true.
Our inductive hypothesis is that the lemma holds after \(n\) iterations, and we now prove that this makes the lemma hold after \(n+1\) iterations.
We step through the possible ways that the lemma's precondition can be made true, and how in those cases the algorithm ensures that the relevant column is in \colsToCheck{}:
\begin{enumerate}
\item If \(\V(\s) > \Qsa + \epsilon\) at the start of iteration \(n+1\), then by the inductive hypothesis, \((\s, \ac) \in \colsToCheck{}\) or \(\V(\s) \leq \V^*(\s)\).
So, we only concern ourselves with any constraint violations that are introduced in this iteration.
\item The post-order DFS traversal~(\cref{alg:cg-ilao}~\cref{line:cgilao:dfs}) does not affect \(\partstates \setminus \partgoals\) nor \V, and therefore can not introduce any new constraint violations.
\item \partiallyExpandFringes{}~(\cref{alg:cg-ilao}~\cref{line:cgilao:call-to-partiallyExpandFringes}) does not affect \V, but may introduce new states into \(\s \in \partstates \setminus \partgoals\).
It expands fringe states \s in \(\Ss^{\currp}\) by bringing \s into \(\partstates \setminus \partgoals\), adding some action \(\widehat{\ac} \in \A(\s)\) into \(\partactions(\s)\), and setting \(\currp(\s) \gets \widehat{\ac}\).
This may introduce constraint violations but \s was until now a fringe state, so \(\V(\s) = \h(\s) \leq \V^*(\s)\).
\item In \CGimprovePolicy{}~(\cref{alg:cg-ilao}~\cref{line:cgilao:call-to-CGimprovePolicy}), a new \(\V(\s) > \Qsa + \epsilon\) may occur if \(\V(\s)\) is increased or \Qsa is decreased.
\CGimprovePolicy{} tracks the appropriate \successors or \predecessors for \(\V\) increases and decreases respectively.
\item \fixViolatedConstrs{}~(\cref{alg:cg-ilao}~\cref{line:cgilao:call-to-fixViolatedConstrs}) ensures for each \((\s, \ac) \in \colsToCheck{}\) that \(\V(\s) \leq \Qsa + \epsilon\) before removing \((\s, \ac)\) from \colsToCheck{}.
\(\V(\s)\) can not increase in \fixViolatedConstrs{}, since we are setting \(\V(\s) \gets \min\{\V(\s), \Qsa\}\) for \((\s, \ac) \in \colsToCheck{}\).
However, new instances of \(\V(\s) > \Qsa + \epsilon\) can be introduced if some \Qsa is decreased, which is tracked by \fixViolatedConstrs{}.
\end{enumerate}
So indeed, any constraint violations \(\V(\s) > \Qsa + \epsilon\) introduced by \cgilao in lines~\ref{line:cgilao:dfs}--\ref{line:cgilao:call-to-fixViolatedConstrs} are tracked in \colsToCheck{} or \(\V(\s) \leq \V^*(\s)\).
\end{proof}

Importantly, this lemma applies only to internal states \(\s \in \partstates \setminus \partgoals\), but in terms of actions applies to all applicable actions \(\ac \in \A(\s)\), not only internal ones.
Equipped with this invariant, we prove two important properties of \cgilao when \(\colsToCheck = \emptyset\).

\begin{lemma}\label{lemma:empty-cols-to-check-implies-admissibility}
At the end of \cgilao's main loop~(\cref{alg:cg-ilao}~after~\cref{line:cgilao:main-loop-end-minus-one}), \(\colsToCheck = \emptyset\) implies \V is admissible up to an error, i.e., \(\V(\s) \leq \V^*(\s) + \epsilon \Nbar{\s}{\ssp} \; \forall \s \in \Ss\).
\end{lemma}

\begin{proof}
All fringe, external, and true goal states \(\s \in \partgoals \cup (\Ss \setminus \partstates)\), are initialised with \(\V(\s) \gets \h(\s) \leq \V^*(\s)\) and never have their \(\V(\s)\) updated.
For the remaining states, i.e., the internal states \(\partstates \setminus \partgoals\), we can use \cref{lem:violations-are-in-gamma}.
We ignore the internal states \s with \(\V(\s) \leq \V^*(\s)\), because they immediately satisfy the requirement of the lemma, and we focus on the internal states with \(\V(\s) > \V^*(\s)\), which we call
 \(\tilde{\Ss}\).
We assume that \(\colsToCheck = \emptyset\), so by the contrapositive of \cref{lem:violations-are-in-gamma} it must be the case that  \(\V(\s) \leq \Qsa + \epsilon \; \forall \s \in \tilde{\Ss}\).
It follows that \(\V(\s) \leq \V^*(\s) + \epsilon \Nbar{\s}{\ssp} \; \forall \s \in \tilde{\Ss}\).
To see this: if we apply VI to \(\tilde{\ssp}\), the \ssp restricted to the states \(\tilde{\Ss}\), \V will remain unaffected because \(\residual(\s) \leq \epsilon\) for these states, and then we can apply \cref{thm:vi-error} to get \(\V(\s) \leq \V^*(\s) + \epsilon \N{\s}{\V}{\tilde{\ssp}} \; \forall \s \in \tilde{\Ss}\).
Then, \(\N{\s}{\V}{\tilde{\ssp}} \leq \Nbar{\s}{\ssp} \; \forall \s \in \tilde{\Ss}\) (\cref{def:Nbar}), giving the desired bound.
Thus, we have addressed all states in \Ss, proving our claim.
\end{proof}

\begin{lemma}\label{lemma:cols-to-check-is-empty-upon-termination}
Upon termination, \cgilao has \(\colsToCheck = \emptyset\).
\end{lemma}

\begin{proof}
For contradiction, suppose \cgilao does terminate with \(\colsToCheck \neq \emptyset\).
Then, \colsToCheck must have been populated by the final call to \fixViolatedConstrs{}, but constraints are only added to \(\colsToCheck\) in \fixViolatedConstrs{} when \(\V(\s) > \Qsa + \epsilon\), and therefore \(\residual \geq \V(\s) - \Qsa > \epsilon\).
This yields the desired contradiction because \cgilao can not terminate until \(\residual \leq \epsilon\).
\end{proof}

Finally, we are ready to prove that \cgilao is optimal.

\begin{theorem}
\cgilao is optimal.
\end{theorem}

\begin{proof}
Suppose that \cgilao terminates with the value function \V, the partial SSP \partssp, and the candidate policy \currp.
Upon termination, we know that \(\currp = \partp\), otherwise a policy change must have occurred.
Then, \V satisfies \econsistency \wrt \partssp, since \(\residual \leq \epsilon\) is one of \cgilao's termination conditions (\cref{alg:cg-ilao}~\cref{line:cgilao:whileCondition}).
Now, it suffices to show \(\V(\s) \leq \V^*(\s) + \epsilon \Nbar{\s}{\ssp} \; \forall \s \in \Ss\) and that \partp does not encounter any artificial goals \(\partgoals \setminus \Sg\), and then we can apply \cref{thm:action-specified-econsistency-optimal} to conclude that \cgilao is optimal.
Indeed, upon termination, \cgilao has \(\colsToCheck = \emptyset\) by \cref{lemma:cols-to-check-is-empty-upon-termination}, so we can apply \cref{lemma:empty-cols-to-check-implies-admissibility} to get \(\V(\s) \leq \V^*(\s) + \epsilon \Nbar{\s}{\ssp} \; \forall \s \in \Ss\).
We also know that \partp does not encounter any artificial goals thanks to \cgilao's termination condition \(\fringe \neq \emptyset\)~(\cref{alg:cg-ilao}~\cref{line:cgilao:whileCondition}).
\end{proof}

\label{pg:incDecTol-discussion}
In these proofs, we have assumed that \(\incDecTol = 0\) in order to guarantee that all cost increases and decreases are tracked.
In practice, we use \(\incDecTol > 0\) due to the inherent numerical instability issues of floating-point number representations, and in particular it is convenient to choose \(\incDecTol = \epsilon\), to minimise the number of parameters.
This is not problematic in our benchmarks, but one must be careful because \(\incDecTol > 0\) can introduce arbitrarily large errors into the value function.
For example, consider the SSPs in \cref{fig:pathological-case-for-increase-decrease}, with initial value functions shown in the bottom of each node.
For the SSP on the left, \cgilao with \(\incDecTol = \epsilon\) does the following:

\begin{enumerate}[label=Iter. \arabic*, leftmargin=*]
\item Partially expands \sZ with \(\ac_0\).
\item Expands \(\s_1\) with \(\ac_1\) and updates \(\V(\s_1) = 3 + \frac{2\epsilon}{3}, \V(\sZ) = 4 + \frac{2\epsilon}{3}\).
The increases to \(\V(\s_1)\) and \(\V(\sZ)\) are \(\frac{2\epsilon}{3} < \eta = \epsilon\), so nothing is added to \colsToCheck.
\item Expands \(\s_2\) with \(\ac_2\) and updates \(\V(\s_2) = 2 + \frac{2\epsilon}{3}, \V(\s_1) = 3 + \frac{4\epsilon}{3}, \V(\sZ) = 4 + \frac{4\epsilon}{3}\).
Again, the increases to \V are \(\frac{2\epsilon}{3} < \eta = \epsilon\), so nothing is added to \colsToCheck.
\item Expands \(\s_3\) with \(\ac_3\) and updates \(\V(\s_3) = 1 + \frac{2\epsilon}{3}, \V(\s_2) = 2 + \frac{4\epsilon}{3}, \V(\s_1) = 3 + \frac{6\epsilon}{3}, \V(\sZ) = 4 + \frac{6\epsilon}{3}\).
For the last time, the increases to \V are \(\frac{2\epsilon}{3} < \eta = \epsilon\), so nothing is added to \colsToCheck.
\item There are no fringes left to expand, so \cgilao does one pass of \improvePolicy{} which leaves \V unchanged and exits with \(\V(\sZ) = 4 + 2\epsilon\).
The algorithm terminates.
\end{enumerate}
Thus, \cgilao finds a policy for the SSP on the left, which is suboptimal by \(\epsilon\).
\cgilao behaves similarly on the right SSP, accumulating an error of \((2k-1)\epsilon\) which gets arbitrarily large as the number of states \(n\) grows.
Clearly, as \(\incDecTol \to 0\) this error disappears, so in particular \(\incDecTol = \epsilon\) works in that sense, but one has to be careful.
Regardless of the choice of \(\epsilon\), we now have the additional theoretical issue that for any \incDecTol it is possible to construct a pathological SSP that makes \cgilao's policy arbitrarily bad.

\begin{figure}[t]
\centering
\scalebox{0.8}{
\begin{tikzpicture}
\begin{scope}[every node/.style={circle split, thick, draw, minimum size=0.6cm, inner sep=0.04cm}]
        \coordinate (-1) at (-1, 0);
        \node (0) at (0, 0) {\(\sZ\) \nodepart{lower} \(4\)};
        \node (1) at (2, 0.8) {\(\s_1\) \nodepart{lower} \(3\)};
        \node (2) at (4, 0.8) {\(\s_2\) \nodepart{lower} \(2\)};
        \node (3) at (6, 0.8) {\(\s_3\) \nodepart{lower} \(1\)};
        \node[double] (g) at (8, 0) {\(\s_g\) \nodepart{lower} \(0\)};
\end{scope}
\begin{scope}[every edge/.style={draw=black, very thick}]
        \path [->] (-1) edge (0);
        \path [->] (0) edge[] node[above] {1} node[below] {\(\ac_0\)} (1);
        \path [->] (1) edge[] node[above] {\(1 + \frac{2\epsilon}{3}\)} node[below] {\(\ac_1\)} (2);
        \path [->] (2) edge[] node[above] {\(1 + \frac{2\epsilon}{3}\)} node[below] {\(\ac_2\)} (3);
        \path [->] (3) edge[] node[above] {\(1 + \frac{2\epsilon}{3}\)} node[below] {\(\ac_3\)} (g);
        \path [->] (0) edge[bend right=15] node[above] {\(4 + \epsilon\)} node[below] {\(\ac'_0\)} (g);
\end{scope}
\end{tikzpicture}
\hspace{5mm}
\begin{tikzpicture}
\begin{scope}[every node/.style={circle split, thick, draw, minimum size=0.8cm, inner sep=0.04cm}]
        \coordinate (-1) at (-1, 0);
        \node (0) at (0, 0) {\(\sZ\) \nodepart{lower} \(n \hspace{-0.1cm} + \hspace{-0.1cm} 1\)};
        \node (1) at (2, 0.8) {\(\s_1\) \nodepart{lower} \(n\)};
        \node (n) at (6, 0.8) {\(\s_n\) \nodepart{lower} \(1\)};
        \node[double] (g) at (8, 0) {\(\s_g\) \nodepart{lower} \(0\)};
        \coordinate (spaceLeft) at (3.5, 0.8);
        \coordinate (spaceRight) at (4.5, 0.8);
\end{scope}
\begin{scope}[every edge/.style={draw=black, very thick}]
        \path [->] (-1) edge (0);
        \path [->] (0) edge[] node[above] {1} node[below] {\(\ac_0\)} (1);
        \path [->] (1) edge[] node[above] {\(1 + \frac{2\epsilon}{3}\)} node[below] {\(\ac_1\)} (spaceLeft);
        \path [-] (spaceLeft) edge[dotted] (spaceRight);
        \path [->] (spaceRight) edge[] node[above] {\(1 + \frac{2\epsilon}{3}\)} node[below] {\(\ac_{n-1}\)} (n);
        \path [->] (n) edge[] node[above] {\(1 + \frac{2\epsilon}{3}\)} node[below] {\(\ac_n\)} (g);
        \path [->] (0) edge[bend right=15] node[above] {\(n + 1 + \epsilon\)} node[below] {\(\ac'_0\)} (g);
\end{scope}
\end{tikzpicture}}
\caption{SSPs where \cgilao with a large \incDecTol (\(\incDecTol \geq \epsilon\)) accumulates a large error.
The left SSP ends up with an error of \(\epsilon\) and \(\V(\sZ) = 4 + 2\epsilon\), and the right SSP can have an arbitrarily large error of \((2k-1)\epsilon\) with \(\V(\sZ) = n + 1 + 2k\epsilon\) for \(k = 3n\).
}
\label{fig:pathological-case-for-increase-decrease}
\end{figure}
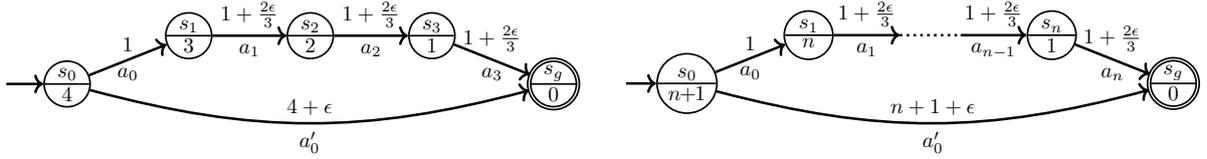

We note that although \cgilao is not \econsistent \wrt \ssp in general, it becomes \econsistent \wrt \ssp if \h is monotonic. %
Also, we can make \cgilao \econsistent \wrt \ssp by modifying \partiallyExpandFringes{} so that for each expanded fringe \(\s_{f} \in \envelope \cap (\partgoals \setminus \Sg)\), we add \(\A(\s_{f})\) to \colsToCheck.
This addresses the edge case in step 3.\ of \cref{lem:violations-are-in-gamma}'s proof.
As explained before, it is not important that \cgilao is \econsistent \wrt \ssp for the normal use-case of \cgilao.

\section{Different Expansion Methods for \cgilao}\label{sec:cgilao-expansions}

In this section we explore different methods that \cgilao can use to expand fringe states.
For the correctness of \cgilao,  \(\partiallyExpandFringes{}(\ssp, \partssp, \partp, \envelope, \V)\) has two requirements:
\begin{enumerate}
\item at least one fringe state \(\s \in \envelope \cap (\partgoals \setminus \Sg)\) is expanded, i.e., afterwards \(\s \not\in \partgoals \setminus \Sg\) and \(\partactions(\s) \neq \emptyset\) or \(\s \in \Sg\),
\item after expansion, \envelope is up to date and new fringe states are accurately recorded in \partgoals.
\end{enumerate}
In this work, \partiallyExpandFringes{} expands all fringe states \(\envelope \cap (\partgoals \setminus \Sg)\), i.e., we do not consider a method that expands a subset of fringe states.
This is based on the empirical results of \citeAuthor{Hansen2001:ilao} that suggest expanding all reachable fringes is more efficient than expanding subsets.

We now describe different methods for expanding each fringe state.

\subsection{One-step-look-ahead}\label{sec:bellman-expansions}

Here, we consider three methods for expanding a state \s that only look one step ahead and greedily select actions based on their \qvalues.

\bellmanSingle expands \s by computing the \qvalues of all applicable actions, and selects a single greedy action, breaking ties arbitrarily but in a reproducible way (\cref{assump:greedy-policy-tie-breaking}).
Formally, \(\bellmanSingle(\s)\) sets
\begin{equation}
\partactions(\s) \gets \left\{ \widehat{\ac} \right\} \text{ for some } \widehat{\ac} \in \A(\s) \text{ s.t. } \Q(\s, \widehat{\ac}) = \min_{\ac \in \A(\s)} \Qsa.
\end{equation}\label{eq:add-single-greedy-action}
In the best case, this method adds precisely the actions that define an optimal policy, and consequently \cgilao's partial SSP is minimal in the sense that it only contains the states and actions that are necessary to define the relevant optimal policy.

\begin{theorem}\label{thm:cgilao-minimal-with-hstar}
If initialised with \(\V^*\), \cgilao with \bellmanSingle as its expansion mechanism will have a minimal partial SSP.
That is, given \cgilao's output \(\p^*\), its partial SSP will consist precisely of \(\partstates =
\Ss^{\p*}\) and \(\partactions(\s) = \{\p^*(\s)\}\) for each \(\s \in \partstates \setminus \Sg\).
\end{theorem}
\begin{proof}
In each expansion, \cgilao will insert the greedy action according to \(\V = \V^*\).
Then, Bellman updates will not change \V since the greedy actions' \Qsa are precisely \(\V^*(\s)\) and \(V = \V^*\).
Therefore \colsToCheck remains empty and the greedy policy is never updated (it breaks ties in
favour of its current actions).
So, \cgilao will perform a greedy best-first search using \(\V = \V^*\), which results in expanding precisely the states in the greedy policy envelope \(\Ss^{\p_{\V}}\) and the corresponding actions.
\end{proof}

Note that if there are multiple optimal policies, there is no guarantee that \cgilao returns the one
with the smallest envelope, i.e., it does not guarantee the fewest possible expansions over all optimal policies.

In practice, it is usually worthwhile to add all tied-greedy actions upon expansion.
So, we define the expansion \(\bellmanTied(\s)\) which sets
\begin{equation}
\partactions(\s) \gets \left\{\widehat{\ac} \in \A(\s) : \Q(\s, \widehat{\ac}) = \min_{\ac \in \A(\s)} \Qsa\right\}.
\end{equation}
To motivate that it often pays off to add tied-greedy actions, suppose \cgilao expands state \s with the single greedy action \(\ac_{\min}\).
If \(\V(\s)\) increases, then \cgilao is forced to add all remaining external successor actions \(\A(\s) \setminus \{\ac_{\min}\}\) to \colsToCheck, and from these must add all actions with \(\Qsa = \Q(\s, \ac_{\min})\) to \partssp.
Such increases to \(\V(\s)\) are common, e.g., an increase is guaranteed when \(\h(\s) < \min_{\ac \in \A(\s)} \Qsa\), which is typical for practical heuristics.
Thus, in this common situation all tied-greedy actions get added to \partssp anyway (assuming their own \qvalues did not increase significantly), so it makes sense to avoid the overhead and add them immediately.
This is the expansion presented in \cgilao's pseudocode (\cref{alg:cg-ilao}).

The final option we consider is to simply add all applicable actions. %
With this expansion, \cgilao implements \ilao, noting the minor differences in implementation discussed in \cref{sec:cgilao}.\footnote{There is a minor exception apart from the implementation changes where \cgilao with \expandFringes{} does not behave like \ilao: if \h is not monotonic, then a decrease to \V can occur, causing constraints to be added to \colsToCheck, which \cgilao has to deal with.
Though possible, we have not observed this edge case in our experiments.}

\subsection{Rollout}\label{sec:expansion-rollout}

We now consider expansion methods that use rollouts.
A rollout from state \s is obtained by simulating the execution of a \textit{base policy}: we start in \s, apply the base policy, randomly select the next state according to \(\pr\), apply the base policy to the new state, and so on.
This produces a trajectory through the SSP from \s, which we can use to inform our expansion.
Rollouts have a long history in value-iteration style algorithms, and they have enjoyed success in planning, e.g., PROST~\cite{Keller2012:PROST}, \lrtdp~\cite{Bonet2003:lrtdp}, and Robust-FF~\cite{Teichteil-Koenigsbuch2010:RobustFF}.
The primary benefit of rollouts is that they enable information from the goal to be propagated sooner.
Observe that the value functions of \cgilao with one-step-lookahead expansions are informed exclusively by the heuristic for most of their lifetimes, and only start propagating information from goal states when a goal enters the partial SSP.
To try to address this, we implement rollouts into \cgilao during the expansion phase:
when expanding state \s, we perform a rollout from \s to generate a trajectory \(\langle \s^0, \ac^0, \s^1, \ac^1, \dots, \s^n \rangle\), and then add all the state-action pairs \((\s^i, \ac^i)\) to the partial SSP.
Although this may insert suboptimal actions into the partial SSP, this expansion satisfies the requirements of \partiallyExpandFringes{}, and therefore does not affect \cgilao's correctness.
We do not consider adding rollouts to \ilao's expansion, because \ilao would have to add all applicable actions for all the states in the trajectory, which makes its partial SSP significantly larger than it needs to be.
We now present two base policies to use for \cgilao's rollout expansions: trials and FF.

\paragraph{Trial rollouts}

Following RTDP~\cite{Barto1995:rtdp}, we roll out with a trial by selecting the greedy action \wrt \V at each step.
The pseudocode for our trials is given in \cref{alg:trial}.

SSPs can have cycles, and a na{\"i}ve implementation of trial is susceptible to getting stuck inside a cycle indefinitely.
RTDP deals with cycles by updating \(\V \gets \min_{\ac \in \A(s)} \Qsa\) for each state \s it encounters; suppose RTDP's simulation is stuck inside a cycle, then the relevant \V values increase until the \qvalues to remain inside the cycle get so large that the simulation leaves the cycle.
In practice, to handle SSPs that violate reachability~(\cref{assump:reachability}), we use the fixed-penalty transform~\cite{trevizan17:mcmp} that allows the agent to ``give-up'' with a large cost \(\dCost \in \Reals_{>0}\) --- consequently, if RTDP is stuck within a cycle, the simulation ``gives up'' as soon as \(\V(\s) \geq \dCost\).
Similarly to RTDP, we update \V in our trial, but, importantly, in order for \cgilao to remain correct, it must preserve that \(\V(\s) \leq \V^*(\s)\) for all external states \(\s \in \Ss \setminus \partstates\).
Recall that \cgilao may have \(\V(\s) \not\leq \V^*(\s)\) for internal states \(\s \in \partstates\).
To preserve \(\V(\s) \leq \V^*(\s)\) for external states, we force the trial to terminate whenever it can reach an internal state from its current state \s, i.e., the trial is stopped whenever there is an action \(\ac \in \A(\s)\) such that \(\supp(\s, \ac) \cap \partstates \neq \emptyset\).
Consequently, \(\V(\s)\) is only updated by the trial if all effects of all applicable actions are external.
External states are in this way only updated with full Bellman backups using other external states.
Recall that external states are initialised with \(\V(\s) = \h(\s) \leq \V^*(\s)\), so by updating \(\V(\s)\) for external states only by applying Bellman backups to other external states, we guarantee that \(\V(\s) \leq \V^*(\s)\) for external states.
Terminating the trial whenever an internal state is reachable also has the intuitive interpretation that we stop the rollout as soon as it reaches a state where ``we know what to do already.''

A trial from \s can either ``succeed'' by reaching a goal or an internal state, or it can ``fail'' by reaching a dead end or by its number of steps exceeding \(t_{\max}\).
\begin{itemize}
\item If the trial succeeded: we first remove any cycles from the trial, then we expand the remaining acyclic trajectory by adding all its state-action pairs to \partssp.\footnote{To remove cycles from the trial we identify the first and last occurrence of each state in the trial, and if these are not equal, remove the subsequence between them (as well as one of the copies of the state).}
\item If the trial failed: it does not make sense to expand the entire trial; the only fact we know about this sequence of states and actions is that they do not result in a goal or an internal state.
Instead, we expand \s with \bellmanTied on the updated \V, with the idea that the updated \V will inform the one-step-lookahead to avoid the action that lead to failure.
\end{itemize}

Our trials are parameterised by \(t_{\max}\).
A low value of \(t_{\max}\) serves as another mechanism to escape cycles, and more generally as a mechanism to avoid wasting time on a simulation that ends up in an expensive part of the state space that \cgilao does not benefit from expanding.
However, this is at the cost of finding fewer successful trajectories that end in goals or internal states.

\begin{algorithm}[t!]
{\small
\DontPrintSemicolon
  \caption{trial}\label{alg:trial}

  \function{trial\((\s_\text{init}, \V, t_{\max})\)} {
    \(\s \gets \s_{\text{init}}\) \;
    \(\texttt{trajectory} \gets \langle \s \rangle\) \;
    \(t \gets 0\) \;
    \While{\(\s \not\in \partstates\) and \(\s \not\in \Sg\)}{

      \tcp{Return failure if the trial reaches a dead end or times out}
      \If {\(t > t_{\max}\) or \(\V(\s) > \dCost\) or \(\A(\s) = \emptyset\)} {
        \Return \texttt{failure} \;
      }

      \tcp{If there is an action that reaches \partstates: take the cheapest one}
      \If{\(\exists \ac \in \A(\s) : \supp(\s, \ac) \cap \partstates \neq \emptyset\)} {
        \(\ac_{\min} \gets \argmin_{\ac \in \A(\s) : \supp(\s, \ac) \cap \partstates \neq \emptyset} \Qsa\) \;
        \(\s \gets \text{ some state } \in \supp(\s, \ac_{\min}) \cap \partstates\) \;
      }
      \tcp{Otherwise continue RTDP trial as normal}
      \Else {
        \(\ac_{\min} \gets \argmin_{\ac \in \A(\s)} \Qsa\) \;
        \(\V(\s) \gets \Q(\s, \ac_{\min})\) \;
        \(\s \gets \text{ state randomly chosen from } \pr(\cdot|\s, \ac_{\min})\) \;
      }
      append \(\langle \ac_{\min}, \s \rangle\) to \(\texttt{trajectory}\) \;
      \(t \gets t + 1\) \;

    }

    remove cycles from \(\texttt{trajectory}\) \;
    \Return \(\texttt{trajectory}\) \;
  }
}
\end{algorithm}

\paragraph{FF rollouts}

Inspired by the success of FF-Replan~\cite{Yoon2007:ffReplan} and Robust-FF~\cite{Teichteil-Koenigsbuch2010:RobustFF}, we use the deterministic algorithm FF~\cite{Hoffmann2001:FF} to find a plan for the SSP's determinisation, and use the plan as a rollout.
A determinisation removes the SSP's probabilistic effects, and produces a deterministic shortest path problem.
We consider two determinisations:
\begin{enumerate}
\item \textbf{All-outcomes determinisation.}
The all-outcomes determinisation splits each probabilistic action \(\ac \in \A(\s)\) into multiple deterministic ones.
For each \(\s' \in \supp(\s,\ac)\), the determinisation produces the actions \(\ac_{\s \to \s'}\) that lead to \(\s'\) deterministically.
This is proper relaxation of the SSP: if a proper policy exists for the SSP, then a plan must exist for the all-outcomes determinisation.
\item \textbf{Most-likely-outcome determinisation.}
The most-likely-outcomes determinisation transforms each probabilistic action \(\ac \in \A(\s)\) into a single deterministic action.
It selects \(\s'\) for some \(\s' \in \argmax_{\s''} \pr(\s''|\s, \ac)\) (breaking ties arbitrarily), and then produces the single action \(\ac_{\s \to \s'}\) that leads to \(\s'\) deterministically.
This determinisation has the benefit that it has fewer actions and maintains some information about probability distributions by picking the most likely; however, it is not a proper relaxation, because it can happen that the single most-likely outcomes that we selected do not allow a plan to be constructed.
\end{enumerate}

We run FF rollouts by running FF on the respective determinisation.
Then, FF returns the plan \(\langle \s^0, \ac^0, \dots, \s^n \rangle\) where \(\s^0 = s\) and \(\s^n \in \Sg\), or returns that no plan exists.
If it finds a plan, we expand it.
If no plan exists on the all-outcomes determinisation, then it is proof that the state we are expanding is a dead end and we mark it as such.
If no plan exists on the most-likely-outcomes determinisation we can not conclude that it is a dead end, so we run FF again, but on the all-outcomes determinisation, to construct a plan or prove a dead end.

\section{Related Work}\label{sec:related-work}

In this section, we present existing work that is related to \cgilao.
First, we present and discuss action elimination, which has the same goal as \cgilao of avoiding \qvalue computations for unneeded actions.
We explain how it is different, and has various disadvantages compared to our approach.
Afterwards, we discuss existing work in planning that uses constraint generation, motivating that our use of constraint generation is novel.
Finally, we compare \cgilao to the Partial Expansion \astar algorithm, which uses the same paradigm of ignoring actions, but in the deterministic setting.

\subsection{Action Elimination}\label{sec:action-elimination}

Action elimination lets us prove that certain actions can not be part of the optimal policy, and can therefore be permanently removed from search without affecting optimality.
This technique, similarly to \cgilao's mechanism for ignoring actions, allows us to avoid computing \qvalues for actions that will not help with the solution.
However, action elimination approaches the issue from the opposite direction, in the sense that it starts with all actions and permanently removes them when they are proved to be outside optimal policies, compared to \cgilao's mechanism which initially ignores all actions and adds them as required.

Action elimination requires upper and lower bounds on \(\V^*\), called \Vub and \Vlb respectively, and determines which actions can be eliminated using the following theorem.

\begin{theorem}[Action Elimination~\cite{Bertsekas1995}]\label{thm:action-elimination}
Given a state \(\s \in \Ss\) and action \(\tilde{\ac} \in \A(\s)\), if \(\exists \ac \in \A(\s)\) such that \(\Qlb(\s, \tilde{\ac}) > \Qub(\s, \ac)\), then action \(\tilde{\ac}\) cannot be optimal for state \s, and can therefore be eliminated.
\end{theorem}

We have already discussed how to maintain \Vlb as a lower bound.
Note that outside this section we call such value functions \V, but here we call it \Vlb to emphasise the distinction between upper and lower bounds.
The main drawback of action elimination, and the reason its use is not widespread in planners, is that obtaining the upper bound \Vub is difficult.
There are methods for deriving an upper bound \Vub from only a lower bound and its Bellman residual~\cite{Hansen2001:ilao,Hansen2015:UB-for-V}, but they require the expensive computation of the expected number of steps to reach the goal or that the policy is proper, which is impractical for our purposes.
The standard way to track \Vub is to initialise \Vub \st \(\Vub \geq \V^*\), and then apply Bellman backups to it, which are guaranteed to preserve \(\Vub \geq \V^*\).
This is done in addition to tracking \Vlb, so the algorithm tracks two value functions, and each partial backup must be applied both to \Vlb and \Vub.
Thus, if all else is the same, then tracking the two value functions exactly doubles the number of computed \qvalues.
An additional difficulty for tracking \Vub, on top of the additional maintenance cost, is that we are not aware of any efficient and domain-independent algorithms to compute an initial upper bound.
The only domain-independent upper bound we are aware of, is the trivial upper bound
\[\Vub(\s) = \begin{cases} 0 &\text{ if } \s \in \Sg \\ \dCost &\text{ if } \s \not\in \Sg \end{cases}\]
where \(\dCost = \infty\) or the penalty term, if we are using a fixed-penalty transformation.

Action elimination has not been used by optimal heuristic-search algorithms to remove actions during search.
There are variants of RTDP~\cite{Barto1995:rtdp} that track both \Vub and \Vlb, such as BRTDP~\cite{McMahan2005:BRTDP}, FRTDP~\cite{Smith2006:FRTDP}, VPI-RTDP~\cite{Sanner2009:vpirtdp}; and variants for \lao~\cite{Hansen2001:ilao} such as IBLAO\(^*\)~\cite{Warnquist2010:IBLAO}.
However, none of these use action elimination: they use lower and upper bounds to define different ways to prioritise expansion and backups, and to define alternative termination conditions.

The only algorithm we are aware of that explicitly uses action elimination is \ftvi~\cite{Dai2011:tvi}. %
\ftvi consists of a computation stage where it runs Topological Value Iteration (TVI), and a preprocessing stage, where it tries to make the search space more amenable to TVI.
TVI works well when the state-space has small strongly connected components (SCCs), and degenerates to standard VI as the SCCs get bigger.
The preprocessing stage attempts to break any large SCCs into smaller ones by removing suboptimal actions with action elimination.
This step iteratively performs a Depth-First Search on the state-space, applies Bellman backups on \Vlb and \Vub in post-order, and uses action elimination to remove actions when it can.
Note that \ftvi is not a heuristic search algorithm because its search step does not use a heuristic and considers the whole state-space.

Action elimination has another disadvantage \wrt to \cgilao: it initially considers all actions, and only eliminates them later in its lifetime.
In comparison, \cgilao starts with no actions, and adds them as they are deemed necessary.
This means that \cgilao starts with a small partial SSP, and grows it over time, enjoying a relatively small partial SSP during its entire execution; whereas action elimination has a large partial SSP for most of its execution, only shrinking it when \Vlb and \Vub become sufficiently accurate to prove actions are suboptimal.

To summarise, action elimination has two key disadvantages compared to \cgilao's mechanism for ignoring actions:
\begin{enumerate}
\item Action elimination requires both a lower bound \Vlb and an upper bound \Vub, whereas \cgilao only needs \Vlb.
Maintaining \Vub incurs an overhead of \qvalue computations.
\item Action elimination is only able to shrink its partial SSPs late in its lifetime, whereas \cgilao enjoys smaller partial SSPs from the start.
\end{enumerate}
We show experimentally in \cref{sec:action-elimination} that \cgilao's mechanism for ignoring actions is indeed more efficient than action elimination.

\subsection{Constraint Generation in Planning}

The paradigm of constraint generation lends itself well to complex planning problems where a relaxation can be solved efficiently.
An example that demonstrates this very clearly is Multi-Agent Path-Finding (MAPF)~\cite{Stern2021:MAPF}, where the aim is to coordinate multiple agents to achieve their objectives as efficiently as possible, making sure that the agents do not interfere with each, e.g., they should not crash into each other.
This problem is notoriously difficult, especially as the number of agents increases;
finding optimal solutions is NP-hard~\cite{Yu2013:MAPF-is-NP}, and in some cases even finding a feasible solution is NP-hard~\cite{Nebel2020:MAPF-is-NP}.
On the other hand, if we focus on each agent in isolation, then finding their individual optimal paths is a classical planning problem, which can be solved efficiently.
If the isolated optimal paths do not interfere, then this immediately gives an optimal solution for the MAPF problem, but in general there will be interferences that must be addressed.
We can impose constraints that block such interferences; for example, suppose that following the isolated paths causes agents \(A\) and \(B\) to crash into each other in location \((x, y)\) at time \(t\), then we can prevent the crash by disallowing \(A\) from entering \((x, y)\) at time \(t\).
Then, a viable algorithmic loop is to find paths for the agents in isolation, determine where they interfere, add constraints that block the undesired behaviour, then find updated plans for the agents in isolation that respect the added constraints, and repeat until no interferences remain.
The MAPF community calls this technique Conflict Based Search (CBS)~\cite{Sharon2012:CBS}, and it has seen a reasonable amount of success~\cite{Boyarski2015:iCBS}.
Constraint generation has also been used explicitly (by name) in MAPF with Mixed Integer Programs~\cite{Calliess2021:ConstraintGenerationMultiAgentMIP}.
The same paradigm applies similarly to other complex planning problems, e.g., in metric hybrid factored planning in nonlinear domains~\cite{Say2019:ConstraintGenerationHybridPlanning}.

Counter-Example Guided Abstraction Refinement (CEGAR)~\cite{Seipp2013:CEGAR} can also be considered a form of constraint generation, which has seen success in constructing heuristics~\cite{Rovner2019:CEGAR-heuristics} and policy verification~\cite{Vinzent2023}.
In the context of classical planning, the idea is to use abstraction to compress the original problem's state-space, so that the ensuing problem is smaller and easier to solve, but still captures some of the original problem's structure.
Then, we can solve the smaller problem to obtain a plan.
A plan for the smaller problem may immediately be a solution for the original problem, in which case the problem is solved; more likely, the plan has some \textit{flaw} when we try to apply it to the original problem.
This is called a counter-example, and is used to refine the abstraction so that this particular flaw can not occur again, effectively adding constraints prohibiting it.
Consequently, we get the algorithmic loop that solves an abstraction, finds flaws, refines the abstraction to prohibit these flaws, resolves the abstraction, and then repeats until no more flaws can be found.
Again, this fits the framework of constraint generation.

These constraint-generation approaches differ quite significantly from \cgilao, because their separation oracle determines how a candidate solution violates the constraint of the original planning problem, whereas \cgilao's separation oracle determines violations of monotonicity in the value function.
There is also existing work in the planning literature that applies constraint generation to various value functions.

For Partially Observable MDPs (POMDPs) the value function is encoded by a set of vectors, and an LP with constraint generation can prune unneeded vectors~\cite{Walraven2017:VectorPruningPOMDP}.
Although this LP is related to the value function, pruning vectors is a fundamentally different task from the one we are addressing.
Similarly, a linear value function can give a compact approximation of the value function for factored MDPs, and if we replace the value function in \ref{lp:vi} with this approximation, we obtain the approximate LP (ALP)~\cite{Mausam2012:MDPs} with a significant reduction in the number of variables.
\citeAuthor{Schuurmans2001:ConstrGenMDP} apply constraint generation to the ALP, and the separation oracle has a similar condition for adding constraints as \cgilao; however, it checks for the condition na{\"i}vely, which is practical with the compact representation of factored MDPs with a linear value function approximation, but not for SSPs.

In all these works, to find constraint violations the separation oracle either na{\"i}vely checks all possible constraints or relies on sampling, compared to \cgilao, which efficiently tracks potential violations via changes in the value function.

\subsection{Partial Expansion \astar}

The concept of ignoring unneeded actions has been considered in the deterministic shortest path setting by the Partial Expansion \astar (\peastar) algorithm~\cite{Yoshizumi2000:PEAstar}.
Similarly to \cgilao, states are partially expanded: the algorithm only considers a greedy subset of applicable actions, and efficiently tracks the other actions to be reconsidered when they may become needed.
We assume that the reader is familiar with \astar~\cite{Hart1968:astar,Russell2016}, and we only give a brief reminder of how the algorithm works.
We use similar notation to \citeAuthor{Yoshizumi2000:PEAstar}, but talk about states instead of nodes.
Then, for a state \s
\begin{itemize}
\item \(g(\s)\) denotes the lowest path cost to \s, i.e., the currently best-found cost to move from \sZ to \s;
\item \(h(\s)\) is an admissible heuristic function evaluated at \s (defined similarly to heuristic functions in our setting);
\item \(f(\s) = g(\s) + h(\s)\) is the evaluation function.
\end{itemize}
\astar maintains an \textit{open list} of states (also called the frontier), from which the algorithm iteratively pops a state \s with lowest \(f(\s)\) and then expands it.
We assume that the open list is implemented as a priority queue ordered by \(f\).
To expand \s, \astar iterates over \s's successors and places them into the open list with their respective \(g\), \(h\), and \(f\) values.
A key feature of \astar is that it will stop as soon as it pops a goal from the open list, and any states that remain in the open list are discarded, not needing to be considered for constructing the optimal plan.
For problems with large branching factors the number of such discarded states can get very large, which makes it expensive maintain the priority queue, and in extreme cases causes the algorithm to run out of memory.
\peastar addresses this by using partial expansions.
A partial expansion for \s is defined as follows: let \(\text{SUCC}\) denote the successor states of \s, then the partial expansion only adds the greedy successor states
\[
\text{SUCC}_{\leq 0} = \left\{ \s_{\min} \in \text{SUCC} : f(\s_{\min}) = \min_{\s' \in \text{SUCC}} f(\s') \right\}
\]
and all non-greedy successors are captured by placing the parent state \s back into the open list with an updated priority
\[F(\s) = \min_{\s' \in \text{SUCC} \setminus \text{SUCC}_{\leq 0}} f(\s').\]
This \(F(\s)\) can be read as the smallest \(f\)-value of \s's non-greedy successors.
Note that \peastar uses \(F\) to sort its priority queue, and we use \(F(\s) = f(\s)\) if a state has not been partially expanded before.
The key idea is that, if search determines that the states we have ignored become relevant, then \s will be popped again as a proxy for its successors, and the partial re-expansion of \s will place the newly greedy successors into the open list, and progress as though these states were always on the open list.
Thus, the algorithm can maintain a smaller open list by ignoring unpromising states, at the cost of requiring additional re-expansion steps.

If \ilao is the generalisation of \astar to SSPs, then \cgilao is the generalisation of \peastar, and \cgilao shares the key behaviours of \peastar when applied to deterministic problems.
It is awkward to compare \cgilao and \peastar directly, because their settings are different: \cgilao computes the states' costs-to-go in \V and tracks the current greedy policy \currp, whereas \peastar is computing the lowest path costs \(g\) and implicitly tracking the best plan by recording each state's parent.
Nevertheless, the order in which \cgilao and \peastar handle states and actions is comparable.
We illustrate this in \cref{fig:pea-vs-cgilao-pea} and \cref{fig:pea-vs-cgilao-cgilao}.
\Cref{fig:pea-vs-cgilao-pea} is the example from \citet{Yoshizumi2000:PEAstar} that demonstrates the expansion order of \peastar.
Solid and dotted circles indicate that the state has or has not been added to the open list, respectively.
The number at the bottom of each circle represents the corresponding state's priority on the open list, i.e., its \(F(\s)\), and when a value is crossed out this indicates that \(F(\s)\) was updated in this step.
\Cref{fig:pea-vs-cgilao-cgilao} shows how \cgilao behaves on the same problem.
Importantly, the numbers at the bottom of each node have a different semantic: they now represent \(\V(\s)\), and the solid and dotted circles represent states that are either inside or outside the current partial SSP.
We use the same heuristic for both algorithms, namely \(h\) with:
\[
h(\s_0) = 7,
h(\s_1) = 5,
h(\s_2) = 6,
h(\s_3) = 9,
h(\s_4) = 6,
h(\s_5) = 7,
h(\s_6) = 8,
h(\s_7) = 5,
h(\s_8) = 6,
h(\s_9) = 9.
\]
At an intuitive level, we see that the expansions of \cgilao behave similarly to \peastar's first-time expansions, and \cgilao's fixing of violated constraints lines up with \peastar's re-expansions, so that steps (a), (b), (c), and (d) match between the two figures.
Within these steps, the candidate plans are the same, with the exception that in step (b) \cgilao considers an extra action and state, since \peastar immediately recognises that \sZ needs to be re-expanded, and \cgilao needs to compute the policy first.
To understand the relationship between \cgilao and \peastar, we start by pointing out that \ilao behaves similarly to \astar: in each step of \ilao in a deterministic problem, its current policy is a plan \(p\) in the partial SSP from \sZ to \(\s \in \partgoals\) such that \(\C(p) + h(\s)\) is minimal; and \astar considers the state in its open list with minimal \(f(\s) = g(\s) + h(\s)\).
With consistent tie breaking, it turns out that in matching steps this \s is the same, and \(\C(p) = g(\s)\), so that \(\V(\sZ) = f(\s)\).
\cgilao and \peastar inherit this relationship, and it only remains to understand how \cgilao's fixing of violated constraints corresponds to \peastar's re-expansions.
This does not line up exactly, as we see in our example in step (b), because \peastar immediately places a state \s back into the open list if \(f(\s) > f(\s')\) for its parent \(\s'\), whereas \cgilao does not take such shortcuts and computes the greedy policy.
Nevertheless, the steps roughly correspond: \cgilao recognises a violated constraint when a single action can be added to improve its candidate policy; and \peastar, re-expands \s and brings in a new successor state \(\s'\) with \(F(\s') = F(\s)\), effectively improving the candidate plan by adding the missing action from \s to \(\s'\).
Thus, it seems clear that \cgilao and \peastar have near-identical behaviour on deterministic problems.

We note that the definition of \citet{Yoshizumi2000:PEAstar} is more general than what we have presented, because they allow \(\text{SUCC}_{\leq C}\) for any cutoff \(C \in \Reals_{\geq 0}\) and we only presented the special case \(C = 0\).
With the special case \(\text{SUCC}_{\leq 0}\), \peastar is analogous to \cgilao with the \(\bellmanTied\) expansion~(\cref{sec:bellman-expansions}).
Larger values of \(C\) would translate to a \cgilao expansion that adds actions that are within \(C\) of greedy.
We investigated such expansions, but we did not find them theoretically insightful, and preliminary experiments were unpromising, so we have left them out of this paper.

\begin{figure}

\begin{subfigure}[t]{0.5\textwidth}
\centering
\scalebox{0.7}{
\begin{tikzpicture}
\begin{scope}[every node/.style={circle split, thick, draw, minimum size=0.6cm, inner sep=0.04cm}]
        \node[] (0) at (0, 0) {\(\s_0\) \nodepart{lower} \(\xcancel{6} 7\)};
        \node[] (1) at (-3, -2) {\(\s_1\) \nodepart{lower} \(6\)};
        \node[dashed] (2) at (0, -2) {\(\s_2\) \nodepart{lower} \(7\)};
        \node[dashed] (3) at (3, -2) {\(\s_3\) \nodepart{lower} \(9\)};
\end{scope}
\begin{scope}[every edge/.style={draw, very thick}]
        \path [->] (0) edge[] (1);
        \path [->] (0) edge[dashed] (2);
        \path [->] (0) edge[dashed] (3);
\end{scope}
\end{tikzpicture}}
\caption{First expansion (expanding \(\s_0\)).}
\label{fig:pea-iter1}
\end{subfigure}
\begin{subfigure}[t]{0.5\textwidth}
\centering
\scalebox{0.7}{
\begin{tikzpicture}
\begin{scope}[every node/.style={circle split, thick, draw, minimum size=0.6cm, inner sep=0.04cm}]
        \node[] (0) at (0, 0) {\(\s_0\) \nodepart{lower} \(7\)};
        \node[] (1) at (-3, -2) {\(\s_1\) \nodepart{lower} \(\xcancel{6} 8\)};
        \node[dashed] (2) at (0, -2) {\(\s_2\) \nodepart{lower} \(7\)};
        \node[dashed] (3) at (3, -2) {\(\s_3\) \nodepart{lower} \(9\)};
        \node[dashed] (4) at (-3-1, -4) {\(\s_4\) \nodepart{lower} \(8\)};
        \node[dashed] (5) at (-3, -4) {\(\s_5\) \nodepart{lower} \(9\)};
        \node[dashed] (6) at (-3+1, -4) {\(\s_6\) \nodepart{lower} \(10\)};
\end{scope}
\begin{scope}[every edge/.style={draw, very thick}]
        \path [->] (0) edge[] (1);
        \path [->] (0) edge[dashed] (2);
        \path [->] (0) edge[dashed] (3);
        \path [->] (1) edge[dashed] (4);
        \path [->] (1) edge[dashed] (5);
        \path [->] (1) edge[dashed] (6);
\end{scope}
\end{tikzpicture}}
\caption{Second expansion (expanding \(\s_1\)).}
\label{fig:pea-iter2}
\end{subfigure}

\begin{subfigure}[t]{0.5\textwidth}
\centering
\scalebox{0.7}{
\begin{tikzpicture}
\begin{scope}[every node/.style={circle split, thick, draw, minimum size=0.6cm, inner sep=0.04cm}]
        \node[] (0) at (0, 0) {\(\s_0\) \nodepart{lower} \(\xcancel{7} 9\)};
        \node[] (1) at (-3, -2) {\(\s_1\) \nodepart{lower} \(8\)};
        \node[] (2) at (0, -2) {\(\s_2\) \nodepart{lower} \(7\)};
        \node[dashed] (3) at (3, -2) {\(\s_3\) \nodepart{lower} \(9\)};
        \node[dashed] (4) at (-3-1, -4) {\(\s_4\) \nodepart{lower} \(8\)};
        \node[dashed] (5) at (-3, -4) {\(\s_5\) \nodepart{lower} \(9\)};
        \node[dashed] (6) at (-3+1, -4) {\(\s_6\) \nodepart{lower} \(10\)};
\end{scope}
\begin{scope}[every edge/.style={draw, very thick}]
        \path [->] (0) edge[] (1);
        \path [->] (0) edge[] (2);
        \path [->] (0) edge[dashed] (3);
        \path [->] (1) edge[dashed] (4);
        \path [->] (1) edge[dashed] (5);
        \path [->] (1) edge[dashed] (6);
\end{scope}
\end{tikzpicture}}
\caption{Third expansion (re-expanding \(\s_0\)).}
\label{fig:pea-iter3}
\end{subfigure}
\begin{subfigure}[t]{0.5\textwidth}
\centering
\scalebox{0.7}{
\begin{tikzpicture}
\begin{scope}[every node/.style={circle split, thick, draw, minimum size=0.6cm, inner sep=0.04cm}]
        \node[] (0) at (0, 0) {\(\s_0\) \nodepart{lower} \(9\)};
        \node[] (1) at (-3, -2) {\(\s_1\) \nodepart{lower} \(8\)};
        \node[] (2) at (0, -2) {\(\s_2\) \nodepart{lower} \(\xcancel{7} 8\)};
        \node[dashed] (3) at (3, -2) {\(\s_3\) \nodepart{lower} \(9\)};
        \node[dashed] (4) at (-3-1, -4) {\(\s_4\) \nodepart{lower} \(8\)};
        \node[dashed] (5) at (-3, -4) {\(\s_5\) \nodepart{lower} \(9\)};
        \node[dashed] (6) at (-3+1, -4) {\(\s_6\) \nodepart{lower} \(10\)};
        \node[] (7) at (-1, -4) {\(\s_7\) \nodepart{lower} \(7\)};
        \node[dashed] (8) at (0, -4) {\(\s_8\) \nodepart{lower} \(8\)};
        \node[dashed] (9) at (+1, -4) {\(\s_9\) \nodepart{lower} \(11\)};
\end{scope}
\begin{scope}[every edge/.style={draw, very thick}]
        \path [->] (0) edge[] (1);
        \path [->] (0) edge[] (2);
        \path [->] (0) edge[dashed] (3);
        \path [->] (1) edge[dashed] (4);
        \path [->] (1) edge[dashed] (5);
        \path [->] (1) edge[dashed] (6);
        \path [->] (2) edge[] (7);
        \path [->] (2) edge[dashed] (8);
        \path [->] (2) edge[dashed] (9);
\end{scope}
\end{tikzpicture}}
\caption{Fourth expansion (expanding \(\s_2\)).}
\label{fig:pea-iter4}
\end{subfigure}

\caption{\peastar's behaviour (comparing \cgilao and \peastar).}
\label{fig:pea-vs-cgilao-pea}

\begin{subfigure}[t]{0.5\textwidth}
\centering
\scalebox{0.7}{
\begin{tikzpicture}
\begin{scope}[every node/.style={circle split, thick, draw, minimum size=0.6cm, inner sep=0.04cm}]
        \node[line width=2pt] (0) at (0, 0) {\(\s_0\) \nodepart{lower} \(6\)};
        \node[line width=2pt] (1) at (-3, -2) {\(\s_1\) \nodepart{lower} \(5\)};
        \node[dashed] (2) at (0, -2) {\(\s_2\) \nodepart{lower} \(6\)};
        \node[dashed] (3) at (3, -2) {\(\s_3\) \nodepart{lower} \(8\)};
\end{scope}
\begin{scope}[every edge/.style={draw, very thick}]
        \path [->] (0) edge[line width=2pt] (1);
        \path [->] (0) edge[dashed] (2);
        \path [->] (0) edge[dashed] (3);
\end{scope}
\end{tikzpicture}}
\caption{First expansion (expanding \(\s_0\)).}
\label{fig:cgilao-pea-iter1}
\end{subfigure}
\begin{subfigure}[t]{0.5\textwidth}
\centering
\scalebox{0.7}{
\begin{tikzpicture}
\begin{scope}[every node/.style={circle split, thick, draw, minimum size=0.6cm, inner sep=0.04cm}]
        \node[line width=2pt] (0) at (0, 0) {\(\s_0\) \nodepart{lower} \(\xcancel{6} 8\)};
        \node[line width=2pt] (1) at (-3, -2) {\(\s_1\) \nodepart{lower} \(\xcancel{5} 7\)};
        \node[dashed] (2) at (0, -2) {\(\s_2\) \nodepart{lower} \(6\)};
        \node[dashed] (3) at (3, -2) {\(\s_3\) \nodepart{lower} \(8\)};
        \node[line width=2pt] (4) at (-3-1, -4) {\(\s_4\) \nodepart{lower} \(6\)};
        \node[dashed] (5) at (-3, -4) {\(\s_5\) \nodepart{lower} \(7\)};
        \node[dashed] (6) at (-3+1, -4) {\(\s_6\) \nodepart{lower} \(8\)};
\end{scope}
\begin{scope}[every edge/.style={draw, very thick}]
        \path [->] (0) edge[line width=2pt] (1);
        \path [->] (0) edge[dashed] (2);
        \path [->] (0) edge[dashed] (3);
        \path [->] (1) edge[line width=2pt] (4);
        \path [->] (1) edge[dashed] (5);
        \path [->] (1) edge[dashed] (6);
\end{scope}
\end{tikzpicture}}
\caption{Second expansion (expanding \(\s_1\)).}
\label{fig:cgilao-pea-iter2}
\end{subfigure}

\begin{subfigure}[t]{0.5\textwidth}
\centering
\scalebox{0.7}{
\begin{tikzpicture}
\begin{scope}[every node/.style={circle split, thick, draw, minimum size=0.6cm, inner sep=0.04cm}]
        \node[line width=2pt] (0) at (0, 0) {\(\s_0\) \nodepart{lower} \(\xcancel{8} 7\)};
        \node[] (1) at (-3, -2) {\(\s_1\) \nodepart{lower} \(7\)};
        \node[line width=2pt] (2) at (0, -2) {\(\s_2\) \nodepart{lower} \(6\)};
        \node[dashed] (3) at (3, -2) {\(\s_3\) \nodepart{lower} \(8\)};
        \node[] (4) at (-3-1, -4) {\(\s_4\) \nodepart{lower} \(6\)};
        \node[dashed] (5) at (-3, -4) {\(\s_5\) \nodepart{lower} \(7\)};
        \node[dashed] (6) at (-3+1, -4) {\(\s_6\) \nodepart{lower} \(8\)};
\end{scope}
\begin{scope}[every edge/.style={draw, very thick}]
        \path [->] (0) edge[] (1);
        \path [->] (0) edge[line width=2pt] (2);
        \path [->] (0) edge[dashed] (3);
        \path [->] (1) edge[] (4);
        \path [->] (1) edge[dashed] (5);
        \path [->] (1) edge[dashed] (6);
\end{scope}
\end{tikzpicture}}
\caption{After fixing violated constraints from second expansion.}
\label{fig:cgilao-pea-iter3}
\end{subfigure}
\begin{subfigure}[t]{0.5\textwidth}
\centering
\scalebox{0.7}{
\begin{tikzpicture}
\begin{scope}[every node/.style={circle split, thick, draw, minimum size=0.6cm, inner sep=0.04cm}]
        \node[line width=2pt] (0) at (0, 0) {\(\s_0\) \nodepart{lower} \(7\)};
        \node[] (1) at (-3, -2) {\(\s_1\) \nodepart{lower} \(7\)};
        \node[line width=2pt] (2) at (0, -2) {\(\s_2\) \nodepart{lower} \(6\)};
        \node[dashed] (3) at (3, -2) {\(\s_3\) \nodepart{lower} \(8\)};
        \node[] (4) at (-3-1, -4) {\(\s_4\) \nodepart{lower} \(6\)};
        \node[dashed] (5) at (-3, -4) {\(\s_5\) \nodepart{lower} \(7\)};
        \node[dashed] (6) at (-3+1, -4) {\(\s_6\) \nodepart{lower} \(8\)};
        \node[line width=2pt] (7) at (-1, -4) {\(\s_7\) \nodepart{lower} \(5\)};
        \node[dashed] (8) at (0, -4) {\(\s_8\) \nodepart{lower} \(6\)};
        \node[dashed] (9) at (+1, -4) {\(\s_9\) \nodepart{lower} \(9\)};
\end{scope}
\begin{scope}[every edge/.style={draw, very thick}]
        \path [->] (0) edge[] (1);
        \path [->] (0) edge[line width=2pt] (2);
        \path [->] (0) edge[dashed] (3);
        \path [->] (1) edge[] (4);
        \path [->] (1) edge[dashed] (5);
        \path [->] (1) edge[dashed] (6);
        \path [->] (2) edge[line width=2pt] (7);
        \path [->] (2) edge[dashed] (8);
        \path [->] (2) edge[dashed] (9);
\end{scope}
\end{tikzpicture}}
\caption{Third expansion (expanding \(\s_2\)).}
\label{fig:cgilao-pea-iter4}
\end{subfigure}

\caption{\cgilao's behaviour on deterministic problem (comparing \cgilao and \peastar). Thick lines indicate the greedy policy envelope.}
\label{fig:pea-vs-cgilao-cgilao}
\end{figure}

\clearpage

\section{Experiments}\label{sec:experiments}

In this section, we experimentally determine the most efficient variants of \cgilao, demonstrate how these variants compete with the state-of-the-art, and analyse these results in detail.
We start by presenting the methodology of our experiments (\cref{sec:experiments-methodology}), describing the benchmark domains that we consider (\cref{sec:experiments-domains}), and we explain the tables and plots that we use to present the results (\cref{sec:experiments-presentation-of-results}).
Then, we compare the versions of \cgilao to identify the best performers~(\cref{sec:experiments-best-version-of-cgilao}), and show that \cgilao's best performers are competitive with the state-of-the-art~(\cref{sec:experiments-state-of-the-art}).
We investigate \cgilao further and show how many actions it is able to ignore~(\cref{sec:experiments-partial-ssp-sizes}), and how its mechanism for ignoring actions compares to action elimination~(\cref{sec:experiments-action-elimination}).

\subsection{Methodology}\label{sec:experiments-methodology}

We compare different versions of \cgilao with the state-of-the-art optimal heuristic-search algorithms \ilao~\cite{Hansen2001:ilao} and \lrtdp~\cite{Bonet2003:lrtdp}.
We will give more details about the algorithms we use in the relevant sections.
Heuristics have a major impact on the performance of heuristic-search algorithms, so we consider the following suite of admissible heuristics from the literature: \lmcut~\cite{Helmert2009:lmcut}; \roc~\cite{trevizan17:hpom}; and \pdb~\cite{Kloessner2021:MAXPROB-PDBs,Kloessner2021:SSP-PDBs}.

Following the original authors of \roc, on problems with dead ends we use \roc with \hMax~\cite{Bonet2000:hmax} as a dead-end detection mechanism.
We only consider PDBs generated up to size 2, following the results of~\citeAuthor{Kloessner2021:SSP-PDBs}.
We convert SSPs into dead-end free SSPs with the fixed-penalty transformation~\cite{trevizan17:mcmp} with a penalty of
\begin{itemize}
\item \(D = 10^4\) for \elevators, \recttireworld, and \schedule;
\item \(D = 10^7\) for \parcr, \parcn, and \sysadmin;
\item \(D = 500\) for all other domains.
\end{itemize}
The domains with \(D > 500\) require large penalties for different reasons.
\parcn and \parcr have very large action costs.
On problems such as \recttireworld, the optimal policy \(\p^*\)'s probability of reaching a goal (without using give-up actions) is very low, with the effect that the cost of \(\p^*\) is within \(\epsilon\) of \(500\), so \(D\) needs to be increased.
On problems such as \elevators, an optimal policy requires the agent to repeat a sequence of actions many times before it reaches a goal, so that the optimal policy's cost exceeds \(500\).
As the parameter for \econsistency, we use \(\epsilon = 0.0001\), and we set \cgilao's parameter \(\eta = \epsilon\).

For each triple of algorithm, heuristic, and problem, we run it \data{5} times with different random seeds.
These triples coupled with a random seed are called \textbf{instances}.
We run multiple instances per problem because, even if all other parameters are fixed, the random seed can cause a significant amount of variance: for algorithms that rely on random choices this is obvious (e.g., \lrtdp), but even for seemingly deterministic algorithms, random seeds have a significant impact on behaviour in practice (e.g., on tie breaking).
To emphasise this point, we have observed a difference of over \data{1000} seconds on the same problem, algorithm, and heuristic, only changing the random seed.
For each instance, the execution is cut off at \data{30 minutes of CPU time and 8GB of memory}.
The experiments were conducted in a cluster of \data{Intel Xeon 3.2 GHz} CPUs and each run used a single CPU core.
The LP solver used for computing \roc was \data{CPLEX version 20.1}.

Our code, benchmarks, and results are available online~\cite{Schmalz2024:CgilaoSourceCode}.
We now give descriptions of our benchmark domains.

\subsection{Benchmark Domains}\label{sec:experiments-domains}

We use the benchmark problems from \citeAuthor{Kloessner2021:Benchmarks}, which is a compilation of domains from the International Probabilistic Planning Competitions (IPPC) that were run in 2004~\cite{Younes2005:ippc}, 2006~\cite{Bonet2006:ippc}, and 2008~\cite{Bryce2008:ippc}.
These IPPCs encoded problems in Probabilistic PDDL~\cite{Younes2005:ippc}; we do not consider later competitions because they use RDDL~\cite{Sanner2010:RDDL}, which is incompatible with our heuristics.
We consider three additional domains, which we also describe in this section.
To reduce wasted compute time, we eliminate problems that are too difficult to solve.
For a domain with problems enumerated from \(1\) to \(n\), we consider the problems \(i\) to \(n\) too difficult if none of these were solved by \lrtdp, \ilao, or \cgilao with any heuristics, given an extended deadline of 45 minutes.
We now describe each domain and list each IPPC where the domain appeared.
We omit the domain ``Boxworld'' because all of its problems are too difficult, and we omit ``Drive'' and ``Tireworld'' because their problems are too easy: all algorithm configurations solved all problems in under 30 seconds (with a few exceptions).

\textbf{Blocks World (\bw) [IPPC 2004, 2006, 2008].}
Deterministic Blocks World has featured in many of the (deterministic) International Planning Competitions (IPCs).
A Blocks World problem consists of a table that has \(n\) unique blocks on it which can be stacked on top of each other, i.e., a block can be placed on top of any other, as long as there is no block on top of it already.
The agent is presented with an initial configuration of blocks, and its task is to rearrange them into a specified configuration, e.g., block 1 should be on block 2, and block 2 should be on the table.
To achieve this, the agent controls a ``gripper'' that can pick up one block at a time that has no other blocks on top of it, and then the gripper can put this block down on the table or on top of another block.
The probabilistic version of Blocks World, which we consider, introduces to each of these actions a 0.25 probability that the block being handled slips out of the gripper and falls on the table.
The probabilistic version also adds high-risk high-reward actions that enable the agent to handle a tower of two blocks, i.e., it can pick up a tower of two and then place it on the table or on top of a block, but it only succeeds with probability 0.1.
In this domain, any action's effects can be reversed, e.g., if the agent tries to place block 1 but drops it, it can repeat the pick-up action until it is holding block 1, returning it to its previous state.
Consequently, the problems have no dead ends.

\textbf{Core Security Pentesting (\coresec)~\cite{Steinmetz2016,Kloessner2021:Benchmarks}.}
This domain represents a penetration testing (pentesting) scenario, where the agent aims to exploit vulnerabilities in a network to compromise specific targets in the fewest possible steps.
Each exploit can be used once, succeeding or failing with some probability.
Thus, \coresec problems violate reachability (\cref{assump:reachability}), because any exploit that is necessary to breach the network may fail, and since the exploit can not be used again, the agent has no way to achieve its goal.
Pentesting is a real-world application of SSPs, where real-world networks are modelled and presented to the agent; the agent's policies for breaching the network help administrators identify and address vulnerabilities.
This particular domain is indirectly based on a test scenario from Core Security~(\url{http://www.coresecurity.com/}).
This domain has not featured in any IPPC.

\textbf{Elevators (\elevators) [IPPC 2006].}
The agent is tasked with collecting all coins distributed within a building.
The building can be visualised as a grid with \(y\) floors and \(x\) horizontal positions.
Within a floor, the agent can freely walk between the horizontal positions.
To move vertically between floors, the agent must use an elevator, which are available at particular horizontal positions.
If the agent is at the same floor and horizontal position as an elevator, it can enter and move between floors as expected.
The coins that the agent must collect are distributed throughout locations of the building, and the agent must try to navigate between them in the minimum number of steps.
The probabilistic factor is that all locations except the ground floor have a ``gate'', and whenever the agent walks past such a gate it has a \(0.5\) probability of falling to the left-most part of the ground floor.
Thus, the agent should avoid walking past gates unnecessarily.
This domain has no dead ends because the agent can simply try to move to the next coin repeatedly, regardless of how many times it gets sent back to the left-most part of the ground floor.
Note that policies can get very expensive, particularly if there is a coin to which the closest elevator is far away horizontally, since it forces the agent to walk past many gates.

\textbf{Exploding Blocks World (\exbw) [IPPC 2004, 2006, 2008].}
This domain, like \bw, is also based on the deterministic Blocks World from IPC.
Instead of the block slipping out of the agent's gripper, the probabilistic element is that each block is rigged with an explosive that can detonate once, instantly destroying the table or the block directly beneath it.
When a block is placed, it detonates with probability 0.4 if placed on the table and 0.1 if placed on another block.
Once the table or a block is destroyed, the agent can no longer interact with them, and if they were not in their goal position, the goal becomes unreachable.
This means \exbw problems can have unavoidable dead ends: there may be a non-zero probability of reaching a state where a necessary block or the table have been destroyed.
Under the fixed-penalty transformation, planners must avoid policies that incur the give-up penalty with high probability.
Earlier versions of this domain had a bug where a block can be placed on top of itself; this bug is fixed in all of the problems we consider.

\textbf{Probabilistic PARC Printer (\parc)~\cite{trevizan17:hpom}.}
The deterministic PARC Printer domain from the IPC models a modular printer consisting of various components, where each page scheduled for printing must pass through multiple components in a specific order.
The agent's task is to print a document of \(s\) pages in as few steps as possible, where each page can have different requirements which require different components in the printer.
For example, the printer has coloured and black-and-white printing components, so pages in the document must be routed accordingly.
In the probabilistic variant, \(c\) components have a 0.1 probability of jamming, rendering the component unusable and requiring the affected page to be reprinted.
The domain has two variants:
\begin{itemize}
\item \parcr (with repair): Jammed components can be repaired and reused with a high cost.
\item \parcn (without repair): Jammed components remain permanently unusable.
\end{itemize}
This domain does not come from \citeAuthor{Kloessner2021:Benchmarks}, and has not featured in any IPPC.

\textbf{Random (\random) [IPPC 2006].}
In this domain, problems are generated randomly with no associated semantics.
\random is exceptional among the domains we consider, because it can have significantly more applicable actions in each state.
The other domains have a number applicable actions on the order of \(10^1\), occasionally \(10^2\); whereas \random has on the order of \(10^4\) (see \cref{tab:partial-ssp-sizes-some-problems} for the number of states and actions considered in a particular \random problem).

\textbf{Schedule (\schedule) [IPPC 2006, 2008].}
This is a network scheduling scenario.
There are \(m\) distinct packets being sent through the network, and \(n\) ``classes'' that can receive them.
The agent's task is to ensure that each class receives and serves one or more packets, with the network staying alive.
In each timestep, each class may randomly receive an incoming packet, which sets the packet's ``time-to-live'' to \(k\) timesteps.
Then, the agent must choose whether the class should serve the packet, or hold onto it until the time-to-live expires, which causes the packet to be reclaimed, i.e., the current class no longer has access to the packet, and the packet gets randomly sent to one of the classes, as before.
A complicating factor is that, whenever a packet gets reclaimed, the previous owner class has a small probability of locking up, which makes the network not alive, i.e., the agent fails its task.
Different classes have different probabilities of locking up, so the agent should use the low-probability ones for reclaiming packets, and avoid reclaiming on the high-probability ones.

\textbf{Search and Rescue (\sar) [IPPC 2008].}
An autonomous helicopter must find a safe spot to land in order to save a survivor.
There are \(n\) many locations, and it is not known a priori whether it is possible to land there or not.
The agent must explore the locations -- each explore action probabilistically determines whether it is possible to land at that location, or not.
Once the agent finds a safe location to land, it can pick up and deliver the survivor to safety.
The more actions the agent applies, the less likely the human survives.
The agent's task is to maximise the human's chance of survival.

\textbf{Sysadmin (\sysadmin) [IPPC 2008].}
The agent must manage a network of servers.
In each timestep, a server has a probability of failing.
If the server's neighbouring servers are all operational, then this server's probability of failing is low (probability \(0.05\)), but if any of its neighbours have failed, then this server's probability of failing is significantly higher (probability \(0.25\)).
The only decision available to the agent is to reboot one problematic server per timestep, fixing it with probability \(0.9\).
In the version we consider, all servers are initially broken, and the agent must make all servers operational while minimising the number of reboots over expectation.
Note that \sysadmin has zero-cost actions to handle the bookkeeping of which servers fail, which violates our assumption that costs are strictly positive.
However, this is not an issue because we manually verified that there are no zero-cost cycles and \cref{assump:infinite-improper} holds.

\textbf{Triangle Tire World (\tritireworld) [IPPC 2008]}
The agent is presented with a network of locations arranged in a triangular layout with corners \(A, B, C\), as in \cref{fig:tireworld}, and the agent's task is to drive from corner \(A\) to \(C\).
Each time the agent moves, the car gets a flat tyre with probability $0.5$.
If the car gets a flat tyre the agent must replace the flat with a spare tyre if it has one; when the agent has no spare tyre, the car is stuck, i.e., it is in a dead end with no actions available.
The agent can only load spare tyres in specific locations (circles in \cref{fig:tireworld}), and can only keep one spare tyre at a time.
This problem was introduced by \citeAuthor{Little2007:PlannerVsReplanner} to be probabilistically interesting: if the probabilistic effects are ignored, the agent will take the direct path from \(A\) to \(C\), which maximises the probability of encountering a dead end; whereas an optimal policy navigates from \(A\) to \(B\) to \(C\), so that spare tyres are always available and it can reach the goal with probability 1.
Thus, \tritireworld has dead ends, but satisfies the reachability assumption (\cref{assump:reachability}) because a proper policy exists.
\tritireworld problems are specified by \(n \in \Naturals\).
We show the layouts of \tritireworld with \(n=1,2,3\) in \cref{fig:tireworld}; the problem with \(n=1\) can be considered a base case, and the problems with \(n>1\) are constructed by connecting \(n\) copies of the \(n=1\) problem along each edge of the triangle.
With this construction, the state-space grows exponentially with respect to \(n\), which is inconvenient for running experiments: it can happen that problem \(n\) is solved easily by all algorithms, but all algorithms time-out for \(n+1\).
To address this shortcoming, we introduced \tritireworld with Head-start~\cite{Schmalz2024:cgilao}, which gives a finer granularity of problem difficulty.
This extension was not taken from \citeAuthor{Kloessner2021:Benchmarks}, and has not featured in any IPPCs.
A \tritireworld with Head-start problem is defined by \(\tritireworld(n, d)\), where \(n \in \Naturals\) denotes the problem size as before, and \(d \in \Naturals\) denotes that the agent is given a head-start along the edge \(AB\), thus the agent has the initial location \(x\)-\(y\) with \(x=(d+1)\) and \(y=1\) in \cref{fig:tireworld}.
For fixed \(n\), \(d\) can take the values \(0 \leq d \leq 2n\) where
\begin{itemize}
\item \(d=0\) defines the original \tritireworld problem, with no head-start, where the agent starts at corner \(A\),
\item \(d=2n\) is the easiest variant of the problem where the agent is given the biggest head-start starting at corner \(B\).
\end{itemize}

So, by starting at \(d=0\) we have the original \tritireworld problem, and by increasing \(d\) we make the problem easier, giving us problems of intermediate difficulty between the standard \tritireworld problems.

\newcommand\triangleScale{0.6}
\newcommand\labelTightness{0.1cm}
\begin{figure}[ht!]
\centering
\scalebox{\triangleScale}{
\begin{tikzpicture}
\Vertex[x=0.350,y=0.350,color={230,230,230},label=1-1,shape=rectangle,RGB]{l-1-1}
\Vertex[x=4.150,y=0.350,color={230,230,230},label=1-3,shape=rectangle,RGB]{l-1-3}
\Vertex[x=2.250,y=0.350,color={230,230,230},label=1-2,shape=rectangle,RGB]{l-1-2}
\Vertex[x=1.300,y=1.500,color={230,230,230},label=2-1,shape=circle,RGB]{l-2-1}
\Vertex[x=3.200,y=1.500,color={230,230,230},label=2-2,shape=circle,RGB]{l-2-2}
\Vertex[x=2.250,y=2.650,color={230,230,230},label=3-1,shape=circle,RGB]{l-3-1}
\node [below = \labelTightness of l-1-1]{A};
\node [right = \labelTightness of l-3-1]{B};
\node [below = \labelTightness of l-1-3]{C};
\Edge[,Direct](l-1-1)(l-1-2)
\Edge[,Direct](l-1-1)(l-2-1)
\Edge[,Direct](l-1-2)(l-1-3)
\Edge[,Direct](l-1-2)(l-2-2)
\Edge[,Direct](l-2-1)(l-1-2)
\Edge[,Direct](l-2-1)(l-3-1)
\Edge[,Direct](l-2-2)(l-1-3)
\Edge[,Direct](l-3-1)(l-2-2)
\end{tikzpicture}
\hspace{5mm}
\begin{tikzpicture}
\Vertex[x=0.350,y=0.350,color={230,230,230},label=1-1,shape=rectangle,RGB]{l-1-1}
\Vertex[x=7.150,y=0.350,color={230,230,230},label=1-5,shape=rectangle,RGB]{l-1-5}
\Vertex[x=2.050,y=0.350,color={230,230,230},label=1-2,shape=rectangle,RGB]{l-1-2}
\Vertex[x=3.750,y=0.350,color={230,230,230},label=1-3,shape=rectangle,RGB]{l-1-3}
\Vertex[x=5.450,y=0.350,color={230,230,230},label=1-4,shape=rectangle,RGB]{l-1-4}
\Vertex[x=1.200,y=1.425,color={230,230,230},label=2-1,shape=circle,RGB]{l-2-1}
\Vertex[x=2.900,y=1.425,color={230,230,230},label=2-2,shape=circle,RGB]{l-2-2}
\Vertex[x=4.600,y=1.425,color={230,230,230},label=2-3,shape=circle,RGB]{l-2-3}
\Vertex[x=6.300,y=1.425,color={230,230,230},label=2-4,shape=circle,RGB]{l-2-4}
\Vertex[x=2.050,y=2.500,color={230,230,230},label=3-1,shape=circle,RGB]{l-3-1}
\Vertex[x=5.450,y=2.500,color={230,230,230},label=3-3,shape=circle,RGB]{l-3-3}
\Vertex[x=3.750,y=2.500,color={230,230,230},label=3-2,shape=rectangle,RGB]{l-3-2}
\Vertex[x=2.900,y=3.575,color={230,230,230},label=4-1,shape=circle,RGB]{l-4-1}
\Vertex[x=4.600,y=3.575,color={230,230,230},label=4-2,shape=circle,RGB]{l-4-2}
\Vertex[x=3.750,y=4.650,color={230,230,230},label=5-1,shape=circle,RGB]{l-5-1}
\node [below = \labelTightness of l-1-1]{A};
\node [right = \labelTightness of l-5-1]{B};
\node [below = \labelTightness of l-1-5]{C};
\Edge[,Direct](l-1-1)(l-1-2)
\Edge[,Direct](l-1-1)(l-2-1)
\Edge[,Direct](l-1-2)(l-1-3)
\Edge[,Direct](l-1-2)(l-2-2)
\Edge[,Direct](l-1-3)(l-1-4)
\Edge[,Direct](l-1-3)(l-2-3)
\Edge[,Direct](l-1-4)(l-1-5)
\Edge[,Direct](l-1-4)(l-2-4)
\Edge[,Direct](l-2-1)(l-1-2)
\Edge[,Direct](l-2-1)(l-3-1)
\Edge[,Direct](l-2-2)(l-1-3)
\Edge[,Direct](l-2-3)(l-1-4)
\Edge[,Direct](l-2-3)(l-3-3)
\Edge[,Direct](l-2-4)(l-1-5)
\Edge[,Direct](l-3-1)(l-3-2)
\Edge[,Direct](l-3-1)(l-2-2)
\Edge[,Direct](l-3-1)(l-4-1)
\Edge[,Direct](l-3-3)(l-2-4)
\Edge[,Direct](l-3-2)(l-3-3)
\Edge[,Direct](l-3-2)(l-4-2)
\Edge[,Direct](l-4-1)(l-3-2)
\Edge[,Direct](l-4-1)(l-5-1)
\Edge[,Direct](l-4-2)(l-3-3)
\Edge[,Direct](l-5-1)(l-4-2)
\end{tikzpicture}
\hspace{5mm}
\begin{tikzpicture}
\Vertex[x=0.350,y=0.350,color={230,230,230},label=1-1,shape=rectangle,RGB]{l-1-1}
\Vertex[x=10.150,y=0.350,color={230,230,230},label=1-7,shape=rectangle,RGB]{l-1-7}
\Vertex[x=1.983,y=0.350,color={230,230,230},label=1-2,shape=rectangle,RGB]{l-1-2}
\Vertex[x=3.617,y=0.350,color={230,230,230},label=1-3,shape=rectangle,RGB]{l-1-3}
\Vertex[x=5.250,y=0.350,color={230,230,230},label=1-4,shape=rectangle,RGB]{l-1-4}
\Vertex[x=6.883,y=0.350,color={230,230,230},label=1-5,shape=rectangle,RGB]{l-1-5}
\Vertex[x=8.517,y=0.350,color={230,230,230},label=1-6,shape=rectangle,RGB]{l-1-6}
\Vertex[x=1.167,y=1.400,color={230,230,230},label=2-1,shape=circle,RGB]{l-2-1}
\Vertex[x=2.800,y=1.400,color={230,230,230},label=2-2,shape=circle,RGB]{l-2-2}
\Vertex[x=4.433,y=1.400,color={230,230,230},label=2-3,shape=circle,RGB]{l-2-3}
\Vertex[x=6.067,y=1.400,color={230,230,230},label=2-4,shape=circle,RGB]{l-2-4}
\Vertex[x=7.700,y=1.400,color={230,230,230},label=2-5,shape=circle,RGB]{l-2-5}
\Vertex[x=9.333,y=1.400,color={230,230,230},label=2-6,shape=circle,RGB]{l-2-6}
\Vertex[x=1.983,y=2.450,color={230,230,230},label=3-1,shape=circle,RGB]{l-3-1}
\Vertex[x=8.517,y=2.450,color={230,230,230},label=3-5,shape=circle,RGB]{l-3-5}
\Vertex[x=3.617,y=2.450,color={230,230,230},label=3-2,shape=rectangle,RGB]{l-3-2}
\Vertex[x=5.250,y=2.450,color={230,230,230},label=3-3,shape=rectangle,RGB]{l-3-3}
\Vertex[x=6.883,y=2.450,color={230,230,230},label=3-4,shape=rectangle,RGB]{l-3-4}
\Vertex[x=2.800,y=3.500,color={230,230,230},label=4-1,shape=circle,RGB]{l-4-1}
\Vertex[x=4.433,y=3.500,color={230,230,230},label=4-2,shape=circle,RGB]{l-4-2}
\Vertex[x=6.067,y=3.500,color={230,230,230},label=4-3,shape=circle,RGB]{l-4-3}
\Vertex[x=7.700,y=3.500,color={230,230,230},label=4-4,shape=circle,RGB]{l-4-4}
\Vertex[x=3.617,y=4.550,color={230,230,230},label=5-1,shape=circle,RGB]{l-5-1}
\Vertex[x=6.883,y=4.550,color={230,230,230},label=5-3,shape=circle,RGB]{l-5-3}
\Vertex[x=5.250,y=4.550,color={230,230,230},label=5-2,shape=rectangle,RGB]{l-5-2}
\Vertex[x=4.433,y=5.600,color={230,230,230},label=6-1,shape=circle,RGB]{l-6-1}
\Vertex[x=6.067,y=5.600,color={230,230,230},label=6-2,shape=circle,RGB]{l-6-2}
\Vertex[x=5.250,y=6.650,color={230,230,230},label=7-1,shape=circle,RGB]{l-7-1}
\node [below = \labelTightness of l-1-1]{A};
\node [right = \labelTightness of l-7-1]{B};
\node [below = \labelTightness of l-1-7]{C};
\Edge[,Direct](l-1-1)(l-1-2)
\Edge[,Direct](l-1-1)(l-2-1)
\Edge[,Direct](l-1-2)(l-1-3)
\Edge[,Direct](l-1-2)(l-2-2)
\Edge[,Direct](l-1-3)(l-1-4)
\Edge[,Direct](l-1-3)(l-2-3)
\Edge[,Direct](l-1-4)(l-1-5)
\Edge[,Direct](l-1-4)(l-2-4)
\Edge[,Direct](l-1-5)(l-1-6)
\Edge[,Direct](l-1-5)(l-2-5)
\Edge[,Direct](l-1-6)(l-1-7)
\Edge[,Direct](l-1-6)(l-2-6)
\Edge[,Direct](l-2-1)(l-1-2)
\Edge[,Direct](l-2-1)(l-3-1)
\Edge[,Direct](l-2-2)(l-1-3)
\Edge[,Direct](l-2-3)(l-1-4)
\Edge[,Direct](l-2-3)(l-3-3)
\Edge[,Direct](l-2-4)(l-1-5)
\Edge[,Direct](l-2-5)(l-1-6)
\Edge[,Direct](l-2-5)(l-3-5)
\Edge[,Direct](l-2-6)(l-1-7)
\Edge[,Direct](l-3-1)(l-3-2)
\Edge[,Direct](l-3-1)(l-2-2)
\Edge[,Direct](l-3-1)(l-4-1)
\Edge[,Direct](l-3-5)(l-2-6)
\Edge[,Direct](l-3-2)(l-3-3)
\Edge[,Direct](l-3-2)(l-4-2)
\Edge[,Direct](l-3-3)(l-3-4)
\Edge[,Direct](l-3-3)(l-2-4)
\Edge[,Direct](l-3-3)(l-4-3)
\Edge[,Direct](l-3-4)(l-3-5)
\Edge[,Direct](l-3-4)(l-4-4)
\Edge[,Direct](l-4-1)(l-3-2)
\Edge[,Direct](l-4-1)(l-5-1)
\Edge[,Direct](l-4-2)(l-3-3)
\Edge[,Direct](l-4-3)(l-3-4)
\Edge[,Direct](l-4-3)(l-5-3)
\Edge[,Direct](l-4-4)(l-3-5)
\Edge[,Direct](l-5-1)(l-5-2)
\Edge[,Direct](l-5-1)(l-4-2)
\Edge[,Direct](l-5-1)(l-6-1)
\Edge[,Direct](l-5-3)(l-4-4)
\Edge[,Direct](l-5-2)(l-5-3)
\Edge[,Direct](l-5-2)(l-6-2)
\Edge[,Direct](l-6-1)(l-5-2)
\Edge[,Direct](l-6-1)(l-7-1)
\Edge[,Direct](l-6-2)(l-5-3)
\Edge[,Direct](l-7-1)(l-6-2)
\end{tikzpicture}

}
\caption{Triangle Tire World problems 1, 2, 3 (left to right).}
\label{fig:tireworld}
\end{figure}
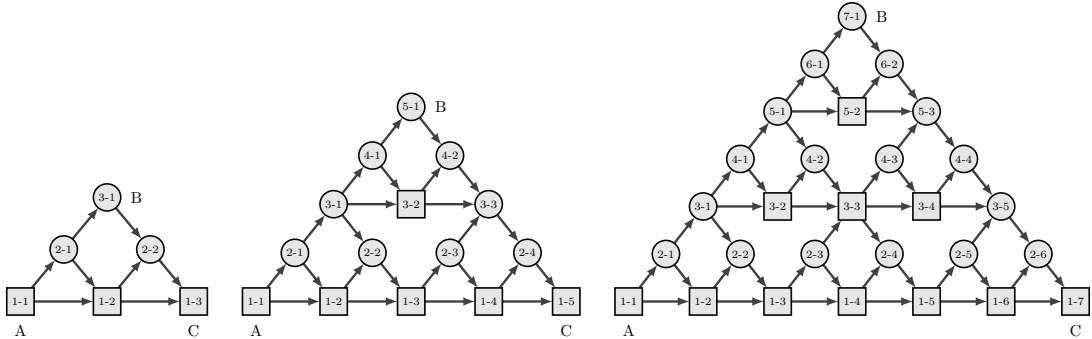

\textbf{Rectangle Tire World (\recttireworld) [IPPC 2008]}

The agent must drive from the bottom-left corner of a rectangular network of roads to the top-right, avoiding dangerous areas as much as possible.
This domain is inspired by \tritireworld~\cite{Little2007:PlannerVsReplanner}, but with two key changes.
First, as the name suggests, the road network is rectangular, rather than triangular, which makes it easier to encode large road systems compactly.
On this network, the agent can drive to orthogonally adjacent cities, as well as diagonally adjacent ones.
Second, the agent can no longer get a flat tyre, instead there are dangerous locations.
If the agent enters a dangerous location, it crashes with certainty, killing the agent, so it can not achieve any goal.
Entire rows and columns of the rectangular grid are specified to be unsafe.
Driving diagonally is ``somewhat dangerous,'' so anytime the agent performs this action it crashes with probability \(0.2\).
The agent should minimise the number of drives between cities, while minimising the probability of crashing.
Unsafe locations may be unavoidable in this domain, so \recttireworld does not satisfy the reachability assumption (\cref{assump:reachability}).
Moreover, randomly generated instances of this problem may be unsolvable in the sense that there is no way for the agent to reach the goal with probability \(> 0\).

\textbf{Zenotravel (\zenotravel) [IPPC 2004, 2006, 2008].}
The agent must route aeroplanes to move people between cities.
When grounded, aeroplanes board and debark passengers as expected.
Aeroplanes can fly at a normal speed or ``zoom'' at a faster speed, allowing the vehicle to reach its destination faster at the cost of more fuel consumption.
Once an aeroplane's fuel is exhausted, it must refuel.
The boarding, debarking, flying to a destination (at any speed), and refuelling actions all succeed only with a specified probability, and upon failure nothing happens.
This means the domain has no dead ends, since any action can be repeated until it succeeds.

\subsection{Presentation of Results}\label{sec:experiments-presentation-of-results}

In the following sections, we will be using similar tables and plots to summarise the results of our experiments:
\begin{itemize}
\item \textbf{Rank tables.}
To obtain these tables, we start by creating a ranking of algorithms' runtimes for each instance, where the fastest algorithm on an instance gets rank \(1\), the second-fastest gets \(2\), and so on.
If multiple algorithms tie with identical runtimes, then these algorithms receive the average rank of the group, e.g., if algorithm \(a\) is strictly fastest, algorithms \(b, c, d\) are tied second, and algorithm \(e\) is strictly slowest, then \(a\) receives rank 1, \(b, c, d\) all receive \(\frac{2+3+4}{3} = 3\), and \(e\) gets \(5\).
If an algorithm times out on an instance, it receives the largest rank for that instance, tied among all algorithms that timed out.
Our rank tables present each algorithm's average (mean) rank over all relevant instances.
Rankings are an intuitive way to determine which algorithms are fastest overall, and abstract away from the concrete times taken.
Different rank tables \textit{are not comparable}.
\item \textbf{Coverage tables.}
This is a standard way to summarise results in the planning community.
For each domain, we present the number of instances that an algorithm and heuristic were able to solve before timing out.
This gives a good indication of how well an algorithm scales, again without worrying about concrete timings.
Different coverage tables \textit{are comparable}.
\item \textbf{Cumulative plots.}
These plots show how many instances an algorithm and heuristic was able to solve with a budget over time, number of \qvalues, or number of calls to the heuristic.
In the cumulative plot over time, every point \(x,y\) on a curve denotes that the relevant algorithm and heuristic solved \(y\)-many instances within \(x\) seconds, and the other plots are similar w.r.t. their respective budgets.
Note that these plots effectively show total coverage, as in the coverage table, but over slices of time (or \qvalues, or heuristic calls).
Different cumulative plots \textit{are comparable}.
\item \textbf{Partial SSP size percentage tables.}
These tables show how large an algorithm and heuristic's final partial SSP is, as a percentage of a baseline's final partial SSP sizes.
Thus, the \(|\partstates|\%\) and \(|\partactions|\%\) give the number of states \(|\partstates|\) and actions \(|\partactions| = \sum_{\s \in \partstates} |\partactions(\s)|\) in the relevant algorithm's final partial SSP, as a percentage of the baseline's counterparts.
We use a percentage \wrt the baseline rather than the entire SSP because we want to emphasise the difference from the baseline, rather than how much smaller the partial SSPs are than the original SSP.\footnote{Our SSPs are so large, that it is often impractical to compute the reachable states to calculate its size --- recall that we use Probabilistic PDDL to represent our problems, and the state-space can be exponentially large \wrt the encoding.}
Note that these tables can only feature \ilao and \cgilao type algorithms, because we have not defined partial SSPs for others.
Different partial SSP size tables \textit{are only comparable if the baseline is the same}.
We consider a single baseline across all these tables, so they are indeed comparable.
\item \textbf{Per-problem tables.}
We present more detailed tables that give the coverage, average runtime, and average \qvalues for each algorithm and heuristic per problem (rather than per domain) at \url{https://github.com/schmlz/cgilao/tree/main/results} and \cite{Schmalz2024:CgilaoSourceCode}.
These are only given externally, but we will occasionally refer to them in this section.
\end{itemize}

\subsection{What are the best versions of \cgilao?}\label{sec:experiments-best-version-of-cgilao}

We investigate which settings of \cgilao perform best, in order to eliminate unsuccessful variants and focus on only successful ones in the following sections.
In particular, we are comparing \cgilao with different expansion settings:
\begin{itemize}
\item \cgilaoBellmanSingle: expand with a single greedy action (\bellmanSingle from \cref{sec:bellman-expansions}).
\item \cgilaoBellmanTied: expand with all tied-greedy actions (\bellmanTied from \cref{sec:bellman-expansions}).
\item \cgilaoTrial: expand with rollout using trial (\cref{sec:expansion-rollout}).
We use \(t_{\max} = 100\), because preliminary results suggested this was a reasonable choice.
\item \cgilaoFFAO: expand with rollout using FF on all-outcomes determinisation (\cref{sec:expansion-rollout}).
\item \cgilaoFFMLO: expand with rollout using FF on most-likely-outcomes determinisation (\cref{sec:expansion-rollout}).
\end{itemize}

\Cref{tab:ranks-cgilao-variants} is a rank table that presents the average ranks and \(95\%\) Confidence Interval (CI) for the variants of \cgilao.
For all heuristics, the ordering of ranks is identical: from 1 to 5, we have \cgilaoBellmanTied, \cgilaoBellmanSingle, \cgilaoTrial, \cgilaoFFAO, \cgilaoFFMLO.
The \(95\%\) CI shows that our experiments were unable to find a statistical difference between the average rank of the top two (\cgilaoBellmanTied and \cgilaoBellmanSingle), but there is a statistically significant difference between the average ranks of the top two and \cgilaoTrial, and between \cgilaoTrial and the variants with FF expansion.
In the coverage table \cref{tab:coverage-cgilao-variants}, we see that \cgilaoBellmanTied, \cgilaoBellmanSingle, and \cgilaoTrial have similar total coverage, and the FF expansions are clustered together with significantly lower total coverage.
Observe that with \roc, \cgilaoBellmanSingle has a slightly higher total coverage than \cgilaoBellmanTied by \data{7} instances, even though it has a lower average rank.
This is not a contradiction, and happens because \cgilaoBellmanTied tends to be faster on the other instances.
Furthermore, the instances that \cgilaoBellmanSingle solves and \cgilaoBellmanTied fails to solve are not solved by the other algorithms either, so \cgilaoBellmanTied's tied-last rank remains small there.
Overall, the coverage supports that \cgilaoBellmanTied and \cgilaoBellmanSingle are the fastest variants.
Finally, we consider the cumulative plot over time in \cref{fig:main-cumulative-cgilao-variants}.
This clearly shows that \cgilaoBellmanTied, \cgilaoBellmanSingle, and \cgilaoTrial outperform the FF expansions with a significant margin over all heuristics.
This performance gap becomes clear for easy problems that take around \data{\(50\)} seconds to solve, and is carried through all larger problems.
Among the top three algorithms (\cgilaoBellmanTied, \cgilaoBellmanSingle, \cgilaoTrial), it is difficult to determine a clear winner, but we observe that \cgilaoTrial's curve is underneath the curves of \cgilaoBellmanTied and \cgilaoBellmanSingle most of the time.
Combining these observations, we declare that overall, \cgilaoBellmanTied and \cgilaoBellmanSingle are the tied-best variants of \cgilao, \cgilaoTrial is a close second, and the FF expansions are last by a significant margin.
We now compare the variants in more detail.

\begin{table}[t!]
\centering
\adjustbox{max width=\linewidth}{

\begin{tabular}{|l S[table-format=1.2(1.2)]|}
\multicolumn{2}{c}{\roc} \\
\cgilaoBellmanTied & 2.12(0.08) \\
\cgilaoBellmanSingle & 2.17(0.07) \\
\cgilaoTrial & 2.68(0.09) \\
\cgilaoFFAO & 3.96(0.08) \\
\cgilaoFFMLO & 4.08(0.08) \\
\end{tabular}

\begin{tabular}{|l S[table-format=1.2(1.2)]|}
\multicolumn{2}{c}{\pdb} \\
\cgilaoBellmanTied & 2.24(0.08) \\
\cgilaoBellmanSingle & 2.39(0.07) \\
\cgilaoTrial & 2.78(0.09) \\
\cgilaoFFAO & 3.75(0.08) \\
\cgilaoFFMLO & 3.85(0.08) \\
\end{tabular}

\begin{tabular}{|l S[table-format=1.2(1.2)]|}
\multicolumn{2}{c}{\lmcut} \\
\cgilaoBellmanTied & 2.25(0.07) \\
\cgilaoBellmanSingle & 2.30(0.07) \\
\cgilaoTrial & 2.63(0.08) \\
\cgilaoFFAO & 3.80(0.07) \\
\cgilaoFFMLO & 4.02(0.08) \\
\end{tabular}

}
\caption{Runtime ranking of \cgilao variants within a specified heuristic (mean and 95\% CI over all instances).}
\label{tab:ranks-cgilao-variants}
\end{table}

\begin{table}[t!]
\centering
\adjustbox{max width=\linewidth}{

\begin{tabular}{llrrrrrrrrrrrrrr}
 & & \rotatebox{90}{\bw} & \rotatebox{90}{\coresec} & \rotatebox{90}{\elevators} & \rotatebox{90}{\exbw} & \rotatebox{90}{\parcn} & \rotatebox{90}{\parcr} & \rotatebox{90}{\random} & \rotatebox{90}{\recttireworld} & \rotatebox{90}{\sar} & \rotatebox{90}{\schedule} & \rotatebox{90}{\sysadmin} & \rotatebox{90}{\tritireworld} & \rotatebox{90}{\zenotravel} & \rotatebox{90}{total} \\
 \cline{1-16}
 & \# of instances & 110 & 35 & 75 & 105 & 30 & 30 & 75 & 70 & 25 & 45 & 20 & 40 & 45 & 705 \\
\cline{1-16}
\multirow[c]{5}{*}{\rotatebox{90}{\roc}} & \cgilaoBellmanTied & \bestCovr{105} & 25 & \bestCovr{75} & \bestCovr{105} & \bestCovr{30} & \bestCovr{25} & 36 & 60 & 20 & \bestCovr{45} & \bestCovr{20} & 35 & \bestCovr{40} & 621 \\
 & \cgilaoBellmanSingle & \bestCovr{105} & 25 & \bestCovr{75} & \bestCovr{105} & \bestCovr{30} & \bestCovr{25} & \bestCovr{43} & 60 & 20 & \bestCovr{45} & \bestCovr{20} & 35 & \bestCovr{40} & \bestCovr{628} \\
 & \cgilaoTrial & \bestCovr{105} & 25 & 70 & \bestCovr{105} & \bestCovr{30} & 24 & 35 & 60 & 21 & \bestCovr{45} & \bestCovr{20} & 35 & \bestCovr{40} & 615 \\
 & \cgilaoFFAO & 100 & 20 & 59 & 95 & \bestCovr{30} & 20 & 35 & 60 & 20 & 35 & 10 & 25 & 30 & 539 \\
 & \cgilaoFFMLO & 85 & 20 & 61 & 95 & \bestCovr{30} & 20 & 40 & 60 & 20 & 35 & 15 & 26 & 26 & 533 \\

\cline{1-16} \multirow[c]{5}{*}{\rotatebox{90}{\pdb}} & \cgilaoBellmanTied & 90 & \bestCovr{30} & \bestCovr{75} & 90 & 0 & 0 & 30 & 60 & \bestCovr{24} & 30 & \bestCovr{20} & \bestCovr{38} & \bestCovr{40} & 527 \\
 & \cgilaoBellmanSingle & 90 & \bestCovr{30} & \bestCovr{75} & 90 & 0 & 0 & 30 & 60 & 20 & 30 & \bestCovr{20} & 37 & \bestCovr{40} & 522 \\
 & \cgilaoTrial & 90 & \bestCovr{30} & 70 & 90 & 0 & 0 & 30 & 60 & \bestCovr{24} & 30 & \bestCovr{20} & 35 & \bestCovr{40} & 519 \\
 & \cgilaoFFAO & 85 & \bestCovr{30} & 61 & 90 & 5 & 0 & 30 & 60 & 20 & 35 & 10 & 21 & 35 & 482 \\
 & \cgilaoFFMLO & 85 & \bestCovr{30} & 61 & 90 & 5 & 0 & 30 & 60 & 20 & 35 & 15 & 22 & 30 & 483 \\
\cline{1-16}

\cline{1-16} \multirow[c]{5}{*}{\rotatebox{90}{\lmcut}} & \cgilaoBellmanTied & 45 & 25 & 70 & 101 & \bestCovr{30} & 20 & 20 & \bestCovr{65} & \bestCovr{24} & \bestCovr{45} & \bestCovr{20} & 30 & 25 & 520 \\
 & \cgilaoBellmanSingle & 45 & 25 & 71 & 100 & \bestCovr{30} & 20 & 20 & \bestCovr{65} & 20 & \bestCovr{45} & \bestCovr{20} & 30 & 25 & 516 \\
 & \cgilaoTrial & 45 & 25 & 70 & 100 & \bestCovr{30} & 20 & 20 & \bestCovr{65} & 22 & \bestCovr{45} & \bestCovr{20} & 30 & 25 & 517 \\
 & \cgilaoFFAO & 45 & 20 & 63 & 95 & 25 & 15 & 25 & \bestCovr{65} & 20 & 35 & 10 & 20 & 20 & 458 \\
 & \cgilaoFFMLO & 45 & 20 & 62 & 95 & 22 & 15 & 25 & \bestCovr{65} & 20 & 35 & 15 & 20 & 15 & 454 \\
\cline{1-16}

\end{tabular}

}
\caption{Coverage for each \cgilao variant and considered heuristic over the benchmark domains.
The highest coverage for each problem is marked with boldface.
}\label{tab:coverage-cgilao-variants}
\end{table}

\begin{figure}[t!]
\raggedright
\includegraphics[scale=0.5, valign=t]{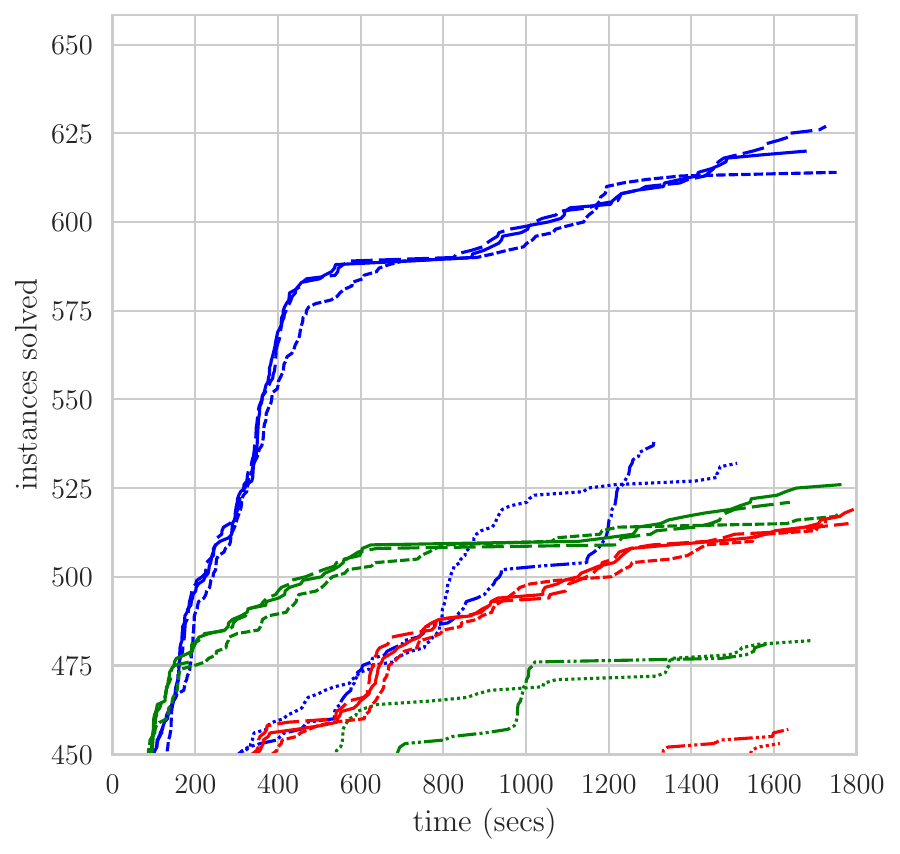}
\includegraphics[scale=0.5, valign=t]{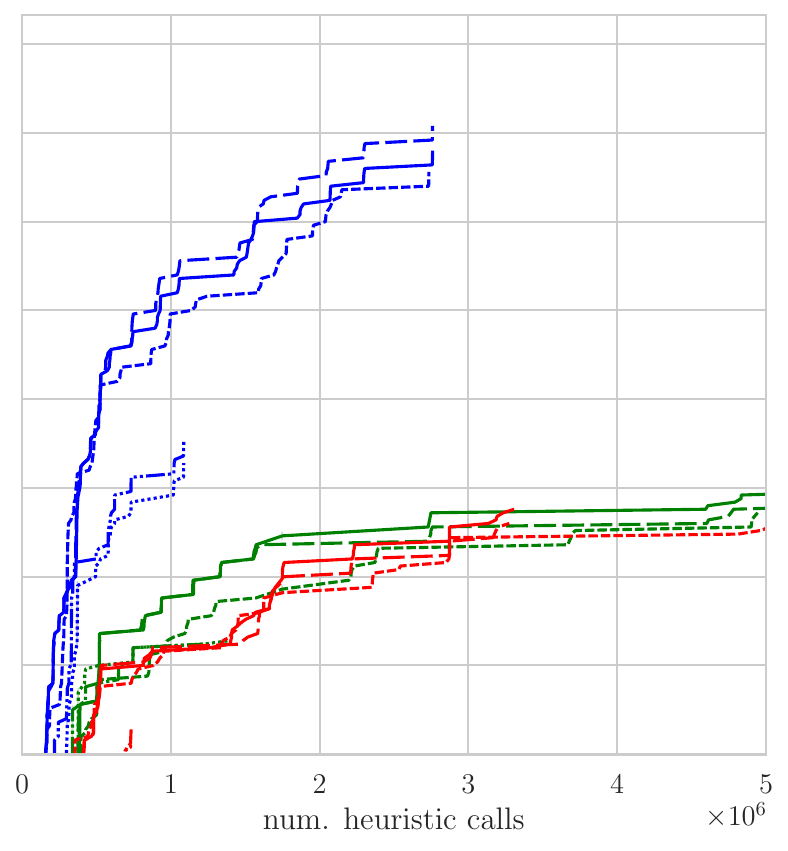}
\includegraphics[scale=0.5, valign=t]{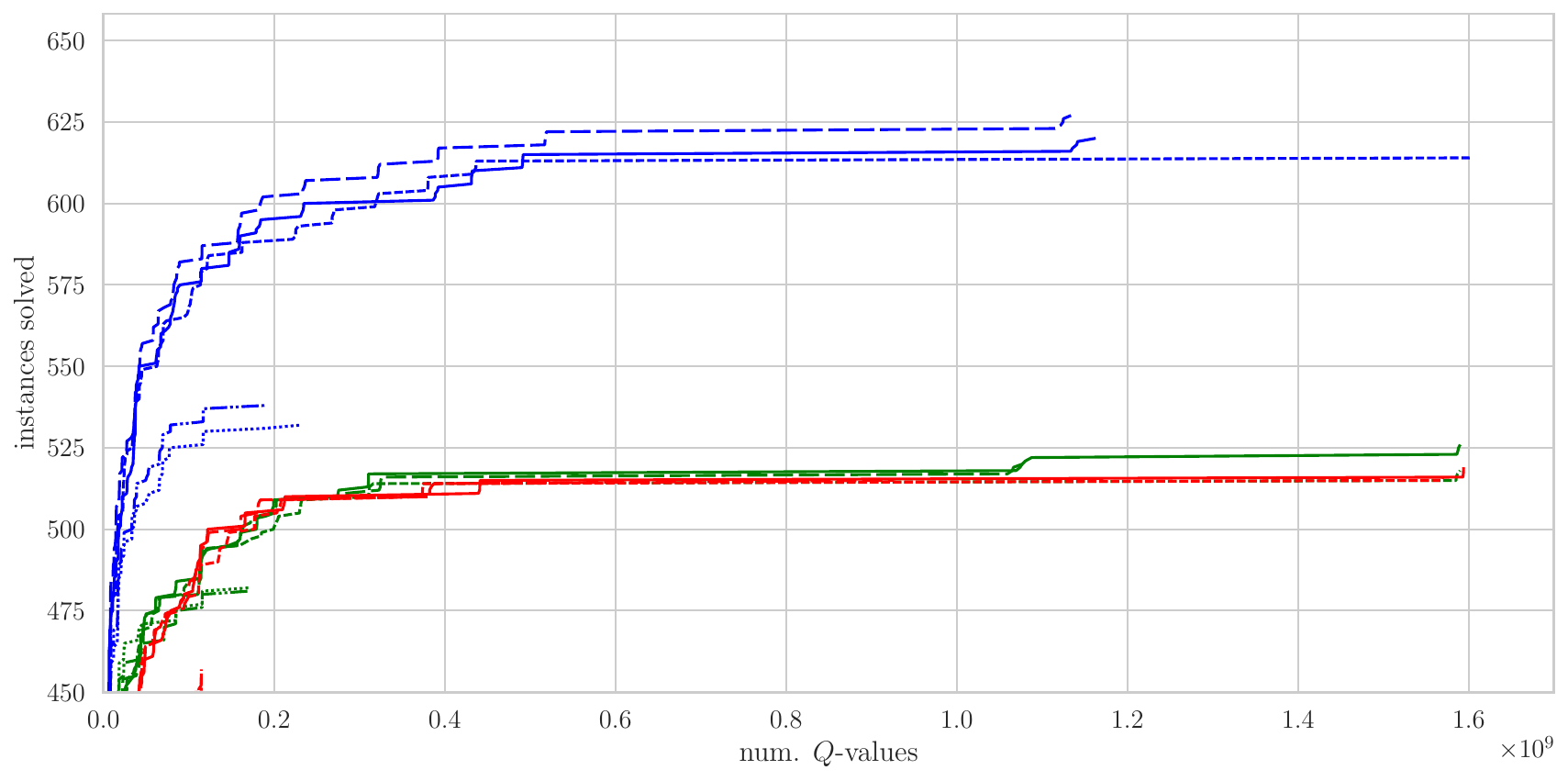}
\begin{center}
\includegraphics[scale=0.5, valign=t]{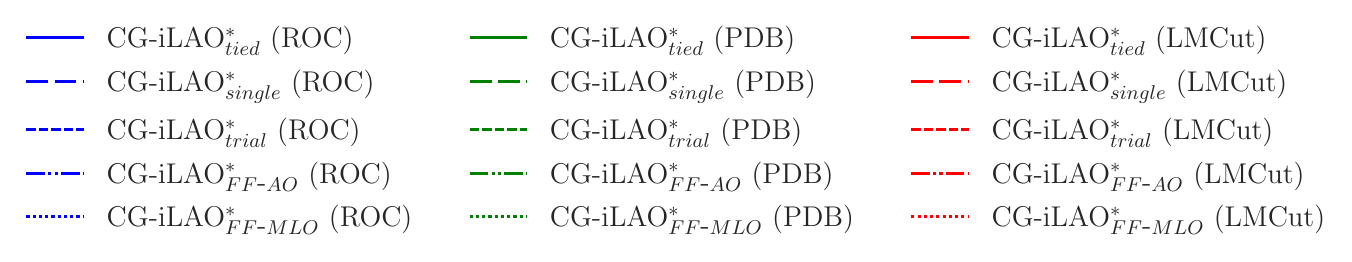}
\end{center}
\caption{
For each algorithm and heuristic, the cumulative plot of how many problems and seeds were solved w.r.t. time in seconds (top left), number of \qvalues (top right), and number of calls to heuristic (bottom).
To focus on where the algorithms are different we start the \(y\)-axis at \data{450} instances, and end the \(x\)-axis at \data{\(5 \times 10^6\)} heuristic calls.
}
\label{fig:main-cumulative-cgilao-variants}
\end{figure}

We have seen that \cgilaoBellmanTied and \cgilaoBellmanSingle are the best-performing variants of \cgilao, with no significant difference between them, overall.
However, we highlight that their behaviours are domain dependent, i.e., they have different strengths and weaknesses, and are not simply duplicates of each other.
In particular, we focus on \random with \roc, where \cgilaoBellmanSingle has a higher coverage than \cgilaoBellmanTied by \data{7} instances.
\random is an outlier among our domains, because it is the only one that is not inspired by real-world problems, and it has many more applicable actions per state than our other domains, on the order of \(10^4\) where the others have \(10^1\) or \(10^2\) (as we will see in \cref{tab:partial-ssp-sizes-some-problems}).
The key difference is that \cgilaoBellmanSingle has significantly fewer actions in its final partial SSP, by up to two orders of magnitude (we will see this later in \cref{sec:experiments-partial-ssp-sizes}).
Interestingly, this does not translate to savings in \qvalue computations nor heuristic calls, but rather, the difference in performance comes from the overhead of managing the partial SSP's data-structures.
Clearly, \cgilaoBellmanTied is adding many tied-greedy actions to its partial SSP that are not necessary for finding an optimal policy, and \cgilaoBellmanSingle avoids the issue in this case by adding only one greedy action in its expansion.
\cgilaoBellmanSingle also produces smaller partial SSPs on \random with the other heuristics, but neither of the algorithms are able to solve the large problems with these heuristics, so the effect of differently sized partial SSPs is not noticeable.
Over other domains, \cgilaoBellmanSingle generally produces smaller partial SSPs than \cgilaoBellmanTied (we investigate this further in \cref{sec:experiments-partial-ssp-sizes}), but the size difference is not as extreme, and does not result in a significant performance difference.
On the other hand, \cgilaoBellmanTied has a higher coverage on \sar with \lmcut and \pdb.
As we will see in \cref{sec:experiments-state-of-the-art}, these are problems where \ilao performs well, i.e., adding all applicable actions pays off.
In these problems, the partial SSP sizes are similar, but \cgilaoBellmanTied has fewer \qvalues, presumably because it avoids the overhead of inserting missing tied-greedy actions.
Outside these configurations, there is no significant difference between the two.
Overall, this comparison shows that \cgilaoBellmanTied and \cgilaoBellmanSingle, while similar, have different strengths and weaknesses that make them more or less suitable to specific problems.

Now, we investigate \cgilaoTrial.
Using the cumulative plots (\cref{fig:main-cumulative-cgilao-variants}), it is difficult to distinguish \cgilaoTrial from \cgilaoBellmanTied and \cgilaoBellmanSingle in terms of time and \qvalues, but we can clearly see that it is less performant in terms of heuristic calls from \data{\(0.53 \times 10^6\)} calls for \roc and \pdb, and for \lmcut it is again difficult to distinguish.
The similarity in \qvalues makes it clear that \cgilaoTrial's trials fail to add all useful actions in a way that saves \qvalues, but it is also not adding a significant number of unnecessary actions that would inflate the number of \qvalues.
The larger number of heuristic calls tells us that \cgilaoTrial considers more states than \cgilaoBellmanTied and \cgilaoBellmanSingle, which makes sense because its trials deliberately probe states beyond the policy envelope.
This is arguably a weakness because \cgilaoTrial needs to perform more heuristic computations without savings in \qvalues.
However, with the less informative \lmcut heuristic, this weakness disappears and \cgilaoTrial has a similar number of heuristic calls to the others.
This is because \cgilaoBellmanTied and \cgilaoBellmanSingle need to eventually expand the same states to prove optimality.
In this case, \cgilaoTrial is more competitive, and even has a slightly higher total coverage than \cgilaoBellmanSingle by \data{1} instance.
This suggests that the trials do not pay off with informative heuristics: it is more efficient to expand greedily; whereas with less informative heuristics, it can be beneficial to run a trial before deeming an action useful.

Overall, the FF expansions performed poorly over our benchmarks, which can be attributed to two factors: the FF expansion often adds unneeded states and actions, and calling FF for each expansion is expensive.
In the cumulative plots in \cref{fig:main-cumulative-cgilao-variants}, the FF expansions incur significantly more \qvalue computations and heuristic calls, which confirms that more states and actions are being considered than necessary.
Since \qvalue computations and heuristic calls are the main computation expenses of \cgilao, it is clear that the FF expansion variants can not keep up.
This increase in calls happens because FF's plans do not coincide with an optimal policy, which is to be expected, given that it completely ignores probabilistic effects with the determinisations.
These issues are only exacerbated by the additional cost of running FF for each expansion, rather than a single Bellman backup.
The only case where FF expansions pay off in our experiments is in \parcn with \pdb.
This is because the \pdb heuristic is very uninformative on the \parcn domain, and FF happens to expand states and actions that are sufficiently close to the optimal policy.

\subsection{Comparing \cgilao to the State-of-the-art}\label{sec:experiments-state-of-the-art}

We compare \cgilao with the state-of-the-art optimal heuristic-search algorithms \ilao~\cite{Hansen2001:ilao} and \lrtdp~\cite{Bonet2003:lrtdp}.
We only consider \cgilaoBellmanTied and \cgilaoBellmanSingle, the best-performing variants of \cgilao according to \cref{sec:experiments-best-version-of-cgilao}.
We consider two implementations of \ilao in this experiment: \ilao as implemented in the original paper~\cite{Schmalz2024:cgilao}, and \cgilaoBellmanComplete, which is \cgilao with the complete Bellman expansion (see \cref{sec:bellman-expansions}).
Although \cgilaoBellmanComplete implements \ilao, there are some subtle differences in implementation (discussed in \cref{sec:ilao} and \cref{sec:cgilao}) with large impacts on performance.

Using \cref{tab:ranks-cgilao-vs-state-of-the-art}, we can see that the ordering of the first three average ranks is consistent over all heuristics, from 1 to 3: \cgilaoBellmanTied, \cgilaoBellmanSingle, \cgilaoBellmanComplete.
The remaining orderings are not consistent across heuristics: \lrtdp and \ilao alternate.
The \(95 \%\) CI reveals that our experiments can not distinguish between \cgilaoBellmanTied and \cgilaoBellmanSingle (consistent with \cref{sec:experiments-best-version-of-cgilao}), nor can we distinguish between \lrtdp and \ilao.
For the other ordering of rankings there is no overlap between \(95 \%\) CIs: \cgilaoBellmanTied and \cgilaoBellmanSingle are ranked better than \cgilaoBellmanComplete, which is itself ranked better than \ilao and \lrtdp.

In terms of total coverage (\cref{tab:coverage-cgilao-vs-state-of-the-art}), \(\cgilaoBellmanSingle(\roc)\) is the leader with \data{628} instances solved, then with \data{621} instances solved, \(\cgilaoBellmanTied(\roc)\) and \(\cgilaoBellmanComplete(\roc)\) are close second.
The next best is \(\ilao(\roc)\) with \data{580} instances solved, \data{41} behind.
We have observed that \(\cgilaoBellmanTied(\roc)\) has a lower rank than \(\cgilaoBellmanComplete(\roc)\), even though they have the same coverage.
This is because for the instances that both algorithms solved, \(\cgilaoBellmanTied(\roc)\) tends to be faster, which can be seen in the cumulative plot over time (\cref{fig:main-cumulative-cgilao-vs-state-of-the-art}): the curve of \(\cgilaoBellmanComplete(\roc)\) lies underneath the curve of \(\cgilaoBellmanTied(\roc)\), i.e., \(\cgilaoBellmanTied(\roc)\) can solve more instances in the same time.
Thus, \(\cgilaoBellmanTied(\roc)\) has better performance than \(\cgilaoBellmanComplete(\roc)\), because it has the same coverage but a lower rank over the solved instances.
By similar argument \(\cgilaoBellmanSingle(\roc)\) also has better performance than \(\cgilaoBellmanComplete(\roc)\), strengthened by the fact that it has slightly higher coverage.
Thus, \(\cgilaoBellmanTied(\roc)\) and \(\cgilaoBellmanSingle(\roc)\) are the best-performing algorithms.

This trend is similar over the other heuristics: \cgilaoBellmanTied and \cgilaoBellmanSingle are ranked first, and have the highest coverage; in this case, with strictly larger coverage than the next-best \cgilaoBellmanComplete, strengthening the arguments from before.
The cumulative plot over time (\cref{fig:main-cumulative-cgilao-vs-state-of-the-art}) confirms this, because the curves for \cgilaoBellmanTied and \cgilaoBellmanSingle lie above the other algorithms' curves within the same heuristic.
Thus, we declare \cgilaoBellmanTied and \cgilaoBellmanSingle as the best algorithms overall.

Now, we analyse the different performances in more detail.
The fundamental operations that take up most of an algorithm's runtime are \qvalues computations and calls to the heuristic.
There are other operations that may take significant runtime, but they can be considered overhead, e.g., tracking the current partial SSP and ensuring the applicable actions are up-to-date.
It is important to recognise that timings are dependent on many factors, e.g., how optimised the implementation is for the hardware in use, whereas the number of \qvalues and heuristic calls is a more robust measure.
In that sense, the following discussions are more insightful than discussing the fastest algorithm.

\begin{table}[t!]
\centering
\adjustbox{max width=\linewidth}{

\begin{tabular}{|l S[table-format=1.2(1.2)]|}
\multicolumn{2}{c}{\roc} \\
\cgilaoBellmanTied & 2.13(0.08) \\
\cgilaoBellmanSingle & 2.23(0.08) \\
\cgilaoBellmanComplete & 2.96(0.08) \\
\ilao & 3.84(0.09) \\
\lrtdp & 3.84(0.11) \\
\end{tabular}

\begin{tabular}{|l S[table-format=1.2(1.2)]|}
\multicolumn{2}{c}{\pdb} \\
\cgilaoBellmanTied & 2.27(0.08) \\
\cgilaoBellmanSingle & 2.43(0.08) \\
\cgilaoBellmanComplete & 2.86(0.07) \\
\lrtdp & 3.63(0.10) \\
\ilao & 3.82(0.09) \\
\end{tabular}

\begin{tabular}{|l S[table-format=1.2(1.2)]|}
\multicolumn{2}{c}{\lmcut} \\
\cgilaoBellmanTied & 2.32(0.07) \\
\cgilaoBellmanSingle & 2.40(0.08) \\
\cgilaoBellmanComplete & 2.78(0.07) \\
\ilao & 3.68(0.09) \\
\lrtdp & 3.81(0.09) \\
\end{tabular}

}
\caption{Runtime ranking of \cgilao and state-of-the-art within a specified heuristic (mean and 95\% CI over all instances).}
\label{tab:ranks-cgilao-vs-state-of-the-art}
\end{table}

\begin{table}[t!]
\centering
\adjustbox{max width=\linewidth}{

\begin{tabular}{llrrrrrrrrrrrrrr}
 & & \rotatebox{90}{\bw} & \rotatebox{90}{\coresec} & \rotatebox{90}{\elevators} & \rotatebox{90}{\exbw} & \rotatebox{90}{\parcn} & \rotatebox{90}{\parcr} & \rotatebox{90}{\random} & \rotatebox{90}{\recttireworld} & \rotatebox{90}{\sar} & \rotatebox{90}{\schedule} & \rotatebox{90}{\sysadmin} & \rotatebox{90}{\tritireworld} & \rotatebox{90}{\zenotravel} & \rotatebox{90}{total} \\
\cline{1-16}
 & \# of instances & 110 & 35 & 75 & 105 & 30 & 30 & 75 & 70 & 25 & 45 & 20 & 40 & 45 & 705 \\
\cline{1-16}
\multirow[c]{5}{*}{\rotatebox{90}{\roc}} & \lrtdp & 55 & 25 & \bestCovr{75} & \bestCovr{105} & \bestCovr{30} & 20 & 33 & 55 & \bestCovr{25} & \bestCovr{45} & \bestCovr{20} & 35 & 20 & 543 \\
 & \ilao & \bestCovr{105} & 25 & 69 & 100 & \bestCovr{30} & 25 & 38 & 48 & 20 & 35 & \bestCovr{20} & 30 & 35 & 580 \\
 & \cgilaoBellmanComplete & \bestCovr{105} & 25 & 70 & \bestCovr{105} & \bestCovr{30} & \bestCovr{27} & 35 & 60 & \bestCovr{25} & \bestCovr{45} & \bestCovr{20} & 35 & 39 & 621 \\
 & \cgilaoBellmanTied & \bestCovr{105} & 25 & \bestCovr{75} & \bestCovr{105} & \bestCovr{30} & 25 & 36 & 60 & 20 & \bestCovr{45} & \bestCovr{20} & 35 & \bestCovr{40} & 621 \\
 & \cgilaoBellmanSingle & \bestCovr{105} & 25 & \bestCovr{75} & \bestCovr{105} & \bestCovr{30} & 25 & \bestCovr{43} & 60 & 20 & \bestCovr{45} & \bestCovr{20} & 35 & \bestCovr{40} & \bestCovr{628} \\

\cline{1-16} \multirow[c]{5}{*}{\rotatebox{90}{\pdb}} & \lrtdp & 84 & 30 & \bestCovr{75} & 90 & 0 & 0 & 30 & 60 & \bestCovr{25} & 35 & \bestCovr{20} & 35 & 30 & 514 \\
 & \ilao & 85 & \bestCovr{35} & 70 & 85 & 0 & 0 & 30 & 57 & 20 & 30 & \bestCovr{20} & 30 & 35 & 497 \\
 & \cgilaoBellmanComplete & 85 & 30 & 70 & 90 & 0 & 0 & 30 & 60 & \bestCovr{25} & 30 & \bestCovr{20} & 37 & \bestCovr{40} & 517 \\
 & \cgilaoBellmanTied & 90 & 30 & \bestCovr{75} & 90 & 0 & 0 & 30 & 60 & 24 & 30 & \bestCovr{20} & \bestCovr{38} & \bestCovr{40} & 527 \\
 & \cgilaoBellmanSingle & 90 & 30 & \bestCovr{75} & 90 & 0 & 0 & 30 & 60 & 20 & 30 & \bestCovr{20} & 37 & \bestCovr{40} & 522 \\
\cline{1-16}

\cline{1-16} \multirow[c]{5}{*}{\rotatebox{90}{\lmcut}} & \lrtdp & 45 & 20 & \bestCovr{75} & 100 & \bestCovr{30} & 5 & 20 & 59 & \bestCovr{25} & 40 & \bestCovr{20} & 25 & 15 & 479 \\
 & \ilao & 45 & 25 & 70 & 95 & \bestCovr{30} & 20 & 20 & 55 & 20 & 35 & \bestCovr{20} & 25 & 25 & 485 \\
 & \cgilaoBellmanComplete & 45 & 25 & 70 & 100 & \bestCovr{30} & 20 & 20 & \bestCovr{65} & \bestCovr{25} & 40 & \bestCovr{20} & 30 & 25 & 515 \\
 & \cgilaoBellmanTied & 45 & 25 & 70 & 101 & \bestCovr{30} & 20 & 20 & \bestCovr{65} & 24 & \bestCovr{45} & \bestCovr{20} & 30 & 25 & 520 \\
 & \cgilaoBellmanSingle & 45 & 25 & 71 & 100 & \bestCovr{30} & 20 & 20 & \bestCovr{65} & 20 & \bestCovr{45} & \bestCovr{20} & 30 & 25 & 516 \\
\cline{1-16}

\end{tabular}

}
\caption{Coverage for \cgilao and state-of-the-art with each considered heuristic over the benchmark domains.
The highest coverage for each problem is marked with boldface.}\label{tab:coverage-cgilao-vs-state-of-the-art}
\end{table}

\begin{figure}[t!]
\raggedright
\includegraphics[scale=0.5, valign=t]{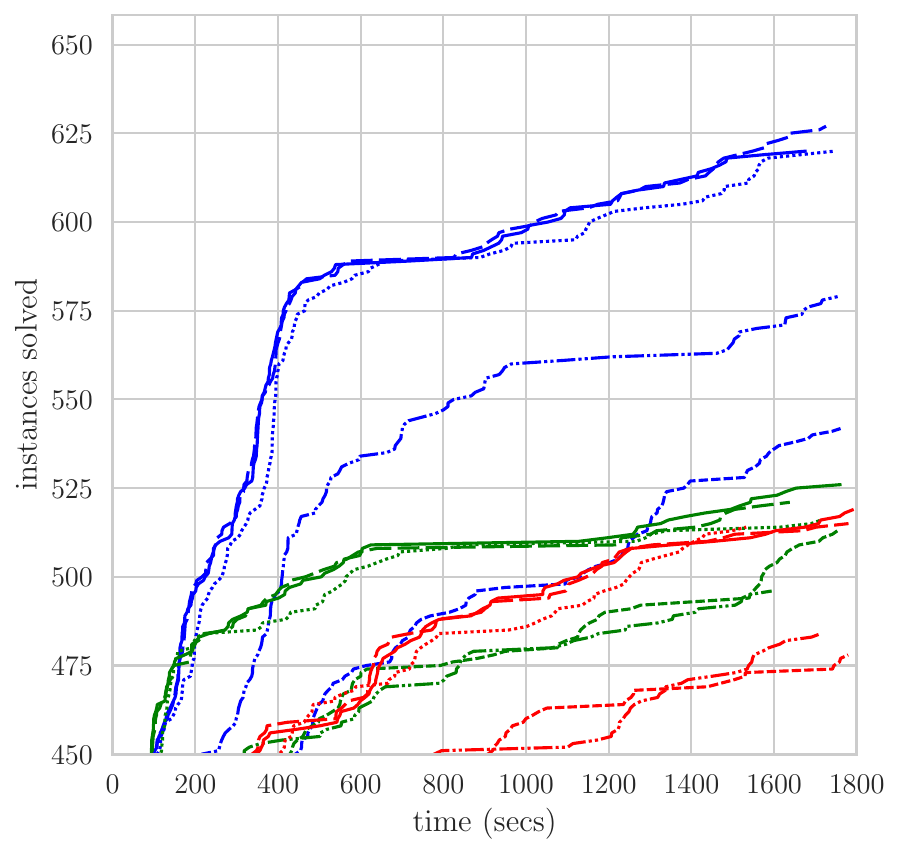}
\includegraphics[scale=0.5, valign=t]{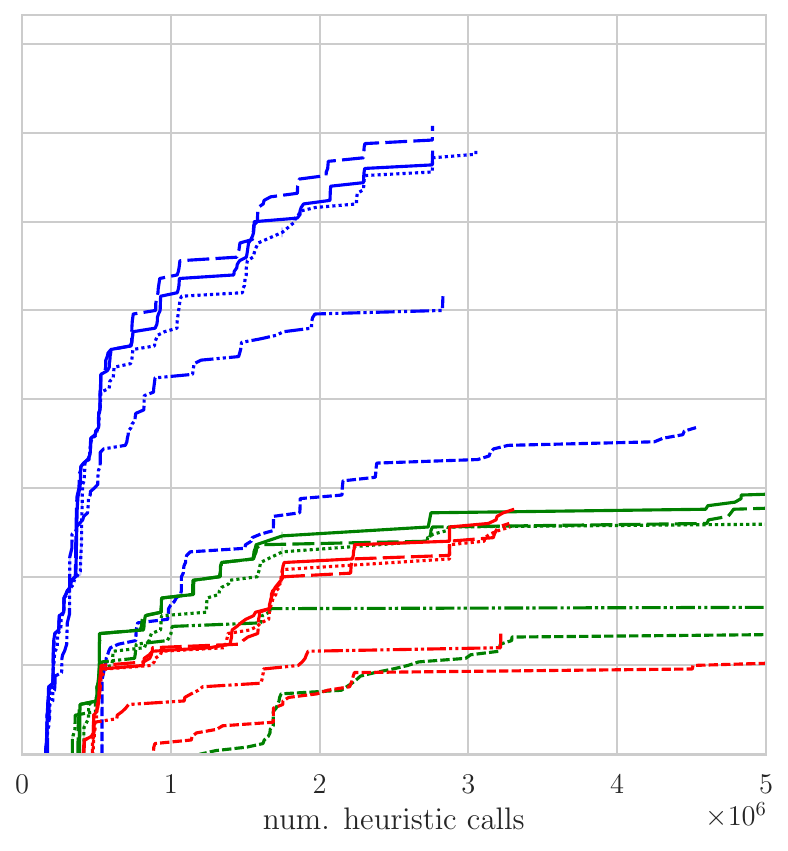}
\includegraphics[scale=0.5, valign=t]{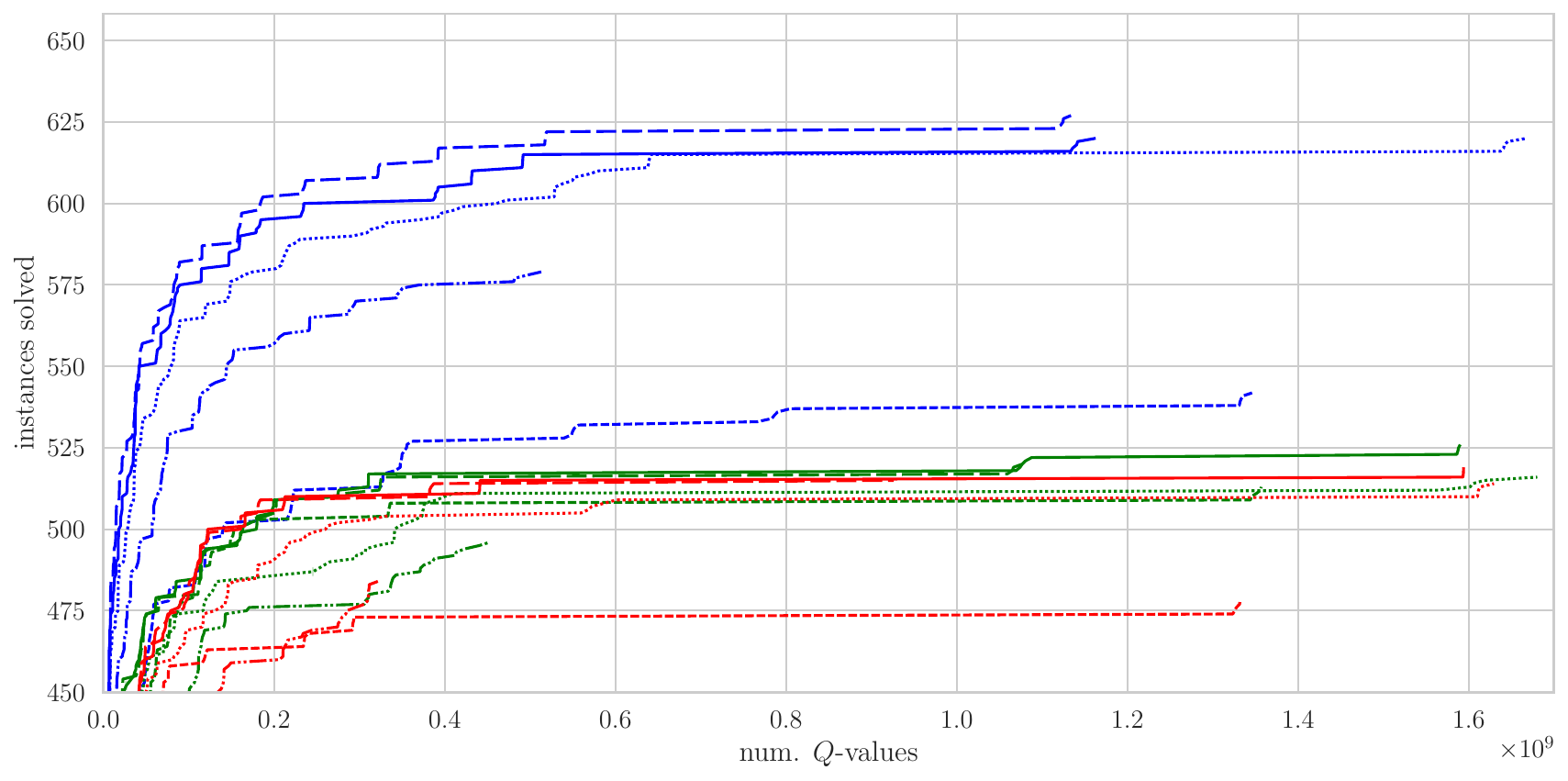}
\begin{center}
\includegraphics[scale=0.5, valign=t]{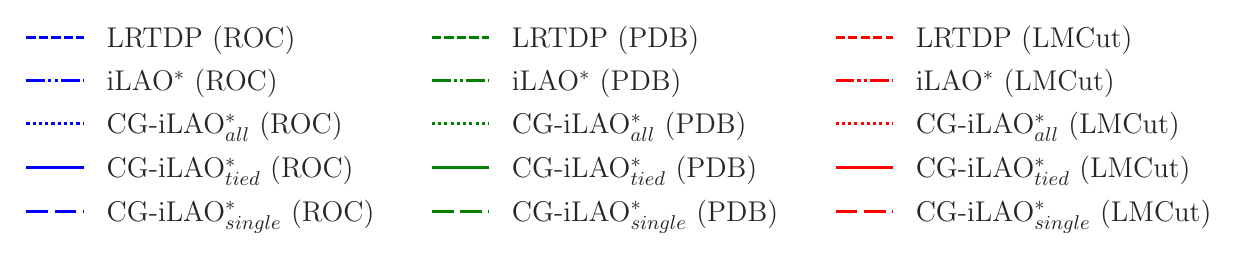}
\end{center}
\caption{
For each algorithm and heuristic, the cumulative plot of how many problems and seeds were solved w.r.t. time in seconds (top left), number of \qvalues (top right), and number of calls to heuristic (bottom).
To focus on where the algorithms are different we start the \(y\)-axis at \data{450} instances, and end the \(x\)-axis at \data{\(5 \times 10^6\)} heuristic calls.
}
\label{fig:main-cumulative-cgilao-vs-state-of-the-art}
\end{figure}

We start by looking at the number of \qvalues.
Focusing on \roc, we see in the cumulative plot over \qvalues in \cref{fig:main-cumulative-cgilao-vs-state-of-the-art} that \(\cgilaoBellmanSingle(\roc)\) and \(\cgilaoBellmanTied(\roc)\) require the fewest \qvalues to solve their instances.
Between \(\cgilaoBellmanSingle(\roc)\) and the third-best \(\cgilaoBellmanComplete(\roc)\), the lead is clear almost immediately, and it maintains a large gap over all problems.
From \data{\(0.04 \times 10^9\)} to \data{\(0.64 \times 10^9\)} \qvalues, \(\cgilaoBellmanTied(\roc)\) leads over \(\cgilaoBellmanComplete(\roc)\) with a gap of around \data{5-15} instances to \(\cgilaoBellmanComplete(\roc)\).
From \data{\(0.64 \times 10^9\)} to \data{\(1.13 \times 10^9\)} \qvalues \(\cgilaoBellmanComplete(\roc)\) catches up to \(\cgilaoBellmanTied(\roc)\), and then \(\cgilaoBellmanTied(\roc)\) makes a jump of \data{4} instances; this is happening because two problems (\elevators 04-13 and \sar 05) require significantly more \qvalues to solve than the rest.
In fact, this emphasises how many fewer \qvalues \(\cgilaoBellmanTied(\roc)\) requires, because it solves these large problems with \data{\(1.16 \times 10^9\)} \qvalues, and \(\cgilaoBellmanComplete(\roc)\) requires \data{\(1.67 \times 10^9\)}.
The gap between the \cgilao variants and other algorithms is even more significant.
Thus, \(\cgilaoBellmanTied(\roc)\) computes the fewest \qvalues.
For the other heuristics, we observe a similar trend, but more extreme in favour of \cgilaoBellmanTied and \cgilaoBellmanSingle.
Thus, the top two \cgilao variants are the algorithms that compute the fewest \qvalues.

To investigate the number of heuristic calls, we use the cumulative plot over heuristic calls (\cref{fig:main-cumulative-cgilao-vs-state-of-the-art}).
We see that \cgilaoBellmanTied, \cgilaoBellmanSingle, and \cgilaoBellmanComplete are decisively ahead of \ilao and \lrtdp over all heuristics, in as few calls as \data{\(0.5 \times 10^6\)}.
Between these top three, there is no consistent trend over heuristics.

\subsection{How many actions does \cgilao ignore?}\label{sec:experiments-partial-ssp-sizes}
The key idea of \cgilao is that it ignores inactive actions and only adds actions when they are needed, with the aim of producing a smaller partial SSP.
In this section, we confirm experimentally that indeed \cgilao has smaller partial SSPs than \ilao, and quantify by how much, i.e., how many actions \cgilao ignores.

\begin{table}[t!]
\newcommand{\midruleOffset}{-2mm} %
\centering
\adjustbox{max width=\linewidth}{

\begin{tabular}{ ll S[table-format=3.1(1.1)]S[table-format=3.1(1.1)] S[table-format=3.1(2.1)]S[table-format=3.1(1.1)] S[table-format=4.1(4.1)]S[table-format=4.1(4.1)] }
 & & \multicolumn{2}{c}{\cgilaoBellmanTied} & \multicolumn{2}{c}{\cgilaoBellmanSingle} & \multicolumn{2}{c}{\ilao} \\
\cmidrule(l{-\midruleOffset}r){3-4}
\cmidrule(l{-\midruleOffset}r){5-6}
\cmidrule(l{-\midruleOffset}r){7-8}
 &  & {\(|\partstates|\%\)} & {\(|\partactions|\%\)} & {\(|\partstates|\%\)} & {\(|\partactions|\%\)} & {\(|\partstates|\%\)} & {\(|\partactions|\%\)} \\
\Xhline{1pt}
\multirow[c]{3}{*}{{\bw}} & {\roc} & 91.8(1.3) & 40.8(1.2) & 98.7(3.5) & 33.2(2.8) & 86.8(1.5) & 85.3(1.5) \\
 & {\lmcut} & 97.5(0.6) & 45.5(1.4) & 97.6(0.5) & 42.5(2.0) & 88.3(0.8) & 84.7(1.2) \\
 & {\pdb} & 90.9(1.4) & 39.7(0.9) & 90.2(3.0) & 30.9(1.6) & 85.9(1.4) & 84.1(1.5) \\
\cline{1-8} \multirow[c]{3}{*}{{\coresec}} & {\roc} & 100.2(0.3) & 33.6(1.9) & 100.2(0.4) & 20.5(1.1) & 85.7(0.6) & 142.0(1.8) \\
 & {\lmcut} & 100.0(0.1) & 39.0(1.9) & 99.9(0.1) & 24.9(1.1) & 95.8(0.7) & 193.4(2.2) \\
 & {\pdb} & 99.6(0.5) & 29.0(1.6) & 100.0(0.4) & 22.2(1.6) & 87.1(0.5) & 113.9(1.5) \\
\cline{1-8} \multirow[c]{3}{*}{{\elevators}} & {\roc} & 97.6(1.1) & 60.5(1.6) & 98.1(1.3) & 53.3(2.2) & 94.9(1.1) & 94.9(1.1) \\
 & {\lmcut} & 93.1(3.0) & 45.0(3.1) & 96.0(2.1) & 40.4(3.0) & 90.2(1.4) & 90.3(1.4) \\
 & {\pdb} & 97.0(1.5) & 56.9(1.7) & 98.2(2.0) & 50.1(2.2) & 92.3(1.6) & 92.5(1.6) \\
\cline{1-8} \multirow[c]{3}{*}{{\exbw}} & {\roc} & 95.8(1.4) & 44.8(1.5) & 99.9(3.2) & 39.6(1.7) & 97.9(0.7) & 103.5(1.8) \\
 & {\lmcut} & 93.9(2.0) & 39.2(2.1) & 99.2(3.3) & 37.1(1.8) & 97.8(0.9) & 102.2(2.1) \\
 & {\pdb} & 95.7(1.9) & 53.9(3.4) & 98.7(3.4) & 49.9(3.7) & 97.6(1.1) & 101.7(2.0) \\
\cline{1-8} \multirow[c]{2}{*}{{\parcn}} & {\roc} & 96.6(0.6) & 44.7(0.9) & 97.0(0.6) & 32.7(0.9) & 87.5(1.1) & 92.7(1.1) \\
 & {\lmcut} & 93.5(0.9) & 44.0(1.2) & 94.5(0.8) & 35.2(1.3) & 75.3(1.5) & 80.9(1.8) \\
\cline{1-8} \multirow[c]{2}{*}{{\parcr}} & {\roc} & 99.6(0.2) & 44.1(1.0) & 99.6(0.2) & 43.0(1.0) & 86.9(1.7) & 86.6(1.6) \\
 & {\lmcut} & 101.1(0.6) & 51.9(1.0) & 101.9(0.9) & 50.1(1.2) & 95.7(1.0) & 95.8(0.9) \\
\cline{1-8} \multirow[c]{3}{*}{{\random}} & {\roc} & 89.8(8.6) & 9.1(1.8) & 93.7(11.3) & 0.9(0.3) & 96.3(3.1) & 96.3(3.1) \\
 & {\lmcut} & 100.0(0.0) & 9.9(3.9) & 100.0(0.0) & 1.8(0.7) & 100.0(0.0) & 100.0(0.0) \\
 & {\pdb} & 95.1(5.3) & 17.0(8.1) & 62.5(16.7) & 1.0(0.5) & 2369.7(3093.0) & 3987.8(5296.5) \\
\cline{1-8} \multirow[c]{3}{*}{{\recttireworld}} & {\roc} & 100.0(0.2) & 31.5(0.8) & 100.0(0.2) & 33.1(1.6) & 100.0(0.0) & 126.0(12.9) \\
 & {\lmcut} & 99.9(0.2) & 30.9(0.9) & 100.0(0.1) & 32.8(1.8) & 99.7(0.3) & 125.8(12.9) \\
 & {\pdb} & 97.9(1.0) & 31.0(1.5) & 98.8(1.1) & 31.5(1.8) & 100.0(1.4) & 127.0(11.3) \\
\cline{1-8} \multirow[c]{3}{*}{{\sar}} & {\roc} & 100.0(0.1) & 91.6(0.8) & 100.2(0.1) & 91.4(0.7) & 101.0(0.2) & 104.5(0.4) \\
 & {\lmcut} & 100.1(0.1) & 92.0(0.8) & 100.0(0.1) & 91.3(0.7) & 100.9(0.2) & 104.5(0.4) \\
 & {\pdb} & 100.3(0.1) & 92.2(0.8) & 100.5(0.2) & 91.4(0.7) & 101.2(0.2) & 104.2(0.3) \\
\cline{1-8} \multirow[c]{3}{*}{{\schedule}} & {\roc} & 100.7(0.3) & 78.1(0.7) & 100.9(0.5) & 50.7(1.1) & 99.0(0.8) & 100.0(0.1) \\
 & {\lmcut} & 99.0(0.5) & 73.8(3.2) & 101.4(0.8) & 50.2(1.2) & 99.3(0.6) & 100.2(0.1) \\
 & {\pdb} & 100.0(0.0) & 79.6(0.5) & 100.0(0.0) & 51.1(1.3) & 100.0(0.0) & 100.0(0.0) \\
\cline{1-8} \multirow[c]{3}{*}{{\sysadmin}} & {\roc} & 100.0(0.1) & 99.5(0.3) & 100.0(0.0) & 97.4(0.3) & 100.0(0.1) & 100.0(0.0) \\
 & {\lmcut} & 100.0(0.0) & 99.5(0.4) & 100.0(0.0) & 97.3(0.4) & 100.1(0.1) & 100.1(0.0) \\
 & {\pdb} & 100.0(0.0) & 100.0(0.0) & 100.0(0.0) & 98.4(0.4) & 100.0(0.0) & 100.0(0.0) \\
\cline{1-8} \multirow[c]{3}{*}{{\tritireworld}} & {\roc} & 98.0(0.8) & 64.4(0.8) & 97.2(1.0) & 63.8(0.9) & 99.7(0.4) & 99.8(0.4) \\
 & {\lmcut} & 100.7(0.6) & 68.6(1.5) & 100.0(0.9) & 67.0(1.5) & 98.0(0.9) & 98.3(0.9) \\
 & {\pdb} & 98.2(0.8) & 65.7(0.7) & 98.0(0.8) & 65.5(0.7) & 99.6(0.4) & 99.8(0.5) \\
\cline{1-8} \multirow[c]{3}{*}{{\zenotravel}} & {\roc} & 88.4(0.5) & 37.1(0.6) & 86.7(0.3) & 30.1(0.8) & 80.5(1.1) & 77.2(1.0) \\
 & {\lmcut} & 89.0(1.2) & 34.3(0.4) & 88.5(1.3) & 31.6(0.5) & 78.0(1.4) & 74.2(1.1) \\
 & {\pdb} & 89.9(1.0) & 44.0(1.3) & 88.0(0.9) & 34.0(1.1) & 82.1(1.6) & 79.1(1.4) \\
\Xhline{1pt}
\end{tabular}

}
\caption{Size of each algorithm's final partial SSP, as a percentage of \cgilaoBellmanComplete's final partial SSP. These values are means over instances where the considered algorithm and \cgilaoBellmanComplete both terminated.}
\label{tab:partial-ssp-sizes-percentages-cgilao}
\undef{\midruleOffset} %
\end{table}

For \cgilaoBellmanTied, \cgilaoBellmanSingle, and \ilao, we present the number of states and actions in their final partial SSPs as a percentage of the number of states and actions in \cgilaoBellmanComplete's final partial SSP in \cref{tab:partial-ssp-sizes-percentages-cgilao}.
The table gives for each domain, the percentages averaged over instances where the algorithm being considered and the baseline both terminate.
Since we require the baseline to terminate, we use \cgilaoBellmanComplete rather than \ilao, because it solves more instances, letting us consider more instances.
In \cref{tab:partial-ssp-sizes-percentages-cgilao} we see that the number of states in the final partial SSPs \(|\partstates|\%\), is around \(100\%\) for all algorithms and heuristics.
This is what we expect, because \cgilao ignores actions, and has no mechanism for ignoring more states than \ilao.
We acknowledge the notable outliers at \random with \pdb, where \cgilaoBellmanSingle has less than \(60\%\) and \ilao has over \(2000\%\) of \cgilaoBellmanComplete's states, but this is a unique exception which does not contribute much.

More importantly, the number of actions in the final partial SSPs \(|\partactions|\%\), is always smaller than \(101\%\) for both \cgilao variants, and below \(50\%\) in most problems.
This confirms that indeed \cgilao is able to leave out many actions that \ilao considers.
Exceptions to this are \sar and \sysadmin, where the \cgilao variants have over \(90\%\) of \cgilaoBellmanComplete's actions.
Note that this is reflected by \cgilaoBellmanComplete having higher coverage on \sar than the \cgilao variants.
It is not entirely clear why \cgilao has such large partial SSPs on these problems, but we conjecture that it is because the heuristics are not informative enough, and with more informative heuristics \cgilao would be able to ignore more actions.
Between \cgilaoBellmanTied and \cgilaoBellmanSingle, we see that \cgilaoBellmanSingle tends to have fewer actions in its final partial SSP.
This is unsurprising, because \cgilaoBellmanTied can add unnecessary actions in its expansion, whereas \cgilaoBellmanSingle only adds actions that are necessary at some point.
Note, that this difference in partial SSP size is not reflected in the number of \qvalues computed, as discussed in \cref{sec:experiments-best-version-of-cgilao}, because \cgilaoBellmanSingle incurs additional overhead in checking which actions should be added.
The exception is again \random where we have seen \(\cgilaoBellmanSingle(\roc)\) obtains highest coverage, precisely because its final partial SSP is significantly smaller.

To make the final partial SSP sizes more concrete, we present the actual final partial SSP sizes for select problems in \cref{tab:partial-ssp-sizes-some-problems}.
This gives a sense of scale of how many actions are being saved, bearing in mind that we picked smaller problems where all algorithms had close-to-complete coverage.

In summary, the \cgilao variants often have significantly smaller final partial SSPs than the \ilao variants, confirming that our mechanism for ignoring actions has a significant impact.

To gain a better insight of how many actions are ignored by \cgilao, we investigate how dense \cgilao's partial SSPs are, i.e., how many actions were added to \cgilao's partial SSP, out of the possible actions that could have been added.
We define the density of state \s as \({|\partactions(\s)|}/{|\A(\s)|}\).
To start, we restrict our attention to a single representative problem from each domain, solved by \cgilaoBellmanTied and \cgilaoBellmanSingle with \roc.
The representative problems were chosen to be the largest problem that was still solved in time (or some large problem that was solved in time, if there is no clear hierarchy).
We give cumulative plots over density for these problems in \cref{fig:density-per-state}, where each point \((x, y)\) represents that a proportion \(x\) of states in the final partial SSP has actions with density \(y\) or less, e.g., \((0.5, 0.3)\) tells us that \(50\%\) of the partial SSP's states have a density of \(0.3\) or less.
Note that the proportion \(x\) of states for \cgilaoBellmanTied and \cgilaoBellmanSingle are evaluated \wrt to their respective final partial SSP sizes, i.e., they may be referring to a different number of states.
However, we have seen in \cref{sec:experiments-partial-ssp-sizes} that the number of states in their partial SSPs is similar, so the curves are roughly comparable.
A plot that grows quickly at the start, with the extreme case \(y=1\) for \(x \geq 0\) (shaped \(\mathop\ulcorner\)), indicates that all states have high density.
Conversely, if the plot grows slowly and then jumps up for large \(x\), with the extreme case \(y=0\) for \(x < 1\) (shaped \(\mathop\lrcorner\)), it means that most states have low density.
Another way to understand this curve is that, if the area under the curve is large then most states are dense, and if the area under the curve is small then most states are not dense.
Variants of \ilao have the extreme \(\mathop\ulcorner\)-shaped curve because \ilao adds all actions, resulting in density \(1\) for all states.

Looking at \cref{fig:density-per-state}, we see that the density profiles vary dramatically across problems.
For problems such as \coresec, \random, and \recttireworld, both \cgilao variants have mostly low-density states, which is reflected by their small final partial SSPs in \cref{tab:partial-ssp-sizes-percentages-cgilao}.
In contrast, \elevators, \sar, and \sysadmin have mostly high-density states, which is respectively reflected by relatively large final partial SSPs in \cref{tab:partial-ssp-sizes-percentages-cgilao}.
This is particularly extreme in \sysadmin, where the \cgilao variants have over \(99\%\) of the number of actions as \cgilaoBellmanComplete.
\cgilaoBellmanTied and \cgilaoBellmanSingle have significantly different profiles for \bw, \coresec, \exbw, \parcn, \random, and \schedule; and nearly identical profiles for the other problems.%
This is generally reflected by the differences in partial SSP sizes for \roc in \cref{tab:partial-ssp-sizes-percentages-cgilao}.
This confirms, again, that \cgilaoBellmanSingle is generally able to ignore more actions than \cgilaoBellmanTied.
Otherwise, we can only conclude that \cgilao equipped with \roc is able to ignore many actions on some problems, and very few on others.

Recall \cref{thm:cgilao-minimal-with-hstar}, which tells us that \cgilaoBellmanSingle with the perfect heuristic adds only the actions necessary to define its optimal policy, i.e., with the perfect heuristic its density plots would have \(y=1\) for \(x \geq 0\).
In that sense, the density of \cgilaoBellmanSingle is heuristic-dependent, rather than problem-dependent.
Consequently, \cref{fig:density-per-state} can be interpreted as a measure of heuristic quality.
This motivates that \cgilao fails to produce small partial SSPs on some problems due to uninformative heuristics, rather than any property of the problems themselves.

To get a higher-level picture we now present a summary of density over entire domains in \cref{fig:density-box-and-whisker}.
Here, we generate the density curves for each algorithm, heuristic, and problem triple (averaged over instances in the problem), and then compute the area under these curves (AUC); we call this value the density AUC.
An instance with mostly dense states and therefore a \(\mathop\ulcorner\)-shaped curve will have AUC close to 1, and conversely an instance with few dense states and a \(\mathop\lrcorner\)-shaped curve will have AUC close to 0.
We aggregate density AUC over domains with box-and-whisker plots to show the distribution over different problems within the domain.
These results are generally consistent with the previous discussion and let us make similar conclusions: the density varies a lot over different domains, and generally \cgilaoBellmanSingle tends to have final partial SSPs with lower density that \cgilaoBellmanTied.
We note that in some domains the densities are very similar across problems within the domain, e.g., \parcn, \parcr, \sar; but in others the densities are varied, e.g., \elevators, \exbw, \schedule.
This does not give any deep insight, only that some of our benchmarks have problems of similar structure within a domain, and others do not.
One must be careful comparing the behaviour over different heuristics, because they are associated with different coverages and therefore have different datasets; nevertheless, it seems that the behaviour does not change significantly over different heuristics.

\begin{landscape}
\begin{table}[t!]
\newcommand{\midruleOffset}{-3mm} %
\setlength\tabcolsep{6px}
\resizebox{\linewidth}{!}{

\begin{tabular}{ l ll S[table-format=7.0(3.0)]S[table-format=7.0(4.0)]S[table-format=7.0(5.0)] S[table-format=7.0(3.0)]S[table-format=7.0(4.0)]S[table-format=7.0(5.0)] S[table-format=7.0(4.0)]S[table-format=7.0(5.0)] S[table-format=7.0(4.0)]S[table-format=7.0(5.0)]}
\multicolumn{2}{c}{} & & \multicolumn{3}{c}{\cgilaoBellmanTied} & \multicolumn{3}{c}{\cgilaoBellmanSingle} & \multicolumn{2}{c}{\ilao} & \multicolumn{2}{c}{\cgilaoBellmanComplete} \\
\cmidrule(l{-\midruleOffset}r){4-6}
\cmidrule(l{-\midruleOffset}r){7-9}
\cmidrule(l{-\midruleOffset}r){10-11}
\cmidrule(l{-\midruleOffset}r){12-13}
\multicolumn{2}{c}{} & & {\(|\partstates|\)} & {\(|\partactions|\)} & {\(|\partactions^{\max}|\)} & {\(|\partstates|\)} & {\(|\partactions|\)} & {\(|\partactions^{\max}|\)} & {\(|\partstates|\)} & {\(|\partactions|\)} & {\(|\partstates|\)} & {\(|\partactions|\)} \\
\Xhline{1pt}
\multirow[c]{4}{*}{\rotatebox{90}{\bw}} & \multirow[c]{2}{*}{04-10} & {\roc} & 38197(187) & 107485(619) & 264000(1337) & 38041(160) & 69964(384) & 263097(1068) & 35601(245) & 249229(1769) & 42269(332) & 310921(2560) \\
 &  & {\pdb} & 38604(308) & 112018(782) & 266356(2200) & 38141(274) & 70045(740) & 263724(1758) & 35514(198) & 248260(1479) & 42229(398) & 310213(3083) \\
\cline{2-13}  & 08-11 & {\roc} & 12961(776) & 42333(2215) & 107413(6639) & 12530(402) & 22106(549) & 104209(3392) & 12172(797) & 102326(6804) & 14114(907) & 116851(7586) \\
\cline{2-13}  & 08-12 & {\roc} & 12961(776) & 42333(2215) & 107413(6639) & 12530(402) & 22106(549) & 104209(3392) & 12172(797) & 102326(6804) & 14114(907) & 116851(7586) \\
\cline{1-13} \cline{2-13} \multirow[c]{3}{*}{\rotatebox{90}{\elevators}} & \multirow[c]{3}{*}{04-15} & {\roc} & 511935(243) & 2236293(8866) & 3063040(1474) & 512158(286) & 2143683(7449) & 3064324(1739) & 513039(171) & 3069532(959) & 515568(988) & 3084272(5788) \\
 &  & {\lmcut} & 441272(717) & 1652518(6622) & 2649696(4449) & 441162(810) & 1551021(8734) & 2648887(4855) & 425933(3021) & 2557899(18032) & 452195(3016) & 2714150(17807) \\
 &  & {\pdb} & 506141(399) & 2134977(8328) & 3028903(2315) & 506320(222) & 2045625(6661) & 3029960(1276) & 504455(808) & 3018725(4709) & 510224(1157) & 3052950(6776) \\
\cline{1-13} \cline{2-13} \multirow[c]{8}{*}{\rotatebox{90}{\random}} & \multirow[c]{3}{*}{04-07} & {\roc} & 8(5) & 40(31) & 707(582) & 8(5) & 12(9) & 707(582) & 8(5) & 707(582) & 8(5) & 707(582) \\
 &  & {\lmcut} & 3(0) & 54(49) & 264(52) & 3(0) & 4(1) & 264(52) & 3(0) & 264(52) & 3(0) & 264(52) \\
 &  & {\pdb} & 4(0) & 16(2) & 265(52) & 4(1) & 5(1) & 265(52) & 4(0) & 265(52) & 4(0) & 265(52) \\
\cline{2-13}  & \multirow[c]{2}{*}{04-09} & {\roc} & 4(1) & 506(33) & 3690(954) & 5(2) & 6(5) & 4731(2508) & 4(1) & 3690(954) & 4(1) & 3690(954) \\
 &  & {\pdb} & 10(2) & 468(50) & 9712(2326) & 10(5) & 15(9) & 9790(4901) & 10(2) & 9448(2300) & 10(2) & 9448(2300) \\
\cline{2-13}  & \multirow[c]{3}{*}{04-13} & {\roc} & 25(5) & 7634(1621) & 46300(9419) & 25(6) & 42(11) & 45070(10729) & 26(4) & 48143(7394) & 26(4) & 48143(7394) \\
 &  & {\lmcut} & 7(0) & 845(3) & 11992(137) & 7(0) & 10(0) & 11992(137) & 7(0) & 12001(42) & 7(0) & 12001(42) \\
 &  & {\pdb} & 128(12) & 37172(4838) & 208000(20635) & 98(23) & 171(41) & 162899(39785) & 197(29) & 325885(48159) & 197(29) & 325905(48161) \\
\cline{1-13} \cline{2-13} \multirow[c]{9}{*}{\rotatebox{90}{\sar}} & \multirow[c]{3}{*}{08-03} & {\roc} & 9910(8) & 35309(32) & 39577(4) & 9921(13) & 35132(30) & 39583(7) & 10022(15) & 39633(7) & 9904(5) & 38081(2) \\
 &  & {\lmcut} & 9914(8) & 35152(12) & 39580(4) & 9913(4) & 35088(27) & 39579(2) & 10024(24) & 39635(12) & 9919(14) & 38088(7) \\
 &  & {\pdb} & 9890(10) & 35316(34) & 39573(5) & 9894(6) & 35153(20) & 39575(3) & 9993(19) & 39625(9) & 9863(2) & 38102(1) \\
\cline{2-13}  & \multirow[c]{3}{*}{08-04} & {\roc} & 113810(17) & 527263(203) & 578498(8) & 113820(47) & 525092(167) & 578502(23) & 114249(76) & 578717(38) & 113769(30) & 560917(15) \\
 &  & {\lmcut} & 113817(22) & 525083(43) & 578501(11) & 113828(13) & 524661(99) & 578506(7) & 114234(78) & 578709(39) & 113799(26) & 560932(13) \\
 &  & {\pdb} & 113763(32) & 527146(311) & 578481(16) & 113766(44) & 524952(158) & 578483(22) & 114178(78) & 578689(39) & 113624(14) & 560916(7) \\
\cline{2-13}  & \multirow[c]{3}{*}{08-05} & {\roc} & {--} & {--} & {--} & {--} & {--} & {--} & {--} & {--} & 1251364(52) & 7485099(26) \\
 &  & {\lmcut} & 1251642(57) & 7078350(557) & 7682167(29) & {--} & {--} & {--} & {--} & {--} & 1251455(60) & 7485144(30) \\
 &  & {\pdb} & 1251444(65) & 7109673(538) & 7682078(33) & {--} & {--} & {--} & {--} & {--} & 1250771(40) & 7484911(20) \\
\cline{1-13} \cline{2-13} \multirow[c]{9}{*}{\rotatebox{90}{\sysadmin}} & \multirow[c]{3}{*}{08-02} & {\roc} & 1541(0) & 1725(0) & 1729(0) & 1541(0) & 1697(4) & 1729(0) & 1539(0) & 1727(0) & 1541(0) & 1729(0) \\
 &  & {\lmcut} & 1580(0) & 1768(0) & 1768(0) & 1580(0) & 1738(2) & 1768(0) & 1584(0) & 1772(0) & 1580(0) & 1768(0) \\
 &  & {\pdb} & 1592(0) & 1780(0) & 1780(0) & 1592(0) & 1760(3) & 1780(0) & 1592(0) & 1780(0) & 1592(0) & 1780(0) \\
\cline{2-13}  & \multirow[c]{3}{*}{08-03} & {\roc} & 4676(0) & 5231(0) & 5231(0) & 4676(0) & 5130(7) & 5231(0) & 4677(0) & 5232(0) & 4676(0) & 5231(0) \\
 &  & {\lmcut} & 4692(0) & 5247(0) & 5247(0) & 4692(0) & 5135(6) & 5247(0) & 4693(1) & 5248(1) & 4692(0) & 5247(0) \\
 &  & {\pdb} & 4699(1) & 5254(1) & 5254(1) & 4699(1) & 5159(8) & 5254(1) & 4703(2) & 5258(2) & 4699(1) & 5254(1) \\
\cline{2-13}  & \multirow[c]{3}{*}{08-04} & {\roc} & 33016(1) & 37041(1) & 37041(1) & 33017(0) & 35730(29) & 37042(0) & 33038(2) & 37063(2) & 33016(1) & 37041(1) \\
 &  & {\lmcut} & 33021(4) & 37046(4) & 37046(4) & 33023(4) & 35680(19) & 37048(4) & 33042(2) & 37067(2) & 33022(3) & 37047(3) \\
 &  & {\pdb} & 33074(4) & 37099(4) & 37099(4) & 33071(6) & 36030(28) & 37096(6) & 33096(1) & 37121(1) & 33074(4) & 37099(4) \\
\Xhline{1pt}
\end{tabular}

}
\caption{The final partial SSP sizes of \cgilao and \ilao.
The number of states in the final partial SSP \(|\partstates|\), the number of actions in the final partial SSP \(|\partactions| = \sum_{\s \in \partstates} |\partactions(\s)|\), and for \cgilao the number of actions available to its partial states \(|\partactions^{\max}| = \sum_{\s \in \partstates} |\A(\s)|\).
We do not count give-up actions from the fixed-penalty transformation.
We restrict the results to the larger problems of each domain.}
\label{tab:partial-ssp-sizes-some-problems}
\end{table}
\undef{\midruleOffset} %
\end{landscape}

\begin{figure*}[t!]
\raggedright
\includegraphics[scale=0.55]{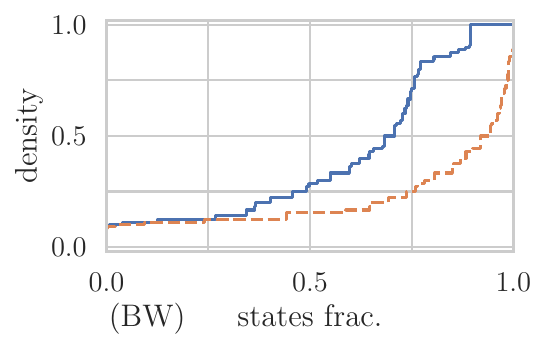}
\includegraphics[scale=0.55]{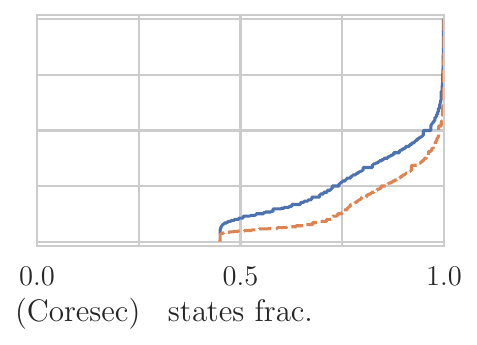}
\includegraphics[scale=0.55]{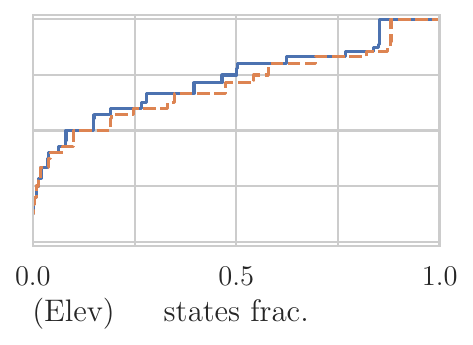}
\\
\includegraphics[scale=0.55]{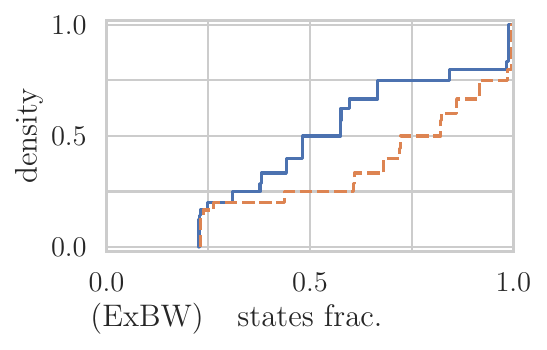}
\includegraphics[scale=0.55]{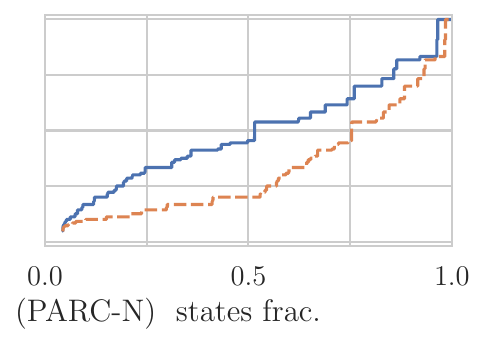}
\includegraphics[scale=0.55]{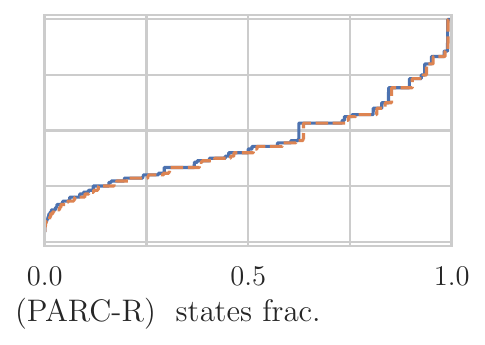}
\\
\includegraphics[scale=0.55]{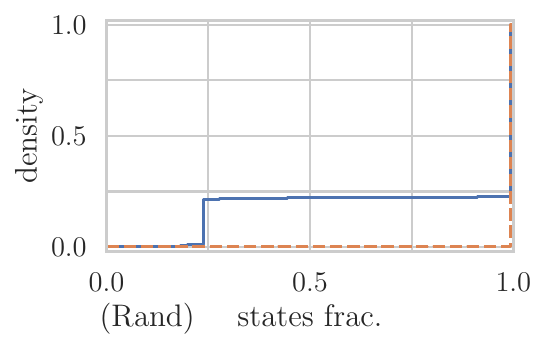}
\includegraphics[scale=0.55]{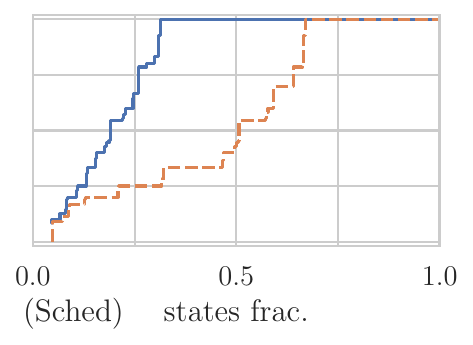}
\includegraphics[scale=0.55]{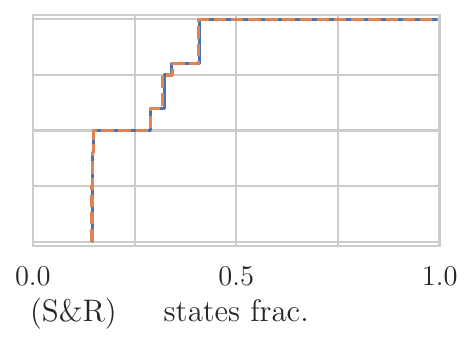}
\\
\includegraphics[scale=0.55]{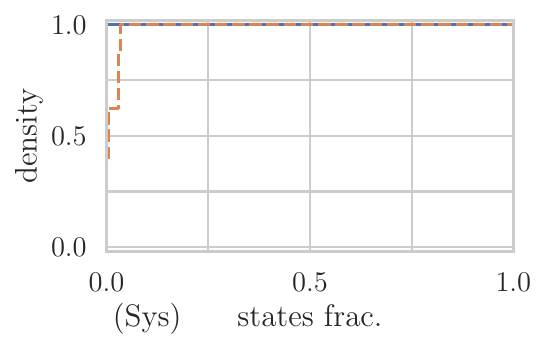}
\includegraphics[scale=0.55]{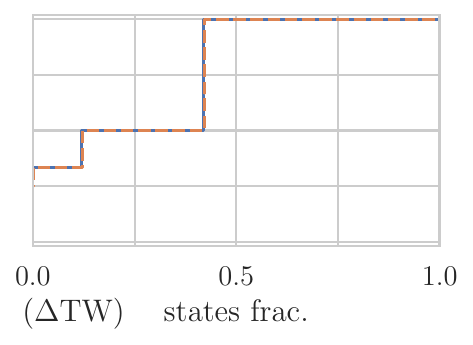}
\includegraphics[scale=0.55]{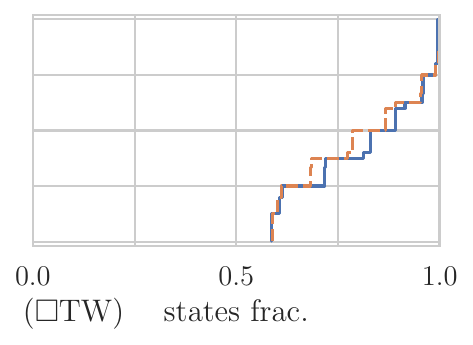}
\\
\includegraphics[scale=0.55]{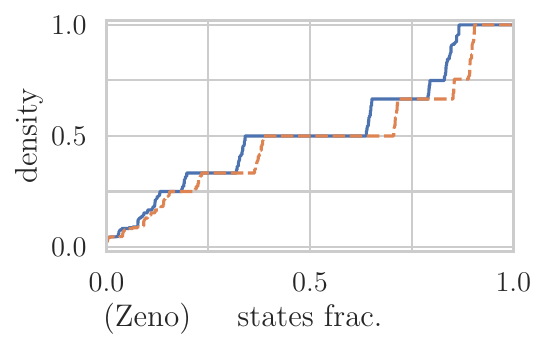}
\\
\begin{center}
\includegraphics[scale=0.55]{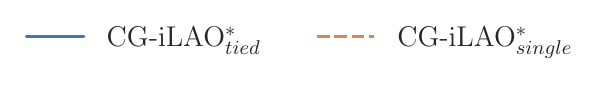}
\end{center}
\caption{Cumulative plot of state density \wrt the fraction of states in the final partial SSP over \data{5} instances for one representative problem per domain.}
\label{fig:density-per-state}
\end{figure*}

\begin{figure*}[t!]
\raggedright
\includegraphics[scale=0.36]{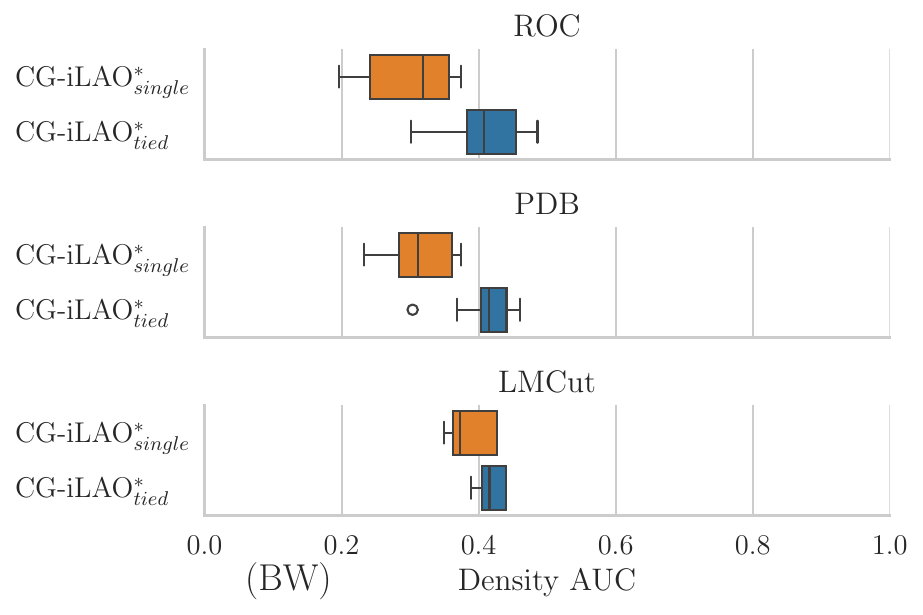}
\includegraphics[scale=0.36]{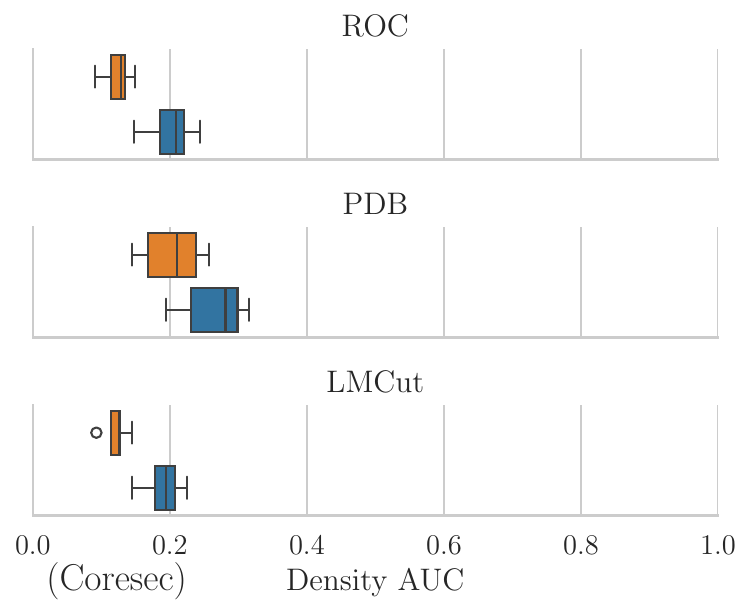}
\includegraphics[scale=0.36]{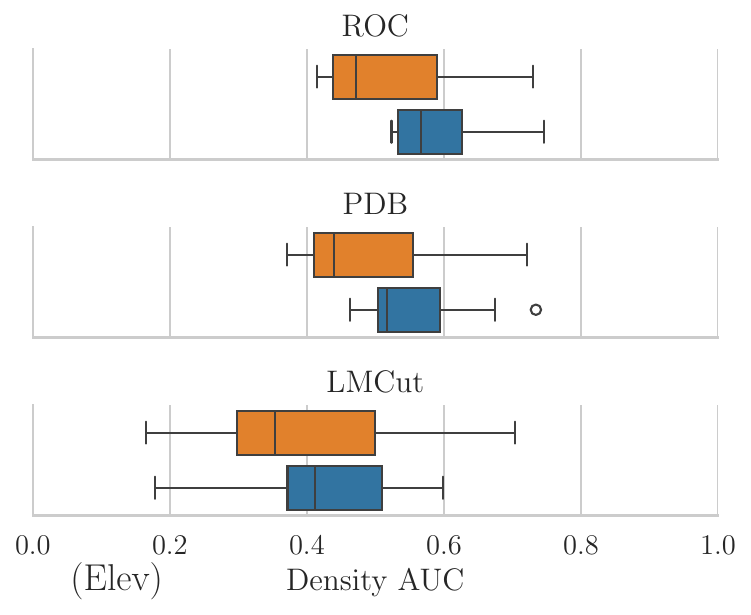}

\vspace{2mm}

\includegraphics[scale=0.36]{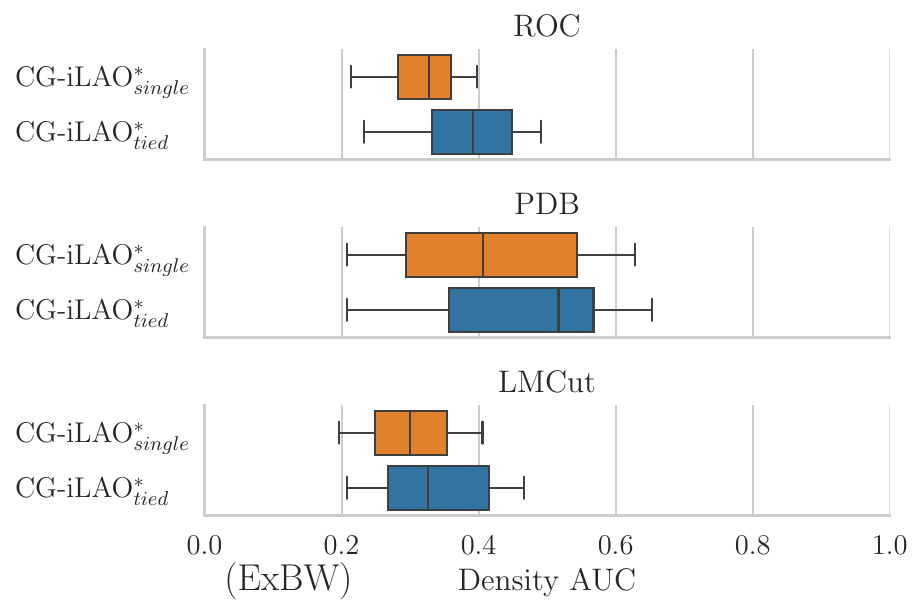}
\includegraphics[scale=0.36]{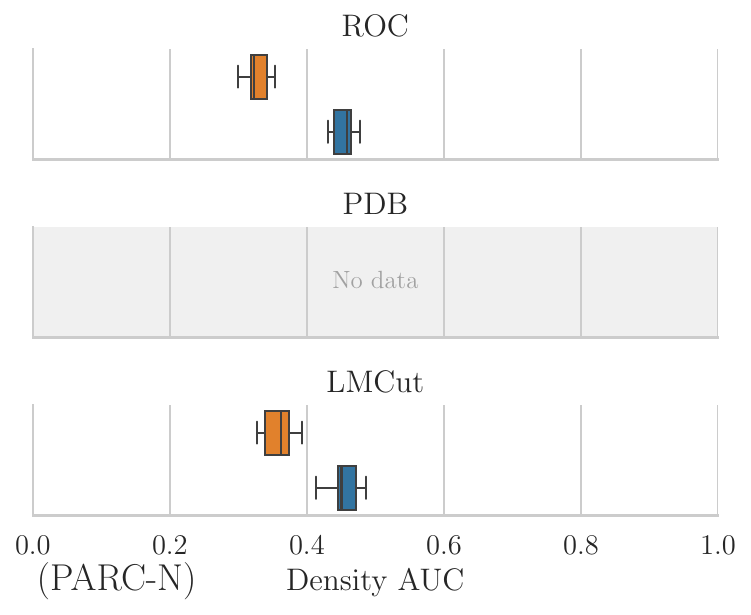}
\includegraphics[scale=0.36]{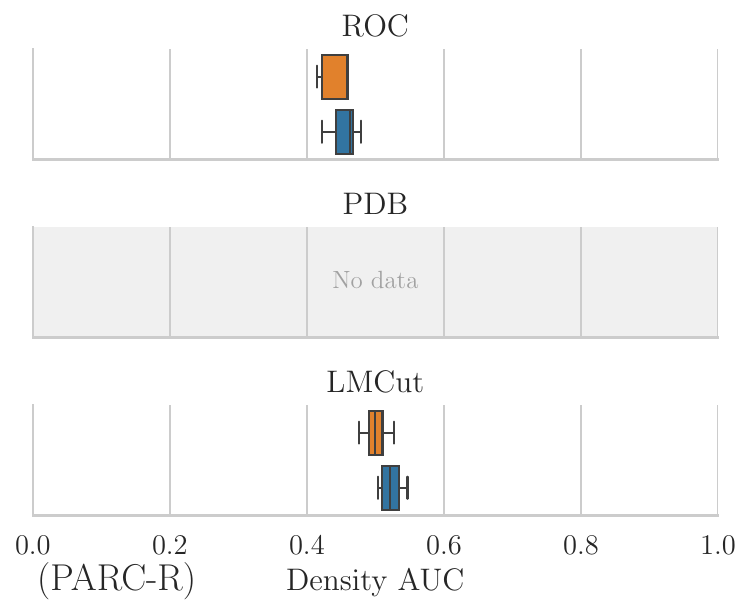}

\vspace{2mm}

\includegraphics[scale=0.36]{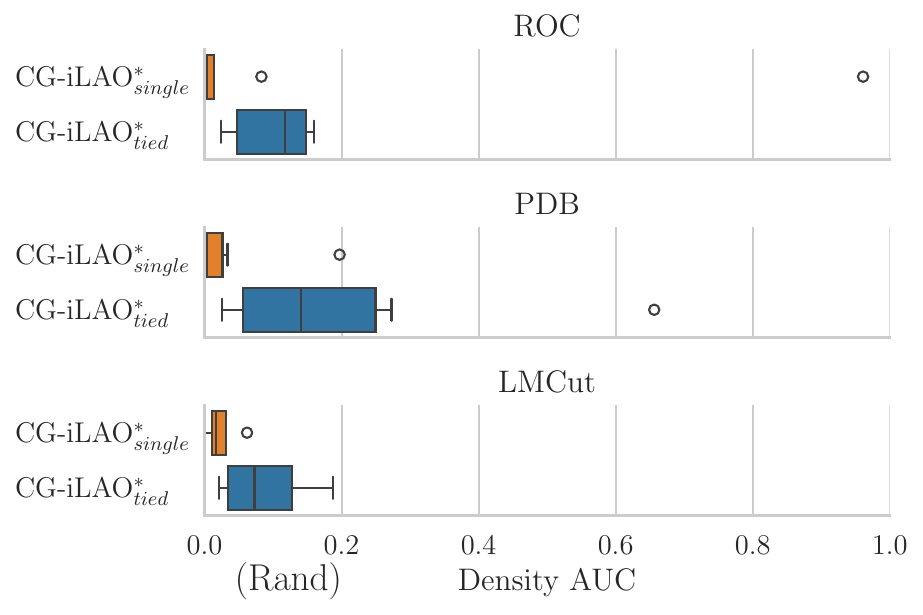}
\includegraphics[scale=0.36]{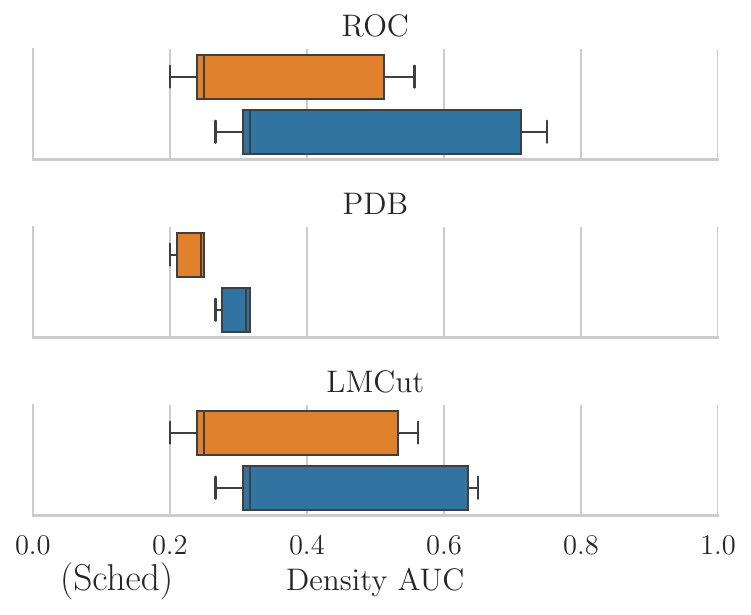}
\includegraphics[scale=0.36]{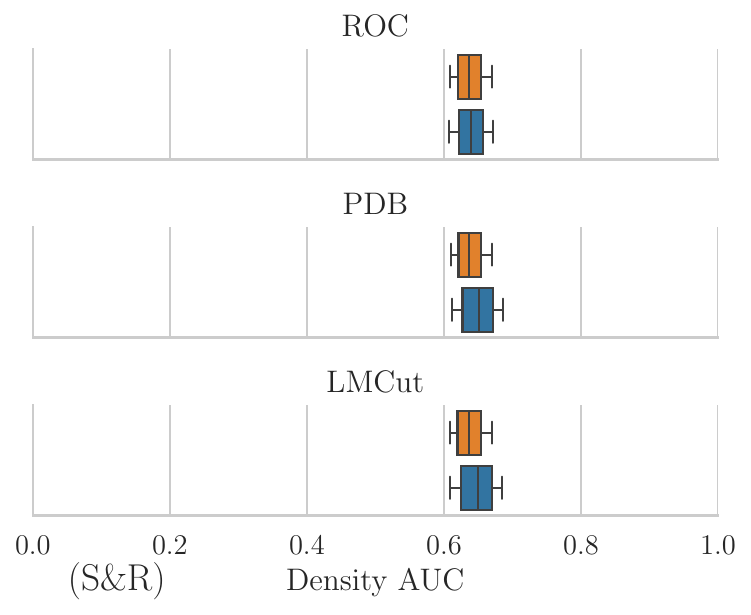}

\vspace{2mm}

\includegraphics[scale=0.36]{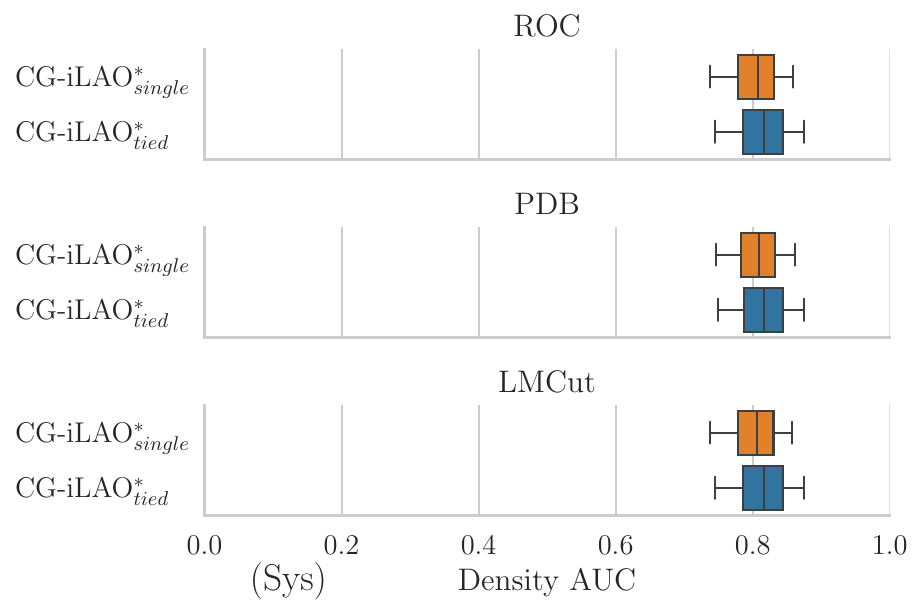}
\includegraphics[scale=0.36]{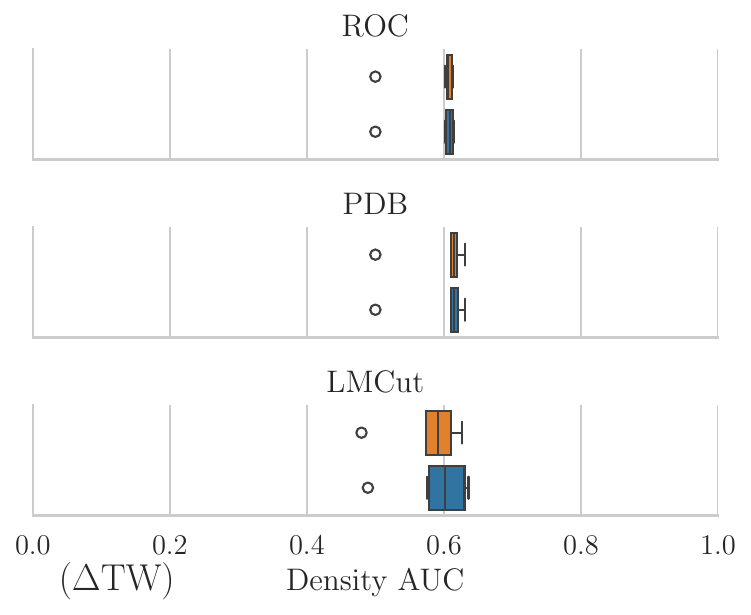}
\includegraphics[scale=0.36]{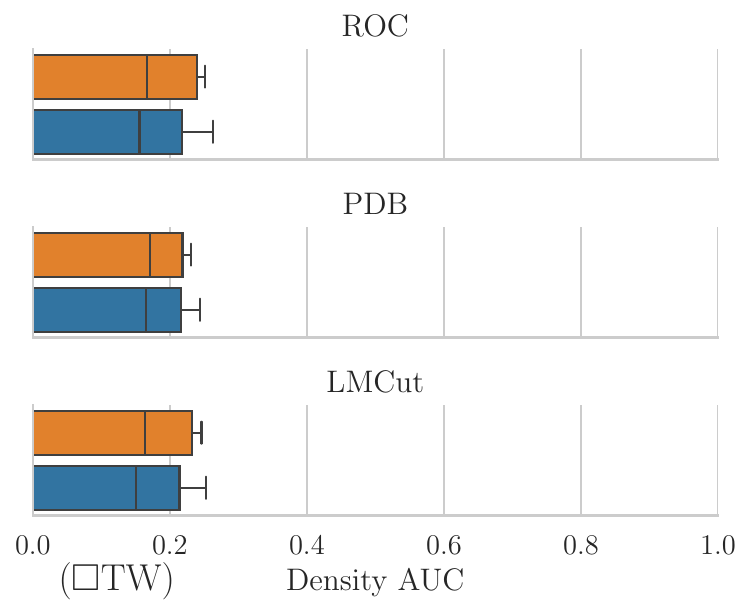}

\vspace{2mm}

\includegraphics[scale=0.36]{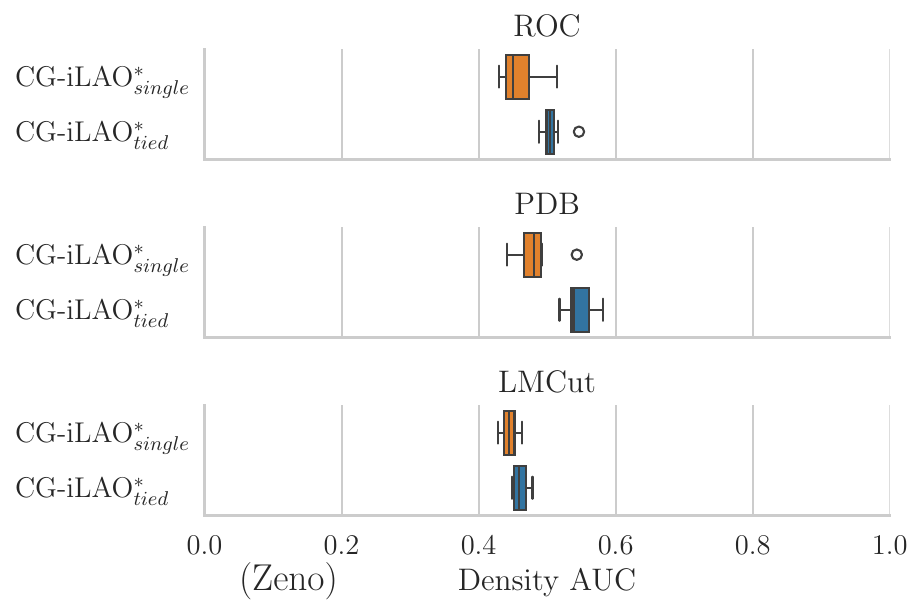}
\caption{Box-and-whisker plots of area under density curve (density AUC) aggregated over problems within each domain.}
\label{fig:density-box-and-whisker}
\end{figure*}

\clearpage %

\subsection{How does \cgilao compare to action elimination?}\label{sec:experiments-action-elimination}

We have discussed that \cgilao's mechanism for ignoring actions has two key advantages over action elimination: it does not require an upper bound, and it starts with a minimal set of actions and only adds more when they are required.
By adding actions as they are required, \cgilao enjoys small partial SSPs from the start and keeps them relatively small; compared to action elimination, which starts with all actions and only removes them when they are proved suboptimal, which is towards the end of the algorithm's lifetime.
These advantages result in \cgilao being more effective than action elimination.
This section confirms our claim experimentally, and furthermore shows that action elimination is able to remove only few actions that \cgilao inserts into its partial SSPs, demonstrating that \cgilao adds few provably suboptimal actions.
In order to compare \cgilao's mechanism with action elimination we now describe how to augment our algorithms with action elimination.

\paragraph{Adding Action Elimination to \ilao}
In general, Bellman-backup-based algorithms can be easily modified to use action elimination with the following modifications:
\begin{enumerate}
\item Rename the lower bound \V to \Vlb, and additionally track an upper bound \Vub.
We initialise \Vub with the trivial upper bound \(\Vub(g) = 0\) for all goals \(g \in \Sg\) and \(\Vub(\s) = \dCost\) for all other states \(\s \in \Ss \setminus \Sg\), where \(\dCost = \infty\) or the penalty term if we are using the fixed-penalty transformation.
\item For each backup to \Vlb, also apply a backup to \Vub (ignoring its residual).
\item After computing \(\Qlb(\s, \ac)\) and \(\Qub(\s, \ac)\) for each \(\ac \in \partactions(\s)\), permanently remove (eliminate) any action \(\tilde{\ac} \in \partactions\) from the partial SSP if it has \(\Qlb(\s, \tilde{\ac}) > \min_{\ac \in \A(\s)} \Qub(\s,\ac)\).
\end{enumerate}
If \(\Vlb \leq \V^* \leq \Vub\), then these modifications preserve the correctness of the algorithm, because we are only removing suboptimal actions with action elimination (\cref{thm:action-elimination}).
Indeed, if \ilao is initialised with \(\Vub \geq \V^*\), then it preserves \(\Vub \geq \V^*\) as an invariant, by using similar arguments to show \ilao has the invariant \(\Vlb \leq \V^*\) when initialised with an admissible heuristic~\cite{Hansen2001:ilao}.
Thus, this definition of \ilao with action elimination behaves as we expect and remains optimal.

\paragraph{Adding Action Elimination to \cgilao}
We can apply the same modifications to \cgilao, but with a subtle yet important change.
Recall that \cgilao does not guarantee \(\Vlb \leq \V^*\) in general, and if we apply action elimination when \(\Vlb \not\leq \V^*\) then optimal actions may be permanently removed, breaking the optimality of the algorithm.
So, we only allow \cgilao to apply action elimination when \(\colsToCheck = \emptyset\), at which point \(\Vlb \leq \V^*\) holds (see \cref{lemma:empty-cols-to-check-implies-admissibility}), and action elimination can be performed safely.\footnote{Actually, \cref{lemma:empty-cols-to-check-implies-admissibility} gives us \(\Vlb \leq \V^* + \epsilon \Nbar{\s}{\ssp}\), i.e., there is an error term. This error term is insignificant in practice, but can theoretically affect optimality.}
There is no analogous concern for \Vub, and it is always a valid upper bound because the absence of actions can only increase \Vub, since optimal actions may be missing.

Thus, we define \cgilaoBellmanTiedActionElim, \cgilaoBellmanSingleActionElim, and \cgilaoBellmanCompleteActionElim by adding action elimination to the relevant variant of \cgilao.
We make note of two technical details:
\begin{enumerate}
\item Initially, the only values of \Vub with \(\Vub(\s) < \dCost\) are goal states, so Bellman backups can not decrease \Vub until goal states enter the partial SSP.
Thus, in our implementation, we only start computing \qvalues for \Vub after a goal state enters the partial SSP, avoiding needless computation of \qvalues fixed at \dCost.
\item We still use \econsistency as the termination condition.
With access to \Vub, other termination conditions are possible, e.g., \(\Vub(\s) - \Vlb(\s) \leq \epsilon \; \forall \s \in \Ss\)~\cite{McMahan2005:BRTDP}, which have slightly different guarantees, and can in some cases be satisfied sooner.
\end{enumerate}

In these experiments, we consider the following algorithms:
\begin{itemize}
\item \cgilaoBellmanTied and \cgilaoBellmanTiedActionElim
\item \cgilaoBellmanSingle and \cgilaoBellmanSingleActionElim
\item \cgilaoBellmanComplete and \cgilaoBellmanCompleteActionElim
\item \ftvi~\cite{Dai2009:ftvi}.
\end{itemize}

To verify that \cgilao is more effective than action elimination, we compare the performance of \cgilaoBellmanTied and \cgilaoBellmanSingle with their action-elimination extended counterparts, and additionally compare them to \cgilaoBellmanComplete and its action-eliminated counterpart, which does not use the mechanism for ignoring inactive actions and implements \ilao.
We also include \ftvi, as a baseline of algorithms that use action elimination.
The ranks (\cref{tab:ranks-cgilao-vs-action-elimination}), coverage (\cref{tab:coverage-cgilao-action-elimination}), and cumulative plots (\cref{fig:main-cumulative-cgilao-vs-action-elimination}) paint a consistent picture:
\begin{itemize}
\item algorithms with action elimination are slower than their counterparts without action elimination,
\item \cgilaoBellmanTied and \cgilaoBellmanSingle are faster than any of the algorithms with action elimination that we consider.
\end{itemize}
The issue with action elimination is that it must track and compute Bellman backups for the additional value function \Vub, incurring an additional overhead of \qvalue computations, and this can only pay off if it eliminates sufficiently many actions to save more \qvalue computations.
By the cumulative plot over \qvalues (\cref{fig:main-cumulative-cgilao-vs-action-elimination}), we see that this is not the case.

\begin{table}[t!]
\centering
\adjustbox{max width=\linewidth}{

\begin{tabular}{|l S[table-format=1.2(1.2)]|}
\multicolumn{2}{c}{\roc} \\
\cgilaoBellmanTied & 3.01(0.12) \\
\cgilaoBellmanSingle & 3.14(0.13) \\
\cgilaoBellmanTiedActionElim & 3.47(0.12) \\
\cgilaoBellmanSingleActionElim & 3.49(0.13) \\
\cgilaoBellmanComplete & 4.33(0.13) \\
\cgilaoBellmanCompleteActionElim & 4.52(0.13) \\
\ftvi & 6.05(0.13) \\
\end{tabular}

\begin{tabular}{|l S[table-format=1.2(1.2)]|}
\multicolumn{2}{c}{\pdb} \\
\cgilaoBellmanTied & 3.01(0.12) \\
\cgilaoBellmanSingle & 3.30(0.13) \\
\cgilaoBellmanTiedActionElim & 3.63(0.11) \\
\cgilaoBellmanSingleActionElim & 3.86(0.12) \\
\cgilaoBellmanComplete & 4.17(0.12) \\
\cgilaoBellmanCompleteActionElim & 4.46(0.12) \\
\ftvi & 5.57(0.13) \\
\end{tabular}

\begin{tabular}{|l S[table-format=1.2(1.2)]|}
\multicolumn{2}{c}{\lmcut} \\
\cgilaoBellmanTied & 3.22(0.11) \\
\cgilaoBellmanSingle & 3.38(0.12) \\
\cgilaoBellmanSingleActionElim & 3.42(0.12) \\
\cgilaoBellmanTiedActionElim & 3.79(0.11) \\
\cgilaoBellmanComplete & 4.08(0.12) \\
\cgilaoBellmanCompleteActionElim & 4.45(0.11) \\
\ftvi & 5.67(0.13) \\
\end{tabular}

}
\caption{Runtime ranking of \cgilao and action elimination methods within a specified heuristic (mean and 95\% CI over all instances).}
\label{tab:ranks-cgilao-vs-action-elimination}
\end{table}

\begin{table}[t!]
\centering
\adjustbox{max width=\linewidth}{

\begin{tabular}{llrrrrrrrrrrrrrr}
 & & \rotatebox{90}{\bw} & \rotatebox{90}{\coresec} & \rotatebox{90}{\elevators} & \rotatebox{90}{\exbw} & \rotatebox{90}{\parcn} & \rotatebox{90}{\parcr} & \rotatebox{90}{\random} & \rotatebox{90}{\recttireworld} & \rotatebox{90}{\sar} & \rotatebox{90}{\schedule} & \rotatebox{90}{\sysadmin} & \rotatebox{90}{\tritireworld} & \rotatebox{90}{\zenotravel} & \rotatebox{90}{total} \\
\cline{1-16}
 & \# of instances & 110 & 35 & 75 & 105 & 30 & 30 & 75 & 70 & 25 & 45 & 20 & 40 & 45 & 705 \\
\cline{1-16}
\multirow[c]{7}{*}{\rotatebox{90}{\roc}} & \ftvi & 55 & 20 & 50 & 65 & 0 & 10 & 33 & 54 & 20 & 30 & \bestCovr{20} & 30 & 5 & 392 \\
 & \cgilaoBellmanComplete & \bestCovr{105} & 25 & 70 & \bestCovr{105} & \bestCovr{30} & \bestCovr{27} & 35 & 60 & \bestCovr{25} & \bestCovr{45} & \bestCovr{20} & 35 & 39 & 621 \\
 & \cgilaoBellmanCompleteActionElim & \bestCovr{105} & 25 & 70 & \bestCovr{105} & \bestCovr{30} & 25 & 35 & 60 & \bestCovr{25} & \bestCovr{45} & \bestCovr{20} & 35 & 35 & 615 \\
 & \cgilaoBellmanTied & \bestCovr{105} & 25 & \bestCovr{75} & \bestCovr{105} & \bestCovr{30} & 25 & 36 & 60 & 20 & \bestCovr{45} & \bestCovr{20} & 35 & \bestCovr{40} & 621 \\
 & \cgilaoBellmanTiedActionElim & \bestCovr{105} & 25 & 74 & \bestCovr{105} & \bestCovr{30} & 25 & 37 & 60 & 20 & \bestCovr{45} & \bestCovr{20} & 35 & 38 & 619 \\
 & \cgilaoBellmanSingle & \bestCovr{105} & 25 & \bestCovr{75} & \bestCovr{105} & \bestCovr{30} & 25 & \bestCovr{43} & 60 & 20 & \bestCovr{45} & \bestCovr{20} & 35 & \bestCovr{40} & \bestCovr{628} \\
 & \cgilaoBellmanSingleActionElim & \bestCovr{105} & 25 & 72 & \bestCovr{105} & \bestCovr{30} & 25 & \bestCovr{43} & 60 & 20 & \bestCovr{45} & \bestCovr{20} & 35 & 38 & 623 \\

\cline{1-16} \multirow[c]{7}{*}{\rotatebox{90}{\pdb}} & \ftvi & 45 & 20 & 50 & 57 & 0 & 0 & 22 & 55 & 20 & 30 & \bestCovr{20} & 30 & 5 & 354 \\
 & \cgilaoBellmanComplete & 85 & \bestCovr{30} & 70 & 90 & 0 & 0 & 30 & 60 & \bestCovr{25} & 30 & \bestCovr{20} & 37 & \bestCovr{40} & 517 \\
 & \cgilaoBellmanCompleteActionElim & 85 & \bestCovr{30} & 70 & 90 & 0 & 0 & 30 & 60 & \bestCovr{25} & 30 & \bestCovr{20} & 35 & \bestCovr{40} & 515 \\
 & \cgilaoBellmanTied & 90 & \bestCovr{30} & \bestCovr{75} & 90 & 0 & 0 & 30 & 60 & 24 & 30 & \bestCovr{20} & \bestCovr{38} & \bestCovr{40} & 527 \\
 & \cgilaoBellmanTiedActionElim & 90 & \bestCovr{30} & 74 & 90 & 0 & 0 & 30 & 60 & \bestCovr{25} & 30 & \bestCovr{20} & 35 & \bestCovr{40} & 524 \\
 & \cgilaoBellmanSingle & 90 & \bestCovr{30} & \bestCovr{75} & 90 & 0 & 0 & 30 & 60 & 20 & 30 & \bestCovr{20} & 37 & \bestCovr{40} & 522 \\
 & \cgilaoBellmanSingleActionElim & 90 & \bestCovr{30} & 72 & 90 & 0 & 0 & 30 & 60 & 20 & 30 & \bestCovr{20} & 35 & \bestCovr{40} & 517 \\
\cline{1-16}

\cline{1-16} \multirow[c]{7}{*}{\rotatebox{90}{\lmcut}} & \ftvi & 45 & 20 & 50 & 70 & 0 & 1 & 18 & 60 & 20 & 30 & \bestCovr{20} & 25 & 5 & 364 \\
 & \cgilaoBellmanComplete & 45 & 25 & 70 & 100 & \bestCovr{30} & 20 & 20 & \bestCovr{65} & \bestCovr{25} & 40 & \bestCovr{20} & 30 & 25 & 515 \\
 & \cgilaoBellmanCompleteActionElim & 45 & 25 & 70 & 100 & \bestCovr{30} & 20 & 20 & \bestCovr{65} & \bestCovr{25} & 40 & \bestCovr{20} & 30 & 25 & 515 \\
 & \cgilaoBellmanTied & 45 & 25 & 70 & 101 & \bestCovr{30} & 20 & 20 & \bestCovr{65} & 24 & \bestCovr{45} & \bestCovr{20} & 30 & 25 & 520 \\
 & \cgilaoBellmanTiedActionElim & 45 & 25 & 70 & 100 & \bestCovr{30} & 20 & 20 & \bestCovr{65} & 21 & 40 & \bestCovr{20} & 30 & 25 & 511 \\
 & \cgilaoBellmanSingle & 45 & 25 & 71 & 100 & \bestCovr{30} & 20 & 20 & \bestCovr{65} & 20 & \bestCovr{45} & \bestCovr{20} & 30 & 25 & 516 \\
 & \cgilaoBellmanSingleActionElim & 45 & 25 & 70 & 101 & \bestCovr{30} & 20 & 20 & \bestCovr{65} & 20 & 42 & \bestCovr{20} & 30 & 25 & 513 \\
\cline{1-16}

\end{tabular}

}
\caption{Coverage for action elimination experiment with each considered heuristic over the benchmark domains.
The highest coverage for each problem is marked with boldface.
}\label{tab:coverage-cgilao-action-elimination}
\end{table}

\begin{figure}[t!]
\raggedright
\includegraphics[scale=0.5, valign=t]{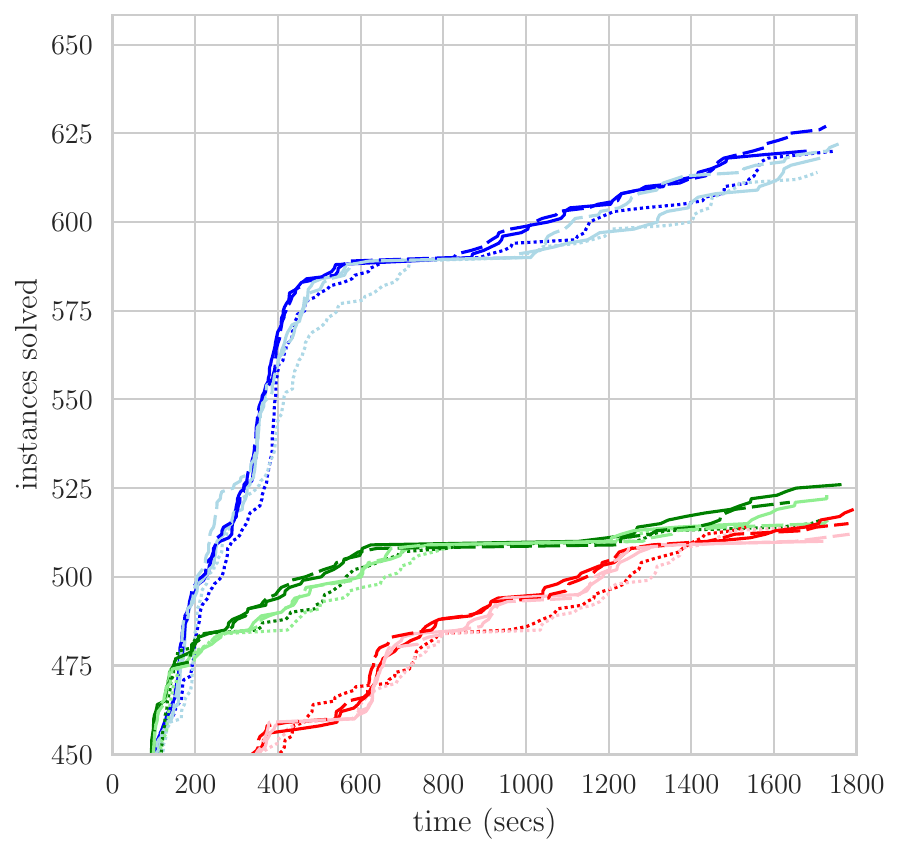}
\includegraphics[scale=0.5, valign=t]{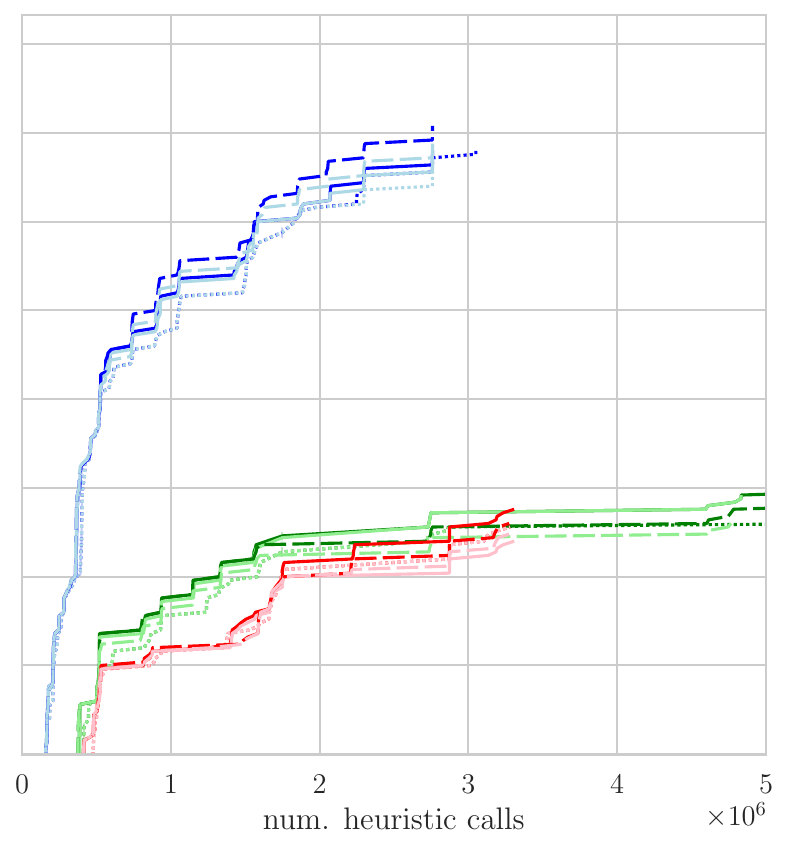}
\includegraphics[scale=0.5, valign=t]{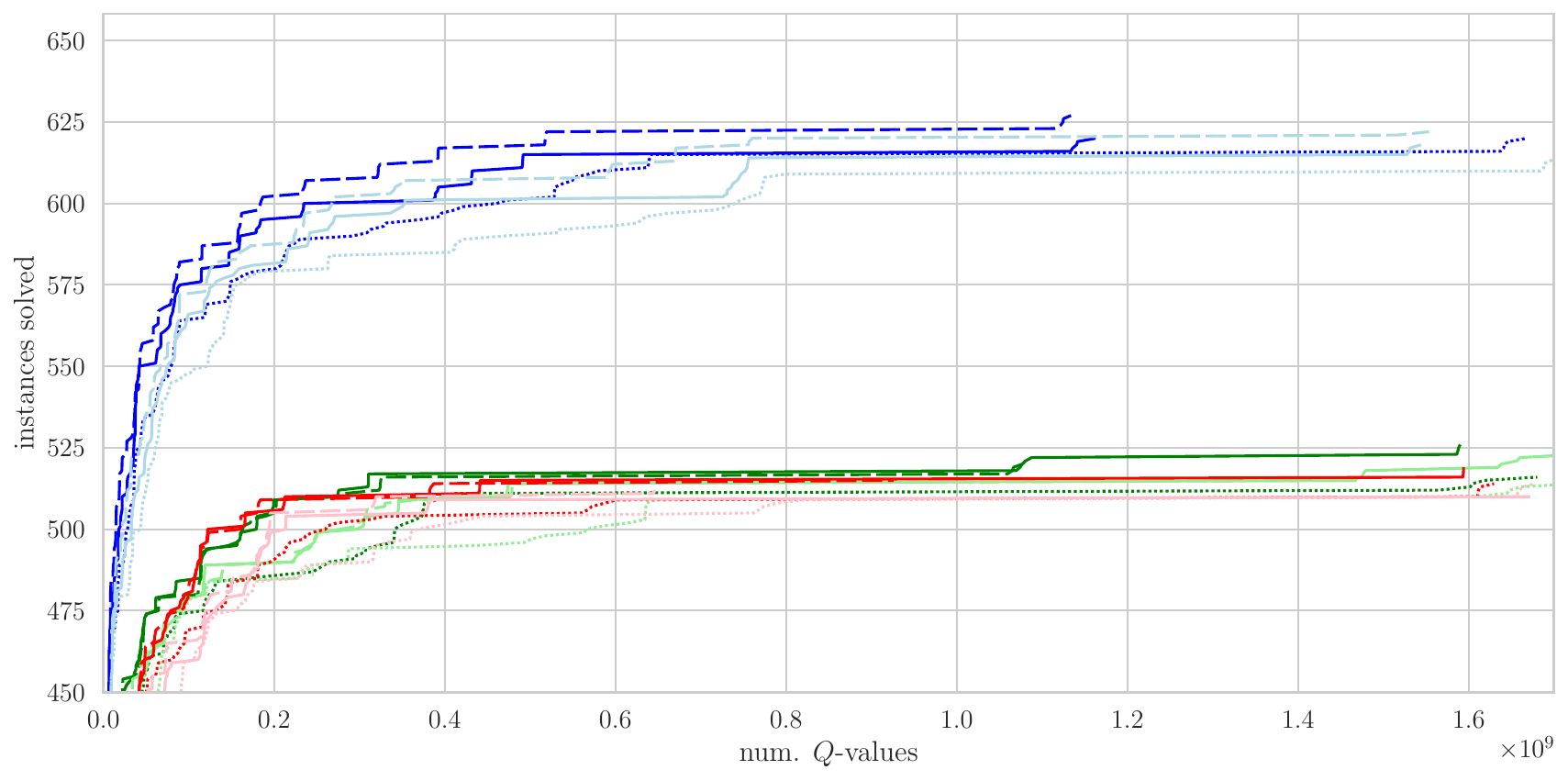}
\begin{center}
\includegraphics[scale=0.5, valign=t]{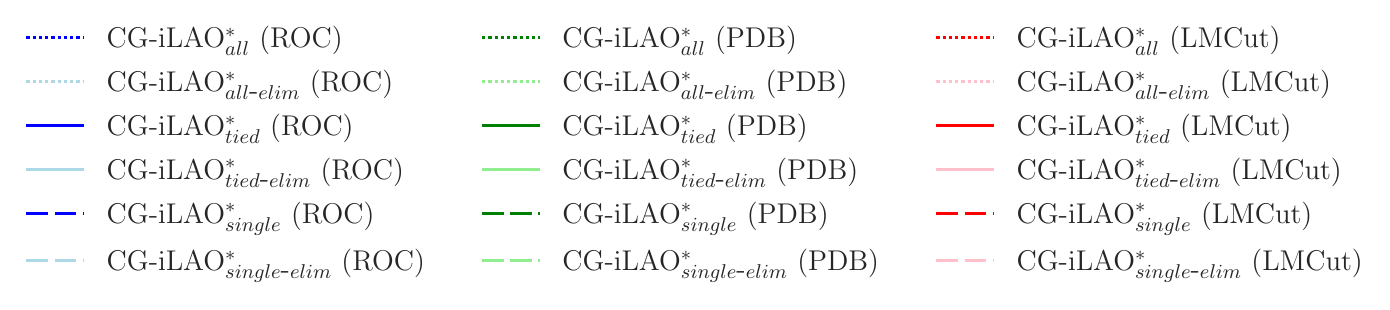}
\end{center}
\caption{
For each algorithm and heuristic, the cumulative plot of how many problems and seeds were solved w.r.t. time in seconds (top left), number of \qvalues (top right), and number of calls to heuristic (bottom).
To focus on where the algorithms are different we start the \(y\)-axis at \data{450} instances, and end the \(x\)-axis at \data{\(5 \times 10^6\)} heuristic calls.
}
\label{fig:main-cumulative-cgilao-vs-action-elimination}
\end{figure}

We now investigate how many actions are removed by action elimination, compared to how many actions are ignored by \cgilao.
\Cref{tab:partial-ssp-sizes-percentages-action-elimination} shows the number of actions in the final partial SSPs of the \cgilao variants.
We do not include states in this table -- action elimination has no way to reduce the number of states, so the percentage of states is equal between algorithms and their counterpart with action elimination (excluding some noise).
In terms of actions, we see that \cgilaoBellmanCompleteActionElim never has significantly fewer actions in its final partial SSP than \cgilaoBellmanTied and \cgilaoBellmanSingle, and often has significantly more actions.
This tells us that \cgilao is more effective at reducing partial SSP sizes than action elimination.
Moreover, if we compare \cgilaoBellmanTied and \cgilaoBellmanSingle with their action elimination counterparts, we see that action elimination is able to shrink the partial SSP slightly, but never significantly, i.e., \cgilao adds very few actions that can be removed by action elimination.
This is a strong argument in favour of \cgilao: its mechanism for ignoring actions induces similarly sized final partial SSPs to action elimination, but it does not require upper bounds, and it avoids these actions from the start, rather than eliminating them towards the end of the algorithm's lifetime.

\begin{table}[t!]
\centering
\adjustbox{max width=\linewidth}{

\begin{tabular}{ll S[table-format=2.1(1.1)] S[table-format=2.1(1.1)] S[table-format=2.1(1.1)] S[table-format=2.1(1.1)] S[table-format=2.1(2.1)] }
 & & \cgilaoBellmanTied & \cgilaoBellmanTiedActionElim & \cgilaoBellmanSingle & \cgilaoBellmanSingleActionElim & \cgilaoBellmanCompleteActionElim \\
\Xhline{1pt}
\multirow[c]{3}{*}{{\bw}} & {\roc} & 40.8(1.2) & 40.7(1.3) & 33.2(2.8) & 33.1(2.8) & 98.8(0.3) \\
 & {\lmcut} & 45.5(1.4) & 44.6(1.3) & 42.5(2.0) & 41.5(1.8) & 97.6(0.3) \\
 & {\pdb} & 39.7(0.9) & 39.3(0.9) & 30.9(1.6) & 30.6(1.6) & 97.4(0.5) \\
\cline{1-7} \multirow[c]{3}{*}{{\coresec}} & {\roc} & 33.6(1.9) & 33.5(1.8) & 20.5(1.1) & 20.4(1.0) & 77.6(4.0) \\
 & {\lmcut} & 39.0(1.9) & 37.4(1.5) & 24.9(1.1) & 23.8(0.9) & 82.6(3.1) \\
 & {\pdb} & 29.0(1.6) & 28.4(1.5) & 22.2(1.6) & 22.0(1.5) & 73.0(3.3) \\
\cline{1-7} \multirow[c]{3}{*}{{\elevators}} & {\roc} & 60.5(1.6) & 60.1(1.5) & 53.3(2.2) & 53.0(2.1) & 98.0(0.6) \\
 & {\lmcut} & 45.0(3.1) & 44.5(3.0) & 40.4(3.0) & 40.0(2.8) & 97.0(1.1) \\
 & {\pdb} & 56.9(1.7) & 56.5(1.6) & 50.1(2.2) & 49.8(2.2) & 97.5(0.8) \\
\cline{1-7} \multirow[c]{3}{*}{{\exbw}} & {\roc} & 44.8(1.5) & 44.1(1.5) & 39.6(1.7) & 38.8(1.7) & 85.6(2.2) \\
 & {\lmcut} & 39.2(2.1) & 38.6(2.0) & 37.1(1.8) & 36.5(1.7) & 88.3(2.4) \\
 & {\pdb} & 53.9(3.4) & 53.3(3.4) & 49.9(3.7) & 49.4(3.7) & 93.4(1.5) \\
\cline{1-7} \multirow[c]{2}{*}{{\parcn}} & {\roc} & 44.7(0.9) & 34.8(1.0) & 32.7(0.9) & 23.9(0.9) & 44.8(1.6) \\
 & {\lmcut} & 44.0(1.2) & 37.1(1.3) & 35.2(1.3) & 29.0(1.1) & 62.0(1.7) \\
\cline{1-7} \multirow[c]{2}{*}{{\parcr}} & {\roc} & 44.1(1.0) & 42.9(0.7) & 43.0(1.0) & 41.9(0.8) & 92.7(2.0) \\
 & {\lmcut} & 51.9(1.0) & 50.8(1.1) & 50.1(1.2) & 49.1(1.4) & 89.8(1.8) \\
\cline{1-7} \multirow[c]{3}{*}{{\random}} & {\roc} & 9.1(1.8) & 9.1(1.8) & 0.9(0.3) & 0.9(0.3) & 100.0(0.0) \\
 & {\lmcut} & 9.9(3.9) & 9.9(3.9) & 1.8(0.7) & 1.8(0.7) & 76.1(18.6) \\
 & {\pdb} & 17.0(8.1) & 17.0(8.0) & 1.0(0.5) & 1.0(0.5) & 75.0(14.8) \\
\cline{1-7} \multirow[c]{3}{*}{{\recttireworld}} & {\roc} & 31.5(0.8) & 26.6(1.1) & 33.1(1.6) & 27.9(1.7) & 66.2(4.2) \\
 & {\lmcut} & 30.9(0.9) & 26.4(1.2) & 32.8(1.8) & 27.9(1.8) & 67.4(4.2) \\
 & {\pdb} & 31.0(1.5) & 26.3(1.3) & 31.5(1.8) & 26.7(1.8) & 66.2(3.9) \\
\cline{1-7} \multirow[c]{3}{*}{{\sar}} & {\roc} & 91.6(0.8) & 83.3(2.3) & 91.4(0.7) & 83.0(2.3) & 91.1(1.9) \\
 & {\lmcut} & 92.0(0.8) & 83.5(2.3) & 91.3(0.7) & 82.9(2.3) & 90.8(2.1) \\
 & {\pdb} & 92.2(0.8) & 85.1(2.5) & 91.4(0.7) & 82.7(2.4) & 91.3(2.0) \\
\cline{1-7} \multirow[c]{3}{*}{{\schedule}} & {\roc} & 78.1(0.7) & 77.6(0.9) & 50.7(1.1) & 50.2(1.1) & 81.4(0.9) \\
 & {\lmcut} & 73.8(3.2) & 73.2(3.5) & 50.2(1.2) & 49.8(1.3) & 79.0(0.5) \\
 & {\pdb} & 79.6(0.5) & 79.6(0.5) & 51.1(1.3) & 51.1(1.3) & 79.6(0.5) \\
\cline{1-7} \multirow[c]{3}{*}{{\sysadmin}} & {\roc} & 99.5(0.3) & 95.8(0.8) & 97.4(0.3) & 93.6(0.8) & 95.9(0.8) \\
 & {\lmcut} & 99.5(0.4) & 95.4(0.9) & 97.3(0.4) & 93.3(0.8) & 95.5(0.9) \\
 & {\pdb} & 100.0(0.0) & 95.6(1.0) & 98.4(0.4) & 93.8(0.7) & 95.6(1.0) \\
\cline{1-7} \multirow[c]{3}{*}{{\tritireworld}} & {\roc} & 64.4(0.8) & 63.0(0.9) & 63.8(0.9) & 62.3(0.9) & 86.3(2.9) \\
 & {\lmcut} & 68.6(1.5) & 67.8(1.8) & 67.0(1.5) & 66.7(1.6) & 88.5(3.9) \\
 & {\pdb} & 65.7(0.7) & 64.2(0.8) & 65.5(0.7) & 64.0(0.8) & 86.2(3.2) \\
\cline{1-7} \multirow[c]{3}{*}{{\zenotravel}} & {\roc} & 37.1(0.6) & 37.0(0.6) & 30.1(0.8) & 30.1(0.8) & 99.9(0.1) \\
 & {\lmcut} & 34.3(0.4) & 34.3(0.4) & 31.6(0.5) & 31.6(0.5) & 99.9(0.1) \\
 & {\pdb} & 44.0(1.3) & 44.0(1.3) & 34.0(1.1) & 34.0(1.1) & 99.8(0.1) \\
\Xhline{1pt}
\end{tabular}

}
\caption{Number of actions in each algorithm's final partial SSP, as a percentage of \cgilaoBellmanComplete's final partial SSP. These values are means over instances where the considered algorithm and \cgilaoBellmanComplete both terminated.}
\label{tab:partial-ssp-sizes-percentages-action-elimination}
\end{table}

\section{Conclusion}\label{sec:conclusion}

In this paper, we addressed an open problem in optimal heuristic search for SSPs: existing methods consider all applicable actions, even when the heuristic makes it clear that some are not needed.
To do this, we built on existing connections between operations research and planning to reframe \ilao, a state-of-the-art optimal heuristic-search algorithm, as constraint and variable generation for LPs.
Under this lens, we were able to refine the existing method's separation oracle, which enables the algorithm to ignore inactive actions, and only add actions when they are deemed necessary.
Bringing this back into a dynamic programming implementation yields our optimal heuristic-search algorithm \cgilao.
\cgilao's novel ability to ignore inactive actions gives it different theoretical properties from existing methods, e.g., it does not guarantee that its value function remains admissible, and the algorithm is not monotonic.
Nevertheless, we have proved that it is an optimal algorithm.
We showed experimentally that \cgilao's mechanism for ignoring actions pays off, and it outperforms the state-of-the-art on our benchmarks.
Going into more detail, we showed that \cgilao is able to ignore a significant proportion of actions, resulting in significantly smaller search spaces than \ilao, which in turn results in fewer \qvalue computations and savings in time.
In comparison, the existing technique of action elimination, which proves that actions are suboptimal in order to permanently remove them, fails to pay off.
We showed that \cgilao's mechanism for ignoring actions is faster, generally ignores the same actions that action elimination removes, and in addition \cgilao not require an upper bound.
Thus, \cgilao successfully addresses the open problem, and represents an important step in heuristic-search that lets us use heuristics to select promising actions, not only states.

\section{Future Work}\label{sec:future-work}

As a direction for future research, we aim to generalise \cgilao to more complex models.
We have seen that \cgilao's mechanism for ignoring inactive actions and adding them as necessary yields in savings of \qvalue computations, which is a promising benefit for problems where \qvalues are expensive to compute.
For example, consider models with imprecise parameters, such as MDPIPs and MDPSTs~\cite{white94mdpip,trevizan07:mdpst}.
These models have \textit{minimax} semantics for the Bellman equations, which means that their value function minimises the expected cost-to-go while accounting for an adversary who selects the values of the imprecise parameters in a way that maximises the cost-to-go.
Thus, computing each \qvalue\ is an expensive operation that requires solving a maximisation problem.
Consequently, \cgilao's mechanism for avoiding actions and saving on \qvalues could potentially improve performance significantly.

Other candidate models include SSPs with \textit{PLTL constraints}~\cite{baumgartner18:pltldual,mallet21:pltlheuristics}.
Typically, these models are solved by extending the state and action space of the original SSP with information about the relevant PLTL formulae, in order to track whether the PLTL constraints are satisfied.
It may be possible to use the concept of inactive actions to avoid actions that lead to constraint violations in their partial problems.
In fact, the methods presented in this paper may be applicable to model checking more broadly.
For example, previous work has explored the use of heuristics to guide the search for probabilistic reachability~\cite{Brazdil2014}, where action elimination can be directly applied.

\section*{Acknowledgements}
This research was undertaken with the assistance of resources and services from the National Computational Infrastructure (NCI), which is supported by the Australian Government.
Johannes Schmalz received funding from DFG grant 389792660 as part of TRR 248 (see \url{https://perspicuous-computing.science}).

\bibliography{references}

\end{document}